\documentclass{article}

\usepackage[utf8]{inputenc} %

\usepackage[preprint]{jmlr_custom}

\usepackage{graphicx} %
\usepackage{enumitem}
\setlist[enumerate]{label=(\roman*), leftmargin=*, itemsep=1pt, topsep=2pt}
\usepackage{microtype}
\usepackage{pdflscape}
\usepackage{amsmath}

\usepackage{etoolbox}
\BeforeBeginEnvironment{remark}{\vspace{6pt}}

\usepackage{amsfonts}
\usepackage{amssymb,amsthm}
\usepackage{bm}
\usepackage{bbm}
\usepackage{parskip}
\usepackage{rotating}
\usepackage{mathtools}
\usepackage{mathrsfs}
\usepackage{mathbbol}
\usepackage{xcolor}
\usepackage[dvipsnames]{xcolor}
\usepackage{adjustbox}
\usepackage[capitalize]{cleveref}
\usepackage{mathtools}
\usepackage[only,llbracket,rrbracket]{stmaryrd}

\usepackage{booktabs}  %
\usepackage{array}     %
\usepackage[most]{tcolorbox}  %

\usepackage{tikz}
\usetikzlibrary{shapes.geometric, arrows.meta, positioning, calc}

\definecolor{intropath}{RGB}{93, 178, 96}
\definecolor{advancedpath}{RGB}{218, 153, 76}
\definecolor{expertpath}{RGB}{172, 109, 218}

\renewcommand{\arraystretch}{1.3}  %
\definecolor{rowgray}{gray}{0.95}  %

\newtcolorbox{continousbox}[1][]{%
  enhanced,
  breakable,
  frame hidden,
  boxrule=0pt,
  parbox=false,
  borderline={0pt}{0pt}{white},
  colback=blue!50!gray!10!white,
  colbacklower=blue!35!gray!10!white,
  arc=2mm,
  title={#1},
  fonttitle=\bfseries,
  coltitle=black, %
  attach boxed title to top left={xshift=2mm,yshift*=-2mm},
  varwidth boxed title,
  boxed title style={
    frame hidden,
    boxrule=0pt,
    colback=blue!15!white,
    arc=1mm,
    left=2mm, right=2mm, top=0.6mm, bottom=0.6mm
  }
}

\newcommand{\continuous}[2][]{%
  \begin{continousbox}[#1]
    #2
  \end{continousbox}
}

\definecolor{customblue}{HTML}{4A90E2}

\newtcolorbox{discretebox}[1][]{%
  enhanced,
  breakable,
  frame hidden,
  boxrule=0pt,
  parbox=false,
  borderline={0pt}{0pt}{white},
  colback=red!50!gray!10!white,
  colbacklower=red!35!gray!10!white,
  arc=2mm,
  title={#1},
  fonttitle=\bfseries,
  coltitle=black, %
  attach boxed title to top left={xshift=2mm,yshift*=-2mm},
  varwidth boxed title,
  boxed title style={
    frame hidden,
    boxrule=0pt,
    colback=red!15!white,
    arc=1mm,
    left=2mm, right=2mm, top=0.6mm, bottom=0.6mm
  }
}

\newcommand{\discrete}[2][]{%
  \begin{discretebox}[#1]#2\end{discretebox}%
}

\definecolor{generalbg}{HTML}{FDF1DB}   
\definecolor{generaltitle}{HTML}{F3DEAA}  

\definecolor{continuous}{HTML}{5580b0}   
\definecolor{general}{HTML}{f3b059}
\definecolor{discrete}{HTML}{b72d40}

\newtcolorbox{generalbox}[1][]{%
  enhanced,
  breakable,
  frame hidden,
  boxrule=0pt,
  parbox=false,
  colback=generalbg,
  colbacklower=generalbg!4!gray,
  arc=2mm,
  title={#1},
  fonttitle=\bfseries,
  coltitle=black,
  attach boxed title to top left={xshift=2mm,yshift*=-2mm},
  varwidth boxed title,
  boxed title style={
    frame hidden,
    boxrule=0pt,
    colback=generaltitle,
    arc=1mm,
    left=2mm, right=2mm, top=0.6mm, bottom=0.6mm
  }
}

\newcommand{\general}[2][]{%
  \begin{generalbox}[#1]#2\end{generalbox}%
}

\newcommand{\transp}{^\top}

\newcommand{\e}{\mathbf{e}}

\newcommand{\ub}{\mathbf{u}}
\newcommand{\x}{\mathbf{x}}

\newcommand{\y}{\mathbf{y}}

\newcommand{\w}{\mathbf{w}}
\newcommand{\p}{\mathbf{p}}
\newcommand{\q}{\mathbf{q}}
\newcommand{\z}{\mathbf{z}}

\newcommand{\mub}{\boldsymbol{\mu}}

\newcommand{\epsilonb}{\boldsymbol{\epsilon}}
\newcommand{\phib}{{\boldsymbol{\phi}}}
\newcommand{\psib}{{\boldsymbol{\psi}}}

\newcommand{\Sigmab}{\boldsymbol{\Sigma}}
\newcommand{\etab}{\boldsymbol{\eta}}
\newcommand{\thetab}{{\boldsymbol{\theta}}}

\newcommand{\s}{\mathbf{s}}

\newcommand{\f}{\mathbf{f}}
\newcommand{\wtilde}{\tilde{\w}}
\newcommand{\Q}{\mathbf{Q}}
\newcommand{\Qseq}{\boldsymbol{\mathcal Q}}
\newcommand{\R}{\mathbf{R}}
\newcommand{\Rseq}{\boldsymbol{\mathcal R}}
\newcommand{\I}{\mathbf{I}}
\newcommand{\J}{\mathbf{J}}
\newcommand{\C}{\mathbf{C}}
\newcommand{\B}{\mathbf{B}}
\newcommand{\Sb}{\mathbf{S}}

\newcommand{\zerob}{\mathbf{0}}
\newcommand{\calX}{\mathcal{X}}
\newcommand{\calD}{\mathcal{D}}

\newcommand{\RR}{\mathbb{R}}
\newcommand{\Normal}{\mathcal{N}}
\newcommand{\Cat}{\mathrm{Cat}}

\newcommand{\ind}{\mathbbm{1}}
\newcommand{\Qpath}{\mathbb{Q}}
\newcommand{\Ppath}{\mathbb{P}}
\newcommand{\Lgen}{\mathscr{L}}
\newcommand{\Ggen}{\mathscr{G}}

\newcommand{\diff}{\mathrm{d}}
\newcommand{\dt}{\diff t}
\newcommand{\dx}{\diff \x}
\newcommand{\dw}{\diff \w}

\newcommand{\KL}{D_\mathrm{KL}}

\newcommand{\E}{\mathbb{E}}

\newcommand{\cond}{\,|\,}

\newcommand{\qdata}{q_{\mathrm{data}}}
\newcommand{\pnoise}{\p_{\mathrm{noise}}}

\definecolor{figblue}{HTML}{3A62B4}
\definecolor{figred}{HTML}{C8143C}
\definecolor{figyellow}{HTML}{FDF2DC}

\title{Foundations of Diffusion Models in General State Spaces:\\A Self-Contained Introduction}

\author{%
\textbf{Vincent Pauline}\thanks{Correspondence to: \texttt{vincent.paulinef@gmail.com} and \texttt{andrea.dittadi@gmail.com}.} \affnum{1,2,3} \and
\textbf{Tobias Höppe}\affnum{1,2,3} \and
\textbf{Kirill Neklyudov}\affnum{4,5} \\[2pt]
\textbf{Alexander Tong}\affnum{4,5,6} \and
\textbf{Stefan Bauer}\affnum{1,2,3} \and
\textbf{Andrea Dittadi}\footnotemark[1] \affnum{1,2,3}
\\[8pt]
\affiliation{1}{Technical University of Munich}\and
\affiliation{2}{Helmholtz AI}\and 
\affiliation{3}{Munich Center for Machine Learning (MCML)} \\
\affiliation{4}{Mila -- Quebec AI Institute} \and
\affiliation{5}{Université de Montréal}\and
\affiliation{6}{AITHYRA}
}

\begin{document}

\maketitle

\begin{abstract}
Although diffusion models now occupy a central place in generative modeling, introductory treatments commonly assume Euclidean data and seldom clarify their connection to discrete-state analogues. This article is a \emph{self-contained} primer on diffusion over \emph{general state spaces}, unifying continuous domains and discrete/categorical structures under one lens. We develop the discrete-time view (forward noising via Markov kernels and learned reverse dynamics) alongside its continuous-time limits---stochastic differential equations (SDEs) in $\mathbb{R}^d$ and continuous-time Markov chains (CTMCs) on finite alphabets---and derive the associated Fokker--Planck and master equations. A common variational treatment yields the ELBO that underpins standard training losses.
We make explicit how forward corruption choices---Gaussian processes in continuous spaces and structured categorical transition kernels (uniform, masking/absorbing and more) in discrete spaces---shape reverse dynamics and the ELBO. The presentation is layered for three audiences: newcomers seeking a self-contained intuitive introduction; diffusion practitioners wanting a global theoretical synthesis; and continuous-diffusion experts looking for an analogy-first path into discrete diffusion. The result is a unified roadmap to modern diffusion methodology across continuous domains and discrete sequences, highlighting a compact set of reusable proofs, identities, and core theoretical principles.\looseness=-1
\end{abstract}

\begin{figure}[h!]
    \centering
    \includegraphics[width=1\linewidth]{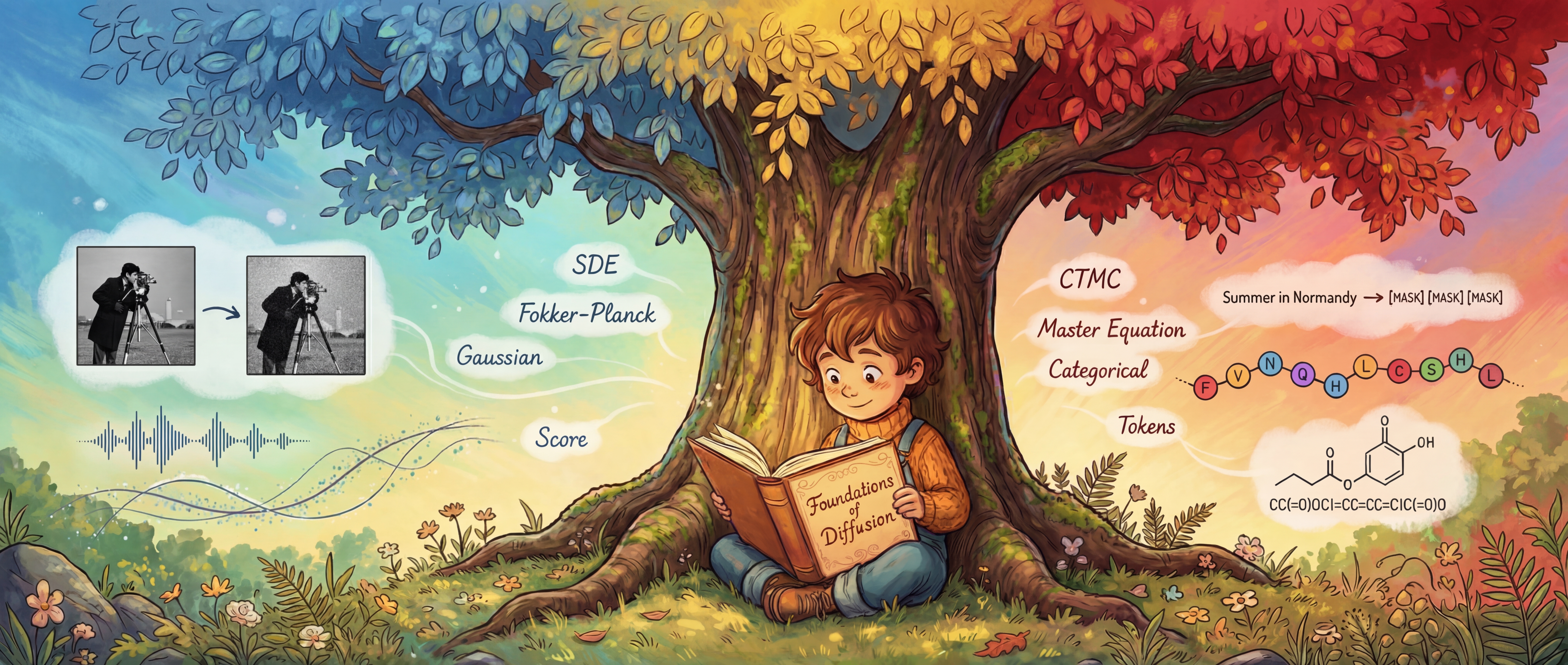}
    \caption{Young researcher reading \emph{Foundations of Diffusion Models in General State Spaces: A Self-Contained Introduction}. The left side illustrates key ideas from continuous-state diffusion models, and the right side highlights corresponding principles for discrete-state models. Image generated with gemini-3-pro-image-preview \cite{team2023gemini}.}
    \label{fig:image_diffusion}
\end{figure}

\newpage

\tableofcontents

\newpage

\section{Overview}

\subsection{Manuscript structure}

\begin{figure}[h!]
    \centering
    \vspace{10pt}
    \includegraphics[width=0.98\linewidth]{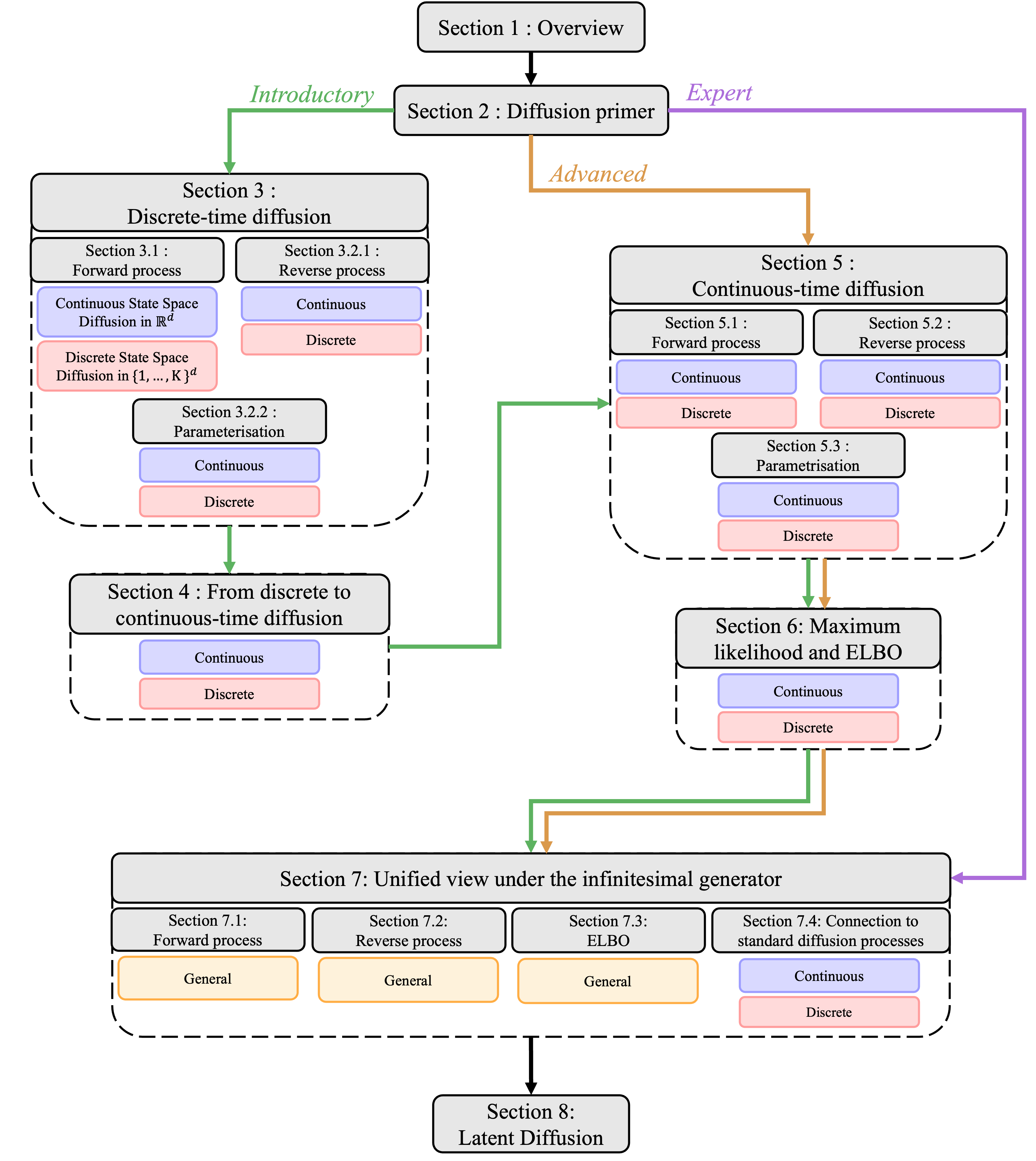}
    \caption{Visual roadmap of the manuscript. Three suggested reading paths are indicated: the \textcolor{intropath}{\textit{introductory}} path (green) for newcomers to diffusion models, the \textcolor{advancedpath}{\textit{advanced}} path (brown) for practitioners familiar with discrete-time diffusion seeking continuous-time theory, and the \textcolor{expertpath}{\textit{expert}} path (purple) for readers seeking the most general theoretical framework.}
    \label{fig:roadmap}
\end{figure}

This manuscript provides a unified treatment of diffusion models across continuous and discrete state spaces, developing the theory from discrete-time formulations through to continuous-time limits, and establishing connections to modern generator-based perspectives. It is structured to serve multiple audiences with different backgrounds and goals:

\textit{First read advice for newcomers to diffusion models} (\textcolor{intropath}{\textit{introductory}} path in \cref{fig:roadmap}): We recommend that newcomers begin with \cref{sect:diffusion-primer} (Diffusion Primer), which provides intuition and historical context from a probabilistic modeling perspective, then proceed to \cref{sect:discrete_time_diffusion} (Discrete-time introduction) for the foundational discrete-time formulation, consult \cref{sect:ELBO} (Maximum likelihood and ELBO) for training objectives, and \cref{sect:latent-diffusion} (Latent diffusion) for performing diffusion in latent spaces. After gaining familiarity with this core material, the continuous-time formulation (\cref{sect:discrete-to-continuous-time-csp,sect:continuous-diffusion}) and the generator perspective (\cref{sect:generator-perspective}) provide valuable theoretical depth and unification.

\textit{First read advice for practitioners familiar with discrete time diffusion} (\textcolor{advancedpath}{\textit{advanced}} path in \cref{fig:roadmap}): If you are familiar with discrete-time diffusion and interested in the continuous-time theory, you may skip directly to \cref{sect:continuous-diffusion} (Continuous-time diffusion) and refer if necessary to \cref{sect:discrete_time_diffusion} (Discrete-time introduction) and \cref{sect:discrete-to-continuous-time-csp} (Discrete-to-continuous time diffusion) to relate continuous to discrete time.

\textit{First read advice for experts seeking a complete synthesis} (\textcolor{expertpath}{\textit{expert}} path in \cref{fig:roadmap}): The entire document provides a unified treatment with parallel derivations for continuous and discrete state spaces. The generator formalism in \cref{sect:generator-perspective} offers the most general framework, subsuming both SDEs and CTMCs as special cases and providing a unified view of time reversal and maximum-likelihood training.

Throughout this article, we maintain parallel treatment of continuous and discrete state spaces, highlighting both their shared structure and their distinctive features. Main results are organised using coloured boxes for easy navigation:
\begin{itemize}
    \item \textcolor{continuous}{\textbf{Blue boxes}} contain results specific to \textbf{continuous state spaces} (e.g., Gaussian diffusion, SDEs, Fokker–Planck equations).
    \item \textcolor{discrete}{\textbf{Red boxes}} contain results specific to \textbf{discrete state spaces} (e.g., categorical diffusion, CTMCs, master equations).
    \item \textcolor{general}{\textbf{Yellow boxes}} contain \textbf{general results} valid for both continuous and discrete settings.
\end{itemize}

The key sections cover:
\begin{itemize}
    \item \textbf{\cref{sect:diffusion-primer}:} High-level intuition and historical motivation from a probabilistic modeling perspective.
    \item \textbf{\cref{sect:discrete_time_diffusion}:} Discrete-time formulation with explicit forward/reverse processes and parameterisations.
    \item \textbf{\cref{sect:discrete-to-continuous-time-csp}:} The limiting procedure from discrete to continuous time.
    \item \textbf{\cref{sect:continuous-diffusion}:} Continuous-time theory via Kolmogorov equations, SDEs, and CTMCs.
    \item \textbf{\cref{sect:ELBO}:} Maximum likelihood training and the ELBO derivation.
    \item \textbf{\cref{sect:generator-perspective}:} General framework unifying continuous and discrete diffusion through the infinitesimal generator and its adjoint, showing how both SDEs and CTMCs emerge as special cases.
    \item \textbf{\cref{sect:latent-diffusion}:} Performing diffusion in learned latent spaces and connections between continuous and discrete approaches for discrete data.

\end{itemize}

\subsection{Notation}
We adopt the following notational conventions throughout:

\vspace{0.5em}
\begingroup
\renewcommand{\arraystretch}{1.25}
\begin{center}
\begin{tabular*}{\textwidth}{@{\extracolsep{\fill}}>{\centering\arraybackslash}p{1.4in} p{4.5in}@{}}
\hline
\textbf{Notation} & \textbf{Meaning} \label{table:notation1}
\\
\hline
\multicolumn{2}{l}{\textit{Vectors, matrices, and indexing}} \\
$\alpha$, $\sigma$, $t$   & scalars (noise schedule parameters, time) \\
$\x$, $\y$, $\z$       & vectors or vector-valued random variables \\
$x^{(k)}$       & $k$-th element of vector $\x$ \\
$\x^{\setminus k}$ & vector $\x$ with all elements except the $k$-th \\
$\R$, $\Q$       & matrices (rate matrix, transition matrix) \\
$\R_{ij}$ or $[\R_t]_{ij}$   & element $(i, j)$ of matrix $\R$ or $\R_t$ \\
$\I$, $\mathbf{1}$, $\mathbf{0}$ & identity matrix, vector of ones, vector of zeros \\
$\e_k$ & $k$-th standard basis (one-hot) vector \\
\hline
\multicolumn{2}{l}{\textit{Matrix and vector operations}} \\
$^\top$ & matrix/vector transpose \\
$\odot$ & Hadamard (element-wise) product \\
$\otimes$ & Kronecker product \\
$\oplus$ & Kronecker sum \\
\hline
\multicolumn{2}{l}{\textit{Differential operators}} \\
$\diff$, $\diff t$, $\diff \x$ & differential, time differential, state differential \\
$\partial_t$ & partial derivative with respect to $t$ \\
$\nabla_\x$ & gradient with respect to $\x$ \\
$\nabla_\x \cdot$ & divergence with respect to $\x$ \\
$\Delta_\x$ & Laplacian with respect to $\x$ \\
\hline
\multicolumn{2}{l}{\textit{Stochastic processes and path measures}} \\
$\x_{[s,t]}$ & sample path from time $s$ to time $t$ \\
$\Ppath$, $\Qpath$ & path measures \\
$\w_t$, $\diff \w_t$ & Wiener process (Brownian motion) and its increment \\
\hline
\multicolumn{2}{l}{\textit{Probability distributions and densities}} \\
$q_t(\x)$ & marginal density/probability at time $t$ \\
$q(\x_t \cond \x_s)$ or $q_{t|s}$ & conditional transition from time $s$ to $t$ \\
$\qdata$ & data distribution \\
$p_{\text{noise}}$  & noise/prior distribution \\
$p^\thetab$ & learned/parametric reverse distribution \\
$\mathcal{N}(x;\mub,\Sigmab)$ & Gaussian distribution over $\calX$ with mean $\mub$ and covariance $\Sigmab$ \\
$\Cat(x;\q)$ & categorical distribution over $\calX$ with probability vector $\q$ \\
\hline
\end{tabular*}
\end{center}
\endgroup
\vspace{0.5em}

\vspace{0.5em}
\begingroup
\renewcommand{\arraystretch}{1.25}
\begin{center}
\begin{tabular*}{\textwidth}{@{\extracolsep{\fill}}>{\centering\arraybackslash}p{1.4in} p{4.5in}@{}}
\hline
\textbf{Notation} & \textbf{Meaning} \label{table:notation2}
\\
\hline
\multicolumn{2}{l}{\textit{Probabilistic operators and functionals}} \\
$\Lgen_t$, $(\Lgen_t)^*$ & infinitesimal generator and its adjoint \\
$\Ggen_t$ & extended generator ($\partial_t + \Lgen_t$) \\
$\E$, $\E_{q}[\cdot]$ & expectation (under distribution $q$) \\
$\KL(\cdot \| \cdot)$ & Kullback--Leibler divergence \\
$\ind_A$, $\delta_{x,y}$ & indicator function of set $A$; Kronecker delta \\
\hline
\multicolumn{2}{l}{\textit{Time reversal (hat convention)}} \\
$t \in [0,1]$ & forward time: data at $t=0$, noise at $t=1$ \\
$s \in [0,1]$ & reverse time: noise at $s=0$, data at $s=1$; related by $s = 1-t$ \\
$\hat q_s \coloneqq q_{1-s}$ & time-reversed marginal (index substitution) \\
$\hat \x_s \coloneqq \x_{1-s}$ & time-reversed process (index substitution) \\
$\hat\Lgen_s$, $\hat\Rseq_s$ & reverse process generator/rate matrix \\
$\hat\Qpath$, $\hat\Ppath^\thetab$ & path measures of the true and approximated reverse process \\
\hline
\multicolumn{2}{l}{\textit{Function spaces}} \\
$\mathcal{D}(\Lgen)$ & domain of definition of operator $\Lgen$ \\
$\langle f, g \rangle$ & inner product $\int f(\x) g(\x) \diff\mu(\x)$ \\
$L^2(\calX^d)$ & functions on $\calX^d$ with finite $L^2$-norm \\
\hline
\multicolumn{2}{l}{\textit{State spaces}} \\
$\calX$, $\calX^d$ & state space \\
$\RR^d$ & $d$-dimensional Euclidean space \\
$\Delta^{K-1}$ & probability simplex over $K$ categories \\
\hline
\end{tabular*}
\end{center}
\endgroup
\vspace{0.5em}

\section{Diffusion primer}\label{sect:diffusion-primer}

This section provides a concise, self-contained primer on diffusion models, situated within the broader development of probabilistic generative modeling. We begin by revisiting maximum-likelihood learning, latent-variable models, and variational inference, using this foundation to articulate the structural limitations that diffusion models overcome. We then introduce diffusion models as latent-variable constructions built around a fixed forward Markov corruption process with tractable marginals and a learned reverse process, emphasising how this design alleviates simulation cost, identifiability, and parameterisation challenges. From this standpoint, we distill a three-step recipe---selecting the forward noising process, parameterising the reverse denoising dynamics, and optimising an ELBO-based objective---which we subsequently instantiate for both discrete and continuous time, and for continuous and discrete state spaces.

\vspace{9pt}

To understand how diffusion models emerged, it is helpful to revisit the foundations of probabilistic modeling. Probabilistic models have long been designed to capture complex phenomena through tractable parametric families \cite{mackay2003information}. Typically, we specify a parametric \emph{probabilistic model} $p^{\thetab}(\x)$ and estimate its parameters by maximising the likelihood of the observed data:
\begin{align}
    \thetab^* = \arg\max_{\thetab} \E_{\x \sim \qdata}\left[\log p^{\thetab}(\x)\right] \ ,
\end{align}
where $\qdata$ denotes the true data distribution. Equivalently, maximum likelihood minimises the KL divergence between the data distribution and the model:
\begin{align}
    \thetab^* = \arg\min_{\thetab} \KL(\qdata \,\|\, p^{\thetab})\,.
\end{align}

A central challenge in probabilistic modeling is achieving both efficient sample generation and tractable density evaluation. Some model families, such as normalising flows \cite{rezende2015variational,tabak2013family,dinh2014nice}, prioritise tractable density evaluation but restrict expressivity. Others, such as generative adversarial networks (GANs) \cite{goodfellow2014generative}, bypass likelihoods entirely and instead rely on adversarial training, often at the cost of stability. A third route enhances expressivity by introducing latent variables.

Latent-variable models enrich the representational capacity of the model while maintaining tractability. They augment the observed variable $\x$ with an unobserved latent variable $\z$ and define a joint distribution $p^{\thetab}(\x,\z)$. The resulting model is specified by the marginal:
\begin{align}
    p^{\thetab}(\x) = \int p^{\thetab}(\x \cond \z) p(\z)\, \diff\z\,,
\end{align}
where we use $\diff\z$ to denote integration over continuous spaces or summation over discrete spaces.

While sampling from this model is straightforward (sample $\z \sim p(\z)$, then $\x \sim p^{\thetab}(\x\cond \z)$), computing the log-likelihood requires integrating over all possible values of $\z$, which is typically intractable except for very simple models. To make learning feasible, we introduce a \emph{variational posterior} distribution $q^{\phib}(\z\cond\x)$ and lower-bound the log-likelihood as follows:
\begin{align}
    \log p^{\thetab}(\x) &= \log \int p^{\thetab}(\x\cond \z) p(\z)\, \diff\z\\
    &= \log \int q^{\phib}(\z\cond \x)\frac{p^{\thetab}(\x\cond \z) p(\z)}{q^{\phib}(\z\cond \x)}\, \diff\z\\
    &\geq \int q^{\phib}(\z\cond \x)\log \frac{p^{\thetab}(\x\cond \z) p(\z)}{q^{\phib}(\z\cond \x)}\, \diff\z\\
    &= \E_{q^{\phib}(\z\cond \x)}\left[\log p^{\thetab}(\x\cond \z)\right] - \KL(q^{\phib}(\z\cond \x) \,\|\, p(\z))\ ,
\end{align}
where the inequality follows from Jensen’s inequality. This is the familiar evidence lower bound (ELBO), which is equal to the log-likelihood (i.e., the ELBO is tight) if and only if $q^{\phib}(\z\cond \x) = p^{\thetab}(\z \cond \x) \propto p^{\thetab}(\x\cond \z) p(\z)$ almost surely.

The variational distribution $q^{\phib}(\z\cond\x)$ plays the same role as a proposal distribution in importance sampling: it guides sampling toward the regions of latent space that meaningfully contribute to the marginal likelihood. If we attempted to estimate $p^{\thetab}(\x) = \int p^{\thetab}(\x\cond\z) p(\z) \diff\z$ by drawing samples $\z \sim p(\z)$, almost all of these samples would assign negligible likelihood to the specific data point $\x$, especially when the latent space is high-dimensional. As a result, the estimator would have extremely high variance, since only a vanishingly small fraction of samples contribute appreciably to the integral. By instead sampling from $q^{\phib}(\z\cond\x)$, we concentrate probability mass in regions where the integrand $p^{\thetab}(\x\cond\z) p(\z)$ is large, dramatically reducing variance. Variational inference formalises this idea: we choose a tractable family for $q^{\phib}(\z\cond\x)$ and optimize it so that it approximates the true posterior, yielding a tight and computable lower bound on the log-likelihood.

Parameterising $p^{\thetab}(\x\cond \z)$ (decoder) and $q^{\phib}(\z\cond \x)$ (encoder) as neural networks, fixing $p(\z)$ as a simple prior distribution, and maximising this lower bound yields the Variational Auto-Encoder (VAE) \citep{kingma2014auto,rezende2014stochastic}.
VAEs had substantial impact across machine learning and beyond: they are easy to implement and train, and they enabled a new class of scalable generative models. Alongside GANs, VAEs became a foundational framework in modern deep generative modeling. They demonstrated the usefulness of the following principles:
\begin{itemize}[label={}]
    \item (+) Use a conditional distribution $p^{\thetab}(\x\cond \z)$ that is easy to sample from and evaluate.
    \item (+) Combine this with a simple prior $p(\z)$ to define the generative model $p^{\thetab}(\x) = \int p^{\thetab}(\x\cond\z) p(\z) \diff\z$
\end{itemize}
To make the generative model more expressive, one can introduce multiple layers of latent variables \cite{rezende2014stochastic,rezende2015variational,burda2015importance,kingma2014semi}.
While these latent variables are typically denoted $\z_1, \ldots, \z_T$, for notational consistency with the diffusion-model literature we write $\x_1, \ldots, \x_T$ and reserve $\x_0$ for the observed data. The generative model defines a top-down hierarchy of latent variables:
\begin{align}
    \label{eq:backward_process}
    p^{\thetab}(\x_0) = \int p(\x_T)\prod_{t=1}^T p^\thetab(\x_{t-1}\cond \x_t)\, \diff \x_{1:T}\,,
\end{align}
where $p(\x_T)$ is a simple prior and the conditional distributions $p^\thetab(\x_{t-1}\cond \x_t)$ are parameterised by neural networks.
A natural first choice for the variational posterior is a bottom-up Markov chain $q^{\phib}(\x_{1:T}\cond \x_{0}) = \prod_{t=1}^{T} q^{\phib}(\x_{t}\cond \x_{t-1})$,
which yields an ELBO in which the KL term decomposes into a sum of per-layer KL divergences.
In practice, however, such ``bottom-up'' inference networks often perform poorly: the inference model and the generative model have incompatible factorisations, share no structure or parameters, and the resulting ELBO involves expectations over multiple stochastic layers, leading to high-variance gradient estimates and weak usage of higher-level latents.

Modern hierarchical VAEs \cite{sonderby2016ladder,child2020very,vahdat2020nvae} adopt a different factorisation for the variational posterior that mirrors the top-down structure of the generative model:
\begin{align}
    q^{\phib}(\x_{1:T}\cond \x_{0}) = q^{\phib}(\x_{T}\cond \x_{0}) \prod_{t=2}^{T} q^{\phib}(\x_{t-1}\cond \x_{t}, \x_0) \; .
\end{align}
This top-down inference allows the encoder to share parameters and structure with the decoder, and supports amortised inference, where a single network predicts approximate posterior parameters for each datapoint.%
\footnote{In more recent formulations \citep{child2020very,vahdat2020nvae,lievin2019towards,maaloe2019biva}, both the generative model and the variational posterior need not be Markov in the generative direction, allowing dependencies such as $q^{\phib}(\x_{t-1}\cond \x_{t:T}, \x_0)$ and $p^{\thetab}(\x_{t-1}\cond \x_{t:T})$.}
The log-likelihood can now be lower-bounded as follows:
\begin{align}
    \log p^\thetab(\x_0) 
    &= \log \int p(\x_T)\prod_{t=1}^{T} p^\thetab(\x_{t-1}\cond \x_t)\, \diff \x_{1:T} \\
    &= \log \int q^{\phib}(\x_{1:T}\cond \x_0) \frac{p(\x_T)\prod_{t=1}^{T} p^\thetab(\x_{t-1}\cond \x_t)}{q^{\phib}(\x_{1:T}\cond \x_0)}\, \diff \x_{1:T}\\
    &= \log \int q^{\phib}(\x_T\cond \x_0)\prod_{t=2}^{T} q^{\phib}(\x_{t-1}\cond \x_t,\x_0) \frac{p(\x_T)\prod_{t=1}^{T} p^\thetab(\x_{t-1}\cond \x_t)}{q^{\phib}(\x_T\cond \x_0)\prod_{t=2}^{T} q^{\phib}(\x_{t-1}\cond \x_t,\x_0)}\, \diff \x_{1:T}\\
    &\geq \E_{q^{\phib}(\x_{1:T}\cond \x_0)} \left[\log \frac{p(\x_T)}{q^{\phib}(\x_T\cond \x_0)} + \sum_{t=2}^{T} \log \frac{p^\thetab(\x_{t-1}\cond \x_t)}{q^{\phib}(\x_{t-1}\cond \x_t,\x_0)} + \log p^\thetab(\x_0\cond \x_1)\right] \\
    &= -\KL\bigl(q^{\phib}(\x_T \cond \x_0)\,\|\,p(\x_T)\bigr) 
    - \sum_{t=2}^{T}\E_{q^{\phib}(\x_{t}\cond \x_0)}\Bigl[\KL\bigl(q^{\phib}(\x_{t-1}\cond \x_t,\x_0)\,\|\,p^{\thetab}(\x_{t-1}\cond \x_t)\bigr)\Bigr] \nonumber \\
    &\phantom{= } \; + \E_{q^{\phib}(\x_1 \cond \x_0)}\left[\log p^\thetab(\x_0 \cond \x_1)\right] \;.
    \label{eq:elbo_kl_form}
\end{align}

Hierarchical VAEs substantially improved over earlier latent-variable models, e.g., for image generation and density estimation, and already shared several structural features with diffusion models. Yet they did not trigger a comparable breakthrough in generative modeling. Several challenges limited their impact:
\begin{itemize}[label={}]
    \item (–) \emph{Costly simulation during training.} Sampling from the marginals
    \begin{align*}
        q^{\phib}(\x_t \cond \x_0) = \int q^{\phib}(\x_T \cond \x_0) \prod_{\tau=t+1}^{T} q^{\phib}(\x_{\tau-1}\cond \x_{\tau}, \x_0) \; \dx_{t+1:T} \ ,
    \end{align*}
    requires repeatedly applying the multi-step inference model. This significantly increases computational cost and can lead to numerical instabilities when backpropagating through long stochastic chains.

    \item (–) \emph{Non-identifiability.}  There are infinitely many sequences of intermediate marginals $\{q^{\phib}(\x_t)\}_{t=1}^T$ that interpolate between the data distribution $\qdata(\x_0)$ and the prior $p(\x_T)$. Consequently, many forward--backward process pairs are equally valid. This lack of identifiability can make optimization unstable, as the model may drift toward different, equally plausible latent trajectories.

    \item (–) \emph{Simple parametric forms.}  
    Although increasing the number of latent steps boosts expressiveness, it remained unclear whether both conditionals---$q^{\phib}(\x_{t-1}\cond \x_t,\x_0)$ in the inference model and $p^\thetab(\x_{t-1}\cond \x_t)$ in the generative model---could be adequately captured by simple parametric families.
\end{itemize}

In this context, \citet{sohl-dickstein_deep_2015} made the key innovation of defining the forward (encoding) process $q(\x_t\cond \x_{t-1})$ to be a \textit{fixed diffusion process}. For continuous spaces, this meant Gaussian transitions with linear structure; the discrete-space analogue was later developed \cite{song_denoising_2022,hoogeboom_argmax_2021,austin2021structured,campbell_continuous_2022}. This fixed forward process bears the following benefits:
\begin{itemize}[label={}]
    \item (+) \textit{Avoiding simulation during training.} Although diffusion processes generally require expensive simulations, specific families admit closed-form marginals. For continuous spaces, Gaussian processes with linear drift (Ornstein--Uhlenbeck processes) can be integrated analytically. For discrete spaces, certain transition matrices (e.g., uniform or absorbing processes) also yield tractable marginals. Thus, sampling from $q(\x_t \cond \x_0)$ can be performed directly without simulating multiple transition kernels $q(\x_t\cond \x_{t-1})$.
    \item (+) \textit{Identifiability.} By fixing the forward process, the reverse process becomes fully determined. This is clear even with two marginals: given $\qdata(\x)$ and $p(\z)$, there exist infinitely many conditionals $p^{\thetab}(\x\cond \z)$ such that $\qdata(\x) = \int p^{\thetab}(\x\cond \z)p(\z)\, \diff\z$. However, once we fix $q(\z\cond \x)$, Bayes' rule determines $p^{\thetab}(\x\cond \z) = q(\z\cond \x)\qdata(\x)/p(\z)$.
    \item (+) \textit{Reverse process shares the forward structure.} From the theory of stochastic processes, we know that time-reversing a diffusion process yields another process of the same family \citep{feller1949theory,anderson1982reverse}. For continuous spaces, reversing an SDE yields an SDE. For discrete spaces, reversing a CTMC yields another CTMC. This mathematical guarantee enables simple parametric modeling of the reverse process.
\end{itemize}

Diffusion models revolutionised the field of generative modeling, but not immediately upon their introduction by \citet{sohl-dickstein_deep_2015}. It took substantial engineering efforts, architecture optimisation, and hyperparameter tuning to make diffusion models practical \citep{ho_denoising_2020}. For discrete spaces, the development followed a similar trajectory, with early theoretical work \cite{song_denoising_2022,austin2021structured,campbell_continuous_2022} followed by practical advances. However, the details of these engineering efforts and design choices are beyond the scope of the current manuscript.

\paragraph{Modern formulation: A three-step recipe.}
Having traced the evolution from VAEs through hierarchical VAEs to diffusion models, we can now distill the key insights into a unified recipe. The innovations introduced by diffusion models---fixing the forward process to enable tractable marginals, leveraging time-reversal structure, and optimising the ELBO with reduced variance---together form the foundation of the modern approach.

The modern perspective can be distilled into a three-step recipe that applies universally across continuous and discrete state spaces, and across discrete and continuous time formulations:
\begin{enumerate}
    \item \textbf{Define the forward process:} Choose a Markov process that gradually corrupts data $\x_0 \sim \qdata$ into noise $\x_T \sim p_{\text{noise}}$. This process should have tractable marginals $q(\x_t \cond \x_0)$ and, ideally, tractable posteriors $q(\x_{t-1} \cond \x_t, \x_0)$.
    
    \item \textbf{Define the reverse process:} Parameterise a learned reverse process $p^\thetab(\x_{t-1} \cond \x_t)$ that denoises from $\x_T$ to $\x_0$. The parameterisation should leverage the structure of the forward process (e.g., Gaussian for continuous spaces, categorical for discrete spaces).
    
    \item \textbf{Define the training objective:} Maximise the ELBO (or equivalently minimise the variational upper bound), typically by matching the learned reverse transitions to the true posteriors: 
    \begin{align}
        \mathcal{L}^\thetab(\x_0) = \sum_{t=2}^{T}\E_{q(\x_{t}\cond \x_0)}\left[\KL(q(\x_{t-1}\cond \x_{t},\x_0) \| p^\thetab(\x_{t-1}\cond \x_t))\right]\,.
    \end{align}
\end{enumerate}

This recipe underpins all diffusion models. In the following, we first analyse the discrete-time formulation (\cref{sect:discrete_time_diffusion}), then show how the continuous-time formulation arises as the limit of the discrete case (\cref{sect:discrete-to-continuous-time-csp}), and finally present the full continuous-time framework (\cref{sect:continuous-diffusion}). We then derive objective functions corresponding to maximum-likelihood learning of these models (\cref{sect:ELBO}). Both continuous and discrete state spaces are treated throughout, and a later section unifies these settings using the framework of infinitesimal generators of stochastic processes (\cref{sect:generator-perspective}).

\section{Discrete-time introduction to diffusion} \label{sect:discrete_time_diffusion}

Before delving into the continuous-time theory of diffusion models, it is useful to start with an intuitive understanding of diffusion in discrete time \cite{ho_denoising_2020}. Although researchers accustomed to SDE formulations might find this approach less direct, we adopt this perspective, as it provides a direct analogy to the first discrete state-space formulations of diffusion models. Readers specifically interested in the continuous-time formulation can move on to \cref{sect:continuous-diffusion}.

To define a diffusion model, one starts with clean data $\x_0$, drawn from an unknown complex distribution $q_0 \coloneqq \qdata$ that we want to learn. This distribution is gradually transformed through a sequence of $T$ small incremental ``noising'' steps into a distribution $q_T \approx p_\text{noise}$ where $p_\text{noise}$ is a fixed known noise distribution, easy to sample from.
Each noisy version of the data is represented by a random variable and its corresponding probability distribution.
Formally, we define a \textit{Markov Chain}\footnote{A Markov chain is a sequence of states where the probability of the next event depends only on the current event, not on the past.} $\{ \x_t \}_{ t \in \{0, \dots, T\}}$ evolving in a state-space $\calX^d$, where $\calX = \RR$ in the continuous setting and  $\calX=\{1, \dots, K\}$ in the discrete setting. This process is governed by a \emph{transition distribution} $q_{t\cond t-1}$, often referred to as the \emph{noising kernel}. In the following, the subscripts of probability distributions will sometimes be omitted when unambiguous, given the arguments, e.g., we will denote $q_{t|t-1}(\x_t \cond \x_{t-1})$ by $q(\x_t \cond \x_{t-1})$.

The interpretation of this probability distribution depends on the state space. In continuous spaces or ordinal-valued spaces (e.g., images), it represents the probability density (e.g., over pixel values). In contrast, in categorical discrete spaces (e.g., protein sequences or text), it captures the probability mass (e.g., amino-acid types or token values).

Crucially, these noising \textit{transitions} are most often applied \textit{independently across dimensions}: each pixel in an image is perturbed independently of others, and each amino acid position can change to any other type regardless of modifications at other positions.
The path distribution is thus defined as the product of the conditional distributions:
\begin{align}
q\left(\x_{0:T}\right) &= \qdata\left(\x_0\right)\prod_{t=1}^T q\left(\x_{t} \cond \x_{t-1}\right) %
\\
&= \qdata\left(\x_0\right)\prod_{t=1}^T \prod_{k=1}^d q\left(x^{(k)}_{t} \cond x^{(k)}_{t-1} \right) 
\label{eq:dimensional_factorisation}
\end{align}

However, this factorisation does not generally extend to the reverse (denoising) process. The reverse process requires modeling complex dependencies across dimensions to reconstruct coherent data. For instance, in face generation, pixel correlations encode spatial relationships between facial features; in sequence modeling, positional dependencies capture syntactic and semantic structure.

For now, we focus on the construction of the forward process. The key takeaway is that, due to its conditional independence across dimensions, we can formulate and analyse the forward process directly in $\calX$ instead of $\calX^d$. We will revisit the structure of the reverse process in \cref{sect:reverse-process-dt}.

\subsection{Forward process}
\label{sect:forward-noising-process}

In this section, we develop the discrete-time forward noising process, where clean data $\x_0 \sim \qdata$ is progressively corrupted through $T$ steps until reaching a final distribution $q_T \approx p_{\text{noise}}$ at time $t=T$. A first unifying perspective arises from expressing the conditional distribution at time $t$ as a \textit{convex combination} of \textit{signal} and \textit{noise}. In continuous spaces, this corresponds to additive Gaussian noise, whereas in discrete spaces, it corresponds to categorical resampling.

The process $\{\x_t\}_{t \in \{0,\ldots,T\}}$ is governed by one-step transition distributions $q(\x_t \cond \x_{t-1})$ that are applied \textit{independently across dimensions} (see \cref{eq:dimensional_factorisation}). The one-step transitions take the following forms:

\continuous[One-step transition]{In continuous Euclidean spaces,\footnote{In non-Euclidean spaces, see the following literature for more details \citet{mathieu2023geometric,huang2022riemannian,chen2023flow}.} $x_t \in \RR$ (or $\x_t \in \RR^d$ under i.i.d.\ noise across dimensions), the \textit{transition distribution} is typically defined as a simple Gaussian, due to its analytical tractability and empirical performance:
\begin{align}
   q\left(x_t \cond x_{t-1}\right) = \mathcal{N}\left(x_t \ ; \ \alpha_{t \cond t-1} x_{t-1},\sigma_{t \cond t-1}^2 \right),
   \quad
   q(\x_t \cond \x_{t-1}) = \mathcal{N}(\x_t \ ; \ \alpha_{t \cond t-1} \x_{t-1}, \sigma_{t \cond t-1}^2 \I),
   \label{eq: forward transition distribution continuous diffusion}
\end{align}
where $\alpha_{t \cond t-1}$ and $\sigma_{t \cond t-1}$ are scalar, time-dependent parameters.}

\discrete[One-step transition]{In discrete state spaces, $x_t \in \{1,\ldots,K\}$, the \emph{transition distribution} is defined by a \emph{categorical distribution} with transition matrices $\Q_{t \cond t-1} \in [0,1]^{K \times K}$:
\begin{align}
    q(x_t\cond x_{t-1})=\Cat \left(x_t;\q = \Q_{t \cond t-1}\,\mathbf{e}_{x_{t-1}}\right), \qquad [\Q_{t \cond t-1}]_{ij} = q(x_t = i \cond x_{t-1} = j),
    \label{eq:def-transition-matrix}
\end{align}
where $\q\in \Delta^{K-1}$ and $\e_{x_{t-1}} \in \{0,1\}^K$ is the column one-hot vector encoding of $x_{t-1}$.\footnote{The multi-dimensional extension to $\calX^d$, obtained by applying the same transition kernel independently to each coordinate, is given by \cref{eq:seq-transition-matrix}.}

\emph{Indexing note:} We place the destination state in the row index (common in machine learning), reversing the usual Markov-chain convention.}

By recursively applying these transitions, we obtain the multi-step transition distributions. For all $s,t \in \{0,\ldots,T\}$ with $s<t$:
\begin{align}
q\left(x_t \cond x_{s}\right) &= \mathcal{N}\left(x_t \ ; \ \alpha_{t|s} x_{s}, \sigma_{t|s}^2 \right),
\qquad
q(x_t\cond x_s) = \Cat \left(x_t;\mathbf{Q}_{t|s}\mathbf{e}_{x_s}\right),
    \label{eq:t|s-dt-dp-transition}
\end{align}
where $\alpha_{t\cond s}\coloneqq\prod_{i=s+1}^t\alpha_{i \cond i-1}$, $\sigma_{t \cond s}^2=\sigma_t^2-\alpha_{t \cond s}^2 \sigma_s^2$ (continuous), and $\mathbf{Q}_{t|s}\coloneqq \prod_{i=s+1}^{t}\Q_{i\cond i-1}$ (discrete). We write $\alpha_t \coloneqq \alpha_{t|0}$ and $\Q_t \coloneqq \Q_{t|0}$ for brevity.

Importantly, the multi-step transitions in \cref{eq:t|s-dt-dp-transition} specialise to a closed-form expression of $q(\x_t \cond \x_0)$ when $s=0$, expressing the noisy state as a \emph{convex combination} of clean data and noise (proof \cref{apx:interpolation-marginals}).

\continuous[Gaussian forward transition]{
For the standard parameterisation of Gaussian diffusion, the transition from time $0$ to $t$ is
\begin{align}
   q(x_t \cond x_0) 
    =  \mathcal{N}\Bigl(x_t \, ; \,\alpha_t\, x_0,\;\sigma_t^{2}\Bigr)
   \quad\Leftrightarrow\quad
   x_t 
    =  \alpha_t\, x_0 \;+\; \sigma_t \,\epsilon\ ,
   \label{eq:interpolation-continuousdiff}
\end{align}
where $\epsilon \sim \mathcal{N}(0,1)$ (or $\epsilonb\sim\mathcal N(\mathbf 0,\I)$ for $\x_t \in \RR^d$), and the coefficients are (proof \cref{apx:interpolation-marginals-continuous-diff}):
\begin{align}
   \alpha_t 
    \coloneqq  \alpha_{t \cond t-1} \alpha_{t-1}, 
   \quad\text{and}\quad
   \sigma_t^{2} 
    \coloneqq  \sigma^2_{t \cond t-1} + \alpha^2_{t \cond t-1} \sigma^2_{t-1}
\end{align}
}

\begin{remark}[Noise schedules]\label{remark:csp_dt_noise_schedules}
The time-dependent pair of coefficients $(\alpha_t,\sigma_t)$ 
is often referred to as the \emph{noise schedule}, as it governs the balance between signal and noise over time.  For instance, to approximately map clean data to a standard Gaussian, we require $\alpha_0 =1,\ \sigma_0=0$ and $\alpha_T\!\approx\!0,\ \sigma_T\!\approx\!1$. For common choices of the schedule and further discussion on their relations, we refer to \citet{kingma2023understanding}.
\end{remark}

\discrete[Categorical forward transition]{
The standard parameterisation of the transition matrices forms a convex combination of staying in the current state and resampling from a noise distribution $\pnoise\in\Delta^{K-1}$:\footnote{This parameterisation is common in the discrete flows literature \cite{gat_discrete_2024,campbell_generative_2024}, paralleling the continuous-state convex combination.}
\begin{align}
\Q_{t \cond t-1}=\alpha_{t \cond t-1}\,\I+(1-\alpha_{t \cond t-1})\,\pnoise\,\mathbf 1^\top,
\label{eq:interpolation-discretediff}
\end{align}
where $\alpha_{t \cond t-1}\in [0,1]$ is the time-dependent weight. Consequently, the transition distribution from time $0$ to $t$ is
\begin{align}
q(x_t\cond x_0)=\Cat\bigl(x_t;\, \q = \alpha_t\,\e_{x_0}+(1-\alpha_t)\,\pnoise\bigr)\ ,
\label{eq:forward-transition-interpolation-discretediff}
\end{align}
where $\alpha_t \coloneqq \prod_{j=1}^{t}\alpha_{j \cond j-1}$ is the cumulative attenuation (proof in \cref{apx:interpolation-marginals-discrete-state}).
}

\begin{proof}[Derivation sketch:]
Using $(\pnoise\mathbf{1}^{\top})(\pnoise\mathbf{1}^{\top})=\pnoise\mathbf{1}^{\top}$, one obtains
$\mathbf{Q}_{t|t-2}
= \alpha_{t \cond t-1}\alpha_{t-1\cond t-2}\mathbf{I}
+ \bigl(1-\alpha_{t \cond t-1}\alpha_{t-1\cond t-2}\bigr)\pnoise\mathbf{1}^{\top}$, and iterating gives the general form above.
By the Woodbury identity, $\mathbf{Q}_t$ is invertible and has the following expression (proof \cref{apx:multistep-woodburry-identity}):
\begin{align*}
\mathbf{Q}_t^{-1}
= \frac{1}{\alpha_t}\Bigl(\mathbf I - (1-\alpha_t)\,\pnoise\,\mathbf{1}^\top\Bigr),
\qquad
\mathbf{Q}_{t|s}=\mathbf{Q}_t\mathbf{Q}_s^{-1}
= \tfrac{\alpha_t}{\alpha_s}\mathbf{I}
+ \Bigl(1-\tfrac{\alpha_t}{\alpha_s}\Bigr)\pnoise\mathbf{1}^{\top}.
\end{align*}
\end{proof}

\begin{remark}[Discrete schedules \& factorised $d$-dimensional process\label{eq:seq-transition-matrix}]
\leavevmode\\[-1.5em]
\begin{enumerate}
\item The attenuation $\alpha_t$ plays the role of a signal coefficient: $\alpha_0=1$ (no noise), and choosing $\alpha_T\!\approx\!0$ drives $q_T\simeq p_{\mathrm{noise}}$.
\item When the same transition matrix $\Q_t$ is applied independently to every coordinate, the transition matrix over sequences in $\calX^d$ is the Kronecker product 
\begin{align*}
    \Qseq_t = \Q_t^{\otimes d} \in \mathbb{R}^{K^d \times K^d}, \qquad q(\x_t \mid \x_0) = \prod_{k=1}^d q(x_t^{(k)}\mid x^{(k)}_0),
\end{align*}
where $\otimes$ corresponds to the Kronecker product. The unconditional marginal $q_t$ need not factorise because $\qdata$ may couple coordinates. Due to its exponential size, the full $K^d\times K^d$ matrix is rarely materialised in practice; the forward kernel is implemented through its coordinatewise factorisation.
\end{enumerate}
\end{remark}

\discrete[Absorbing state and uniform diffusion]{
Two widely used reference noise distributions in discrete space are:

\emph{Absorbing (masking) process} \cite{austin2021structured,chang_maskgit_2022,sahoo_simple_2024,he_diffusionbert_2022,shi_simplified_2025}:  
The state space is augmented with a special [MASK] token, giving $K{+}1$ total categories. At each step, each token has a probability of being replaced by [MASK]. The noise distribution is $\pnoise=\mathbf{e}_{\text{[MASK]}}$.

\emph{Uniform diffusion} \cite{hoogeboom_argmax_2021,lee_deterministic_2018,austin2021structured, campbell_continuous_2022,esser_imagebart_2021, savinov_step-unrolled_2022}:  
Each token transitions uniformly. One may use the uniform law over $\mathcal X$ with $\pnoise=\frac{1}{K}\mathbf{1}$, or, when sharing the alphabet with masking, the uniform law over $\mathcal X_+ \coloneqq \mathcal X\cup\{\text{[MASK]}\}$ with $\pnoise=\frac{1}{K+1}\mathbf{1}$.
}

In contrast to the continuous Euclidean case, where the isotropic Gaussian reference distribution is the standard choice for analytical convenience, discrete diffusion offers a broader design space in terms of noising strategies.  
The absorbing state process is particularly effective in text modeling, where masking aligns naturally with pretraining objectives of Masked Language Models \citep{devlin_bert_2018}, see \cref{eq:mlm-ELBO}. 
Uniform diffusion encourages exploration across all token categories.  
Beyond these standard choices, discretised Gaussian kernels \cite{austin2021structured,campbell_continuous_2022} and non-Markovian diffusion processes \cite{song_denoising_2022,wang2025remasking} have also been explored.

\discrete[Combining noise processes]{
The interpolation formulation allows mixtures of different reference noise distributions. For example, we can combine uniform diffusion and masking in a single transition \cite{gu_vector_2022,von2025generalized,gat_discrete_2024}:
\begin{align}
\mathbf{Q}_t
&= \alpha^1_t\,\mathbf{I}
+ \alpha^u_t\,\pnoise^{u}\,\mathbf{1}^\top
+ \alpha^m_t\,\pnoise^{m}\,\mathbf{1}^\top,
\label{eq:mixture-dtMC}
\end{align}
where $\pnoise^{u}=\frac{1}{K+1}\mathbf{1}$ denotes the uniform distribution on the shared alphabet (including [MASK]) and $\pnoise^{m}=\mathbf{m}$ is the [MASK]-only distribution. The coefficients satisfy $\alpha_t^i \ge 0$ and $ \alpha^1_t+\alpha^u_t + \alpha^m_t = 1$ for column-stochasticity. Equivalently, set
\begin{align}
\alpha_t\coloneqq \alpha_t^1\;,\quad \beta_t\coloneqq 1-\alpha_t\;,\quad
\p_t \coloneqq
\begin{cases}
\dfrac{\alpha_t^u}{\beta_t}\,\pnoise^{u}+\dfrac{\alpha_t^m}{\beta_t}\,\pnoise^{m} & \beta_t>0\;,\\[6pt]
\text{arbitrary in }\Delta^{K}& \beta_t=0\;,
\end{cases}
\end{align}
so that $\mathbf{Q}_t=\alpha_t\,\mathbf I+\beta_t\,\p_t\,\mathbf 1^\top$ is defined with the minimal assumptions $\alpha_t\in[0,1]$, $\p_t\in\Delta$ \;\cite{von2025generalized}.
}

Here, the identity component $\alpha^1_t$ preserves tokens, the uniform term $\alpha^u_t$ replaces tokens with random categories, and the masking term $\alpha^m_t$ replaces tokens with [MASK].  
Other mixtures can be defined analogously by combining additional reference distributions. 

Finally, some discrete diffusion processes modify not only token identities but also sequence length. 
One example, inspired by Levenshtein Transformers \cite{gu2019levenshtein,ruis2020insertion}, is based on \emph{Levenshtein edit operations} \cite{johnson_beyond_2021,reid_diffuser_2022}, which allow insertion, deletion, and replacement of tokens during the forward process.  
This enables the model to generate sequences with variable length, rather than being restricted to fixed-length substitution. Subsequent work that further advances variable-length generation includes \cite{kim2025any,havasi2025edit,baron2025diffusion}.

\subsection{Reverse process\label{sect:reverse-process-dt}}

Once the forward diffusion process has fully converted data into noise, the goal is to recover its clean version from the noise state by learning a reverse or denoising process. This process can then be used to generate new data.
In principle, the generative process that reverses the forward noising is a Markov chain that starts from the terminal distribution $q_T \approx p_{\mathrm{noise}}$ and evolves according to the exact reverse transitions (Bayes' rule):
\begin{align*}
q(\x_{t-1}\cond \x_t) = \frac{q(\x_t\cond \x_{t-1})\,q(\x_{t-1})}{q(\x_t)},
\end{align*}
whenever $q(\x_t)>0$. In practice, both $q_T$ and the exact reverse transitions are unknown or expensive. We therefore define a feasible denoising chain that starts from a tractable reference $p_T\coloneqq p_{\mathrm{noise}}$ and uses learnable reverse transitions $p^{\thetab}(\x_{t-1}\cond \x_t)$ to approximate $q(\x_{t-1}\cond \x_t)$. 

As the standard per-coordinate factorisation does not hold for the \textit{reverse process}, we adopt a factorised approximation of the fully coupled conditional for efficiency, and write  $p^{\thetab}(\x_{t-1}\cond \x_t)\coloneqq \prod_{k=1}^d p^{\thetab}(x_{t-1}^{(k)}\cond \x_t)$. When formulas are dimension-agnostic, we drop the superscript $k$. The model induces a trajectory law
\begin{align}
p^{\thetab}(\x_{0:T})  =  p_T(\x_T)\,\prod_{t=1}^{T}p^{\thetab}(\x_{t-1}\cond \x_t)\;,\qquad \x_T\sim p_T\coloneqq p_{\mathrm{noise}}\;, \label{eq:reverse_p_theta_dt}
\end{align}
and aims for $p^\thetab_{0} \approx q_0\coloneqq \qdata$.

\subsubsection{Reverse transitions}
While the true reverse transition distributions $q(\x_{t-1} | \x_t)$ are generally intractable, they have a closed-form expression when conditioning on the clean data $\x_0$ (proof in \cref{apx:closed-form-reverse-kernel}):
\begin{align*}
q(\x_{t-1}\cond \x_t,\x_0)  =  \prod_{k=1}^d q\!\left(x_{t-1}^{(k)} \cond x_t^{(k)}, x_0^{(k)}\right),
\end{align*}
with continuous and discrete instances below. 

\continuous[Reverse Gaussian transitions]{
For continuous Gaussian diffusion, the conditional reverse transition can be expressed as:
\begin{align}
q(x_{t-1}\cond x_t,x_0) &= \mathcal{N}\!\bigl(x_{t-1} \,;\, \mu_{t-1|t}(x_t,x_0),\, \sigma_{t-1|t}^2 \bigr)\;,
\label{eq:discrete-time-reversal-cont}
\end{align}
where the reverse mean and variance are:
\begin{align}
\mu_{t-1|t}(x_t,x_0)  \coloneqq  \frac{\alpha_{t|t-1}\,\sigma_{t-1}^2}{\sigma_t^2}\,x_t \;+\; \frac{\alpha_{t-1}\,\sigma_{t|t-1}^2}{\sigma_t^2}\,x_0\;,
\quad
\sigma_{t-1|t}  \coloneqq  \sigma_{t|t-1}\,\frac{\sigma_{t-1}}{\sigma_t}\;.
\label{eq:discrete_time_reverse_mu_x0}
\end{align}
}

\discrete[Reverse categorical transitions]{
In discrete state-spaces, conditioning on the clean data yields (proof in \cref{apx:closed-form-reverse-kernel}):
\begin{align}
q(x_{t-1}\cond x_t,x_0) &= \Cat \left(x_{t-1} \,;\, \frac{\Q_{t|t-1}^\top \e_{x_t} \odot \Q_{t-1}\e_{x_0}}{\e_{x_t}^\top \Q_t \e_{x_0}}\right)
\label{eq:discrete-time-reversal}
\end{align}
where $\odot$ denotes element-wise (Hadamard) multiplication and $\Q_{t|t-1}$ is the one-step transition matrix introduced in \cref{eq:t|s-dt-dp-transition}.

Further simplifying \cref{eq:discrete-time-reversal} under the interpolation parameterisation of $\Q_t$ (\cref{eq:interpolation-discretediff}) yields
\begin{align}
q(x_{t-1} \cond x_t, x_0)
= \Cat \left(
x_{t-1} \,;\,
\tfrac{ \bigl[\alpha_{t|t-1}\e_{x_t} + (1-\alpha_{t|t-1})\,\mathbf{1}\,\pnoise^\top \e_{x_t}\bigr]
\odot \bigl[\alpha_{t-1}\e_{x_0} + (1-\alpha_{t-1})\,\pnoise\bigr]}
{\alpha_t \,\e_{x_t}^\top \e_{x_0} + (1-\alpha_t)\,\e_{x_t}^\top \pnoise}
\right).
\label{eq:discrete-time-reversal-interpolationview}
\end{align}
This form is convenient for processes such as masking, where $\pnoise$ is concentrated on a single [MASK] state.}

Using \cref{eq:interpolation-discretediff}, \citet{zheng_reparameterized_2024,zhao_unified_2024} show that \cref{eq:discrete-time-reversal-interpolationview} can be reparameterised by splitting the cases $x_t=x_0$ vs.\ $x_t\neq x_0$:
\begin{align}
q\!\left(x_{t-1} \cond x_t, x_0\right)=
\begin{cases}
(1-\lambda_{t|t-1})\,\e_{x_t}+\lambda_{t|t-1}\,\pnoise, & x_t=x_0,\\[2pt]
(1-\mu_{t|t-1})\,\e_{x_0}+ \mu_{t|t-1}\,\alpha_{t|t-1}\,\e_{x_t} + \mu_{t|t-1}(1-\alpha_{t|t-1})\,\pnoise, & x_t\neq x_0,
\end{cases}
\label{eq:reversal-interpolation-zheng}
\end{align}
with
\begin{align*}
\lambda_{t|t-1} \coloneqq \frac{(1-\alpha_{t-1})(1-\alpha_{t|t-1})\,\e_{x_t}^\top \pnoise}
{\alpha_t+(1-\alpha_t)\,\e_{x_t}^\top \pnoise},
\qquad
\mu_{t|t-1} \coloneqq \frac{1-\alpha_{t-1}}{1-\alpha_t}.
\end{align*}

\discrete[Simplified reverse transition for masked diffusion]{
For masked diffusion with $\Q_t=\alpha_t \I + (1-\alpha_t)\,\e_{\text{[MASK]}}\,\mathbf{1}^\top$, noticing $\e_{x_t}^\top \e_{\text{[MASK]}}=1$ if $x_t=\text{[MASK]}$ and 0 otherwise, \cref{eq:reversal-interpolation-zheng} reduces to \cite{sahoo_simple_2024,shi_simplified_2025}:
\begin{align}
q(x_{t-1} \cond x_t, x_0) =
\begin{cases}
\Cat \left(x_{t-1} \,;\, \e_{x_t}\right), & x_t \neq \text{[MASK]},\\[4pt]
\Cat \left(x_{t-1} \,;\, \dfrac{(1-\alpha_{t-1}) \e_{\text{[MASK]}} + (\alpha_{t-1}-\alpha_t)\,\e_{x_0}}{1-\alpha_t}\right), & x_t = \text{[MASK]}.
\end{cases}
\label{eq:discrete-time-reversal-maskdiff}
\end{align}

The case distinction enables simple and efficient implementations in practice.
}

\subsubsection{Parameterisation}

We approximate $q(\x_{t-1}\cond \x_t)$ with a parametric distribution whose parameters are predicted by a neural network. Specifically, we define:
\begin{align}
    p^{\thetab}(\x_{t-1}\cond \x_t)
    = \prod_{k=1}^d p^{\thetab}(x^{(k)}_{t-1}\cond \x_t)
    = \prod_{k=1}^d \mathcal{F}\bigl(x^{(k)}_{t-1}\,;\,\eta^{(k)}_{\thetab}(\x_t,t)\bigr)\, ,
    \label{eq:dt-reversal-factorisation}
\end{align}
where $\mathcal{F}$ denotes a chosen family of probability distributions (e.g., Gaussian in the continuous case, categorical in the discrete case) and
$\etab_{\thetab}(\x_t,t)$ denotes the parameters of $\mathcal{F}$ obtained from the output of a neural network with parameters~$\thetab$. The parameters $\eta^{(k)}_{\thetab}(\x_t,t)$ are used for the $k$-th factor \cite{sohl-dickstein_deep_2015,ho_denoising_2020,austin2021structured}. Although the forward process is typically factorised across dimensions, the reverse process must reconstruct complex correlations present in the clean data distribution $\qdata$. The network allows for capturing these dependencies by leveraging the entire structure of $\x_t$ at each denoising step. Consequently, conditioning on the \emph{full} corrupted state $\x_t$ is essential.

For continuous state spaces, we write the reverse transition directly for the full state. It can equivalently be written componentwise, omitting the coordinate superscript $^{(k)}$, as in the discrete equations below, where unbolded state variables refer to the same arbitrary coordinate $k$.

\continuous[Direct neural parameterisation]{
In continuous Euclidean spaces, the reverse transition \cref{eq:dt-reversal-factorisation} is defined as a Gaussian with mean parameterised by a neural network, and with the same covariance as $q(\x_{t-1}\cond \x_t,\x_0)$ (\cref{eq:discrete_time_reverse_mu_x0}):
\begin{align}
    p^\thetab(\x_{t-1}\cond \x_t)
    \coloneqq \mathcal{N}\Bigl(
        \x_{t-1}\ ;\
        \boldsymbol{\mu}^\thetab(\x_t,t),\
        \sigma_{t-1\cond t}^2\mathbf{I}_d
    \Bigr)\, .
\label{eq:condition-reversal-nn-cont}
\end{align}
}

\discrete[Direct neural parameterisation]{
In discrete categorical spaces, the single-dimension reverse transition \cref{eq:dt-reversal-factorisation} is defined as a categorical distribution with probabilities predicted by a neural network:
\begin{align}
p^\thetab(x_{t-1}\cond \x_t)
\coloneqq \Cat\Bigl(x_{t-1}\ ;\ \etab^\thetab(\x_t,t)\Bigr)\, .
\label{eq:condition-reversal-nn-disc}
\end{align}
}

Observing that the only unknown in \cref{eq:discrete-time-reversal,eq:discrete-time-reversal-cont} is $\x_0$, many approaches instead train a neural network to estimate the clean state $\x_0$ ($\x_0^\thetab \coloneqq \mathrm{NN}^\thetab(\x_t,t)$), and then substitute this estimate in the closed-form expression for the reverse distribution. One can also manipulate the expression \cref{eq:discrete-time-reversal,eq:discrete-time-reversal-cont} and parameterise other functions of $\x_0$ \cref{eq:score-mean-eps-parameterisation}.
When a reverse transition is written for a single coordinate, the corresponding coordinate of this network output is used.

\continuous[Clean-data prediction and equivalences]{
Instead of predicting the mean directly, one can parameterise the denoising process in terms of the predicted clean state $\x_0^\thetab$:
\begin{align}
p^\thetab(\x_{t-1}\cond \x_t)
= \mathcal{N}\!\left(
\x_{t-1}\ ;\
\frac{\alpha_{t|t-1}\sigma_{t-1}^2}{\sigma_t^2}\,\x_t
+\frac{\alpha_{t-1}\sigma_{t|t-1}^2}{\sigma_t^2}\,\x_0^\thetab,
\ \sigma_{t-1|t}^2\mathbf{I}_d
\right).
\label{eq:reversal-x0-cond-cont}
\end{align}

Since the covariance is diagonal,
\cref{eq:reversal-x0-cond-cont} is equivalent to the product in
\cref{eq:dt-reversal-factorisation}. Suppressing the coordinate superscript
\(^{(k)}\), each factor is
\begin{align*}
p^\thetab(x_{t-1}\cond \x_t)
= \mathcal{N}\!\left(
x_{t-1}\ ;\
\frac{\alpha_{t|t-1}\sigma_{t-1}^2}{\sigma_t^2}\,x_t
+\frac{\alpha_{t-1}\sigma_{t|t-1}^2}{\sigma_t^2}\,x_0^\thetab,
\ \sigma_{t-1|t}^2
\right).
\end{align*}

This can be rewritten in the standard noise-, score-, and $v$-parameterisations via
\begin{align}
    \x_0
    = \frac{\x_t-\sigma_t\epsilonb}{\alpha_t}
    = \frac{\sigma_t^2\,\s_t+\x_t}{\alpha_t}
    = \frac{\alpha_t\x_t-\sigma_t\mathbf{v}_t}
           {\sigma_t^2+\alpha_t^2}
\label{eq:score-mean-eps-parameterisation}
\end{align}
where $\s_t \coloneqq \nabla_{\x_t}\log q(\x_t\cond \x_0)=-\epsilonb/\sigma_t$ is the score and $\mathbf{v}_t \coloneqq \alpha_t\epsilonb-\sigma_t\x_0$ is the $v$-parameterisation}\cite{ho_imagen_2022}.

The equivalence in \cref{eq:score-mean-eps-parameterisation} links \emph{four} standard parameterisations:
(i) \textit{clean data} via $\x_0^\thetab$,
(ii) \textit{noise} via $\epsilonb^\thetab$,
(iii) \textit{score} via $\s_t^\thetab$, and
(iv) \textit{$v$-prediction} via $\mathbf{v}_t^\thetab$.
For each parameterisation, the neural network $\mathrm{NN}^\thetab(\x_t,t)$ predicts the corresponding quantity for the full state. These quantities are linearly related but lead to distinct training objectives and empirical behaviours \cite{karras2022elucidating}. The choice among them can significantly impact sample quality and has thus been extensively studied in the literature \cite{karras2022elucidating,gao2025diffusionmeetsflow,ho_imagen_2022}.

The same clean-data prediction approach applies to discrete state spaces.

\discrete[Clean-data prediction and equivalences]{
In the discrete case, let $\x^\thetab_0\in \Delta^{K-1}$ denote the predicted clean-token probability vector, whose $a$-th element is $x^{(a),\thetab}_0\coloneqq p^\thetab(x_0=a \cond \x_t)$.
\begin{align}
p^\thetab(x_{t-1}\cond \x_t)
\coloneqq
\Cat\!\left(
x_{t-1}\ ;
\p =
\frac{
    \Q_{t|t-1}^{\!\top}\,\e_{x_t}
    \ \odot\
    \Q_{t-1}\,\x_0^\thetab
}{
    \e_{x_t}^{\!\top}\,\Q_t\,\x_0^\thetab
}
\right).
\label{eq:reversal-x0-cond-disc}
\end{align}
Alternatively, using \cref{eq:def-transition-matrix} and marginalising \cref{eq:discrete-time-reversal} over $x_0$ yields \cite{austin2021structured}:
\begin{align}
p^\thetab(x_{t-1}\cond \x_t)
\coloneqq
q(x_t\cond x_{t-1})\,
\left(
\sum_{x_0\in\calX}
\frac{q\left(x_{t-1}\cond x_0\right)}
     {q\left(x_t\cond x_0\right)}
\,p^\thetab(x_0\cond \x_t)
\right).
\label{eq:reversal-austin-nn}
\end{align}
These equations define the reverse transition for one coordinate. The full transition is obtained by taking the product over coordinates as in \cref{eq:dt-reversal-factorisation}.
}

\begin{proof}[Derivation sketch]
We provide a derivation sketch of \cref{eq:reversal-austin-nn}; the detailed derivation is provided in \cref{apx:reversal-austin}:
\begin{align}
q\!\left(x_{t-1}^{(k)}\cond \x_t\right)
&= \sum_{\x_{t-1}^{\setminus k}}
q\!\left(\x_{t-1}\cond \x_t\right)
\nonumber \\
&= \sum_{\x_{t-1}^{\setminus k}}
q\!\left(\x_t\cond \x_{t-1}\right)\,
\frac{q\!\left(\x_{t-1}\right)}
     {q\!\left(\x_t\right)}
\nonumber \\
&= \sum_{\x_{t-1}^{\setminus k}}
q\!\left(\x_t\cond \x_{t-1}\right)\,
\sum_{\x_0}
\frac{q\!\left(\x_{t-1}\cond \x_0\right)}
     {q\!\left(\x_t\cond \x_0\right)}
q(\x_0\cond \x_t)
\nonumber \\
&= \sum_{\x_0} q(\x_0\cond \x_t)
\sum_{\x_{t-1}^{\setminus k}}
\biggl[
\frac{
q\!\left(x_t^{(k)}\cond x_{t-1}^{(k)}\right)
q\!\left(x_{t-1}^{(k)}\cond x_0^{(k)}\right)
}{
q\!\left(x_t^{(k)}\cond x_0^{(k)}\right)
}
\prod_{j\neq k}
\frac{
q\!\left(x_t^{(j)}\cond x_{t-1}^{(j)}\right)
q\!\left(x_{t-1}^{(j)}\cond x_0^{(j)}\right)
}{
q\!\left(x_t^{(j)}\cond x_0^{(j)}\right)
}
\biggr]
\nonumber \\
&=
q\!\left(x_t^{(k)}\cond x_{t-1}^{(k)}\right)
\sum_{x_0^{(k)}}
\frac{
q\!\left(x_{t-1}^{(k)}\cond x_0^{(k)}\right)
}{
q\!\left(x_t^{(k)}\cond x_0^{(k)}\right)
}
q\!\left(x_0^{(k)}\cond \x_t\right).
\label{eq:reversal-austin}
\end{align}
\end{proof}

\paragraph{Extension to arbitrary timesteps.}
Importantly, these reverse transitions, \cref{eq:discrete-time-reversal,eq:discrete-time-reversal-interpolationview,eq:reversal-interpolation-zheng,eq:discrete-time-reversal-maskdiff}, can be extended to \textit{any two arbitrary times} $s<t$ by simply modifying $t-1$ to $s$ in our expressions. A key observation for our work is that, although $s$ and $t$ have so far been considered discrete indices in $\{0,\dots,T\}$, the conditional transition formulas for both forward and reverse processes naturally extend to \emph{continuous time}, with $s,t\in[0,1]$.
We present this extension, along with the formal construction of the continuous-time diffusion, in the following section.

\section{From discrete to continuous-time diffusion \label{sect:discrete-to-continuous-time-csp}}

While the discrete-time approach offers an intuitive formulation for modeling the diffusion process, the explicit discretisation introduces some limitations and reduced flexibility. To address these limitations, we introduce the continuous-time formulation, which connects to a rich body of literature in stochastic processes and differential equations, making a broad range of analytical tools and methods applicable. We show that, in the limit as the number of noising steps $T$ tends to infinity, our discrete-time model converges naturally to its continuous-time counterpart.  Given the assumption of independence of the noise across dimensions, we first describe in this section the univariate case and later extend to the multivariate setting in \cref{sect:continuous-diffusion}.

To facilitate taking the limit $T \to \infty$, we define the normalised time points $t_i \coloneqq \frac{i}{T}$, for $i \in \{0, \dots, T \}$, such that $t_i \in [0,1]$. The discrete-time process can be equivalently indexed by these rescaled times, $\{x_{t_i} \}_{i\in \{0,\dots,T\}}$. Our goal is to consider the limit as $T \to \infty$ to study the continuous-time process $\{x_t\}_{t\in[0,1]}$. In this limit, the noise schedule parameters $\alpha_t$ and $\sigma_t$ become smooth functions over $[0,1]$, and we additionally assume that they are differentiable.

Intuitively, as $T$ grows, and for $x, y$ in the state space $\calX$ (whether continuous or discrete), the transitions $q(x_{t_i} = y \cond x_{t_{i-1}} = x)$ become increasingly peaked around $x = y$, making them insufficient to describe the limiting dynamics. What becomes meaningful is their \textit{rate} of change: how the distribution evolves instantaneously in time. We will show that it naturally leads to \textit{Stochastic Differential Equations (SDEs)} in continuous state space and \textit{Continuous-Time Markov Chains (CTMCs)} in discrete state space.

\continuous[Continuous-time limit (forward SDE)]{
In continuous state space, the forward dynamics are governed by the SDE
\begin{align}
\diff x_t = \underbrace{f_{t} \left(x_t\right) \dt}_{\text{drift term}} + \underbrace{g_t \diff w_t}_{\text{diffusion term}}, \qquad x_0 \sim \qdata
\end{align}
where $f_t$ is the drift, $g_t$ is the diffusion coefficient, and $w_t$ is a Wiener process (Brownian motion).
}

\citet{song_score-based_2021} first demonstrated that discrete-time diffusion models \cite{ho_denoising_2020} can be interpreted as the numerical approximation (i.e. time discretisation) of a continuous-time \textit{score-based} model \cite{song2019generative}. 
The connection between the SDE and the discrete-time process can be established as follows. Using the shorthand $x_{t_i} = x_t$, for a small $\Delta t \coloneqq \frac{1}{T}$, we can write the following approximation:
\begin{align}
     x_{t} - x_{t - \Delta t} 
    &= (\alpha_{t|t - \Delta t}-1) x_{t-\Delta t} + \sigma_{t|t-\Delta t} \epsilon_{t} \\
    &= \frac{\alpha_{t|t- \Delta t} - 1}{\Delta t}x_{t-\Delta t } \Delta t \,  + \frac{\sigma_{t|t- \Delta t}}{\sqrt{\Delta t}} \sqrt{\Delta t} \ \epsilon_{t}
\end{align}
By noticing that $\sqrt{\Delta t} \ \epsilon_t  = \Delta w_t$ is the infinitesimal increment of the Wiener process, we recover the usual \textit{Euler--Maruyama discretisation scheme} which converges for $\Delta t \rightarrow 0$ to the initial SDE:
\begin{align}
\underbrace{
\Delta x_t
 = 
\left[\frac{\alpha_{t\cond t-\Delta t}-1}{\Delta t}\right]\,x_{t-\Delta t}\,\Delta t
\;+\;
\left[\frac{\sigma_{t\cond t-\Delta t}}{\sqrt{\Delta t}}\right]\,\Delta w_t
}_{\xrightarrow[\Delta t \to 0]{} \; \diff x_t = f_t(x_t)\,\dt + g_t\,\diff w_t}.
\end{align}
Under some smoothness assumption over $\alpha_t , \sigma_t$ we recover:
\begin{align*}
    f_t (x_t)
    & \coloneqq \lim_{\Delta t \to 0} \frac{\alpha_{t|t-\Delta t} - 1}{\Delta t} x_{t-\Delta t}
    = \lim_{\Delta t \to 0} \frac{1}{\alpha_{t-\Delta t}} 
      \frac{\alpha_t - \alpha_{t-\Delta t}}{\Delta t} x_{t-\Delta t}
    = \frac{\diff \log \alpha_t}{\diff t} x_t, \\
    g_t^2 
    & \coloneqq \lim_{\Delta t \to 0} \frac{\sigma_{t|t-\Delta t}^2}{\Delta t}
    = \lim_{\Delta t \to 0} \alpha_t^2 
      \frac{\frac{\sigma_t^2}{\alpha_t^2} - \frac{\sigma_{t-\Delta t}^2}{\alpha_{t-\Delta t}^2}}{\Delta t} 
    = \alpha_t^2 \,\frac{\diff}{\diff t}\biggl[\frac{\sigma_t^2}{\alpha_t^2}\biggr].
\end{align*}

\continuous[Discrete to continuous time derivation]{Under the discrete-to-continuous time limit, the one-dimensional Gaussian diffusion process converges to the SDE:
\begin{align}
\diff x_t  =  f_t(x_t)\,\dt \;+\; g_t\,\diff w_t, \qquad f_t(x_t) = \frac{\diff \log \alpha_t}{\diff t} \, x_t ,
\quad
g_t^{\,2}  =  \alpha_t^{2}\,\frac{\diff}{\diff t}\!\left[\frac{\sigma_t^{2}}{\alpha_t^{2}}\right].
\label{eq:cont-time-limit-csp}
\end{align}
This result can be directly extended to the multi-dimensional Gaussian diffusion process.
Once a noise schedule $(\alpha_t, \sigma_t)$ is chosen, the expressions above determine the drift and diffusion coefficients of the corresponding SDE, which yields the desired marginal distributions.}

The discrete-state diffusion process also admits a continuous-time limit and leads to a \emph{Markov jump process} or \emph{continuous-time Markov chain (CTMC)} \cite{campbell_continuous_2022,zhao_unified_2024,austin2021structured}.\footnote{\emph{Markov jump processes} can be defined on either discrete or continuous state space, whereas \emph{continuous-time Markov chains (CTMCs)} are, by definition, restricted to discrete state spaces. In the context of discrete state spaces, these terms can be used interchangeably.}

\discrete[Continuous-time limit (CTMC and rate matrix)]{
Let $\Q_{t\cond t-\Delta t}$ be the one-step transition matrix. The limit process is a CTMC with \emph{rate matrix} $\R_t$ defined by:
\begin{align}
\R_t
 = 
\lim_{\Delta t\to 0}\frac{\Q_{t\cond t-\Delta t}-\I}{\Delta t},
\qquad
q_{t\cond t-\Delta t}(x \cond y)
 = 
\delta_{x,y}
\;+\;
\R_t(x,y)\,\Delta t
\;+\;
o(\Delta t).
\label{eq:def-transition-rate-matrix}
\end{align}
$\R_t(x,y)$ is  the instantaneous transition rate from $y$ to $x$ at time $t$: for $x\neq y$, the probability of a jump $y\!\to\!x$ over $\Delta t$ is approximately $\R_t(x,y)\,\Delta t$, while $-\R_t(y,y)\,\Delta t$ is the probability of leaving $y$ in that interval.}

\begin{remark}[Off-diagonal non-negativity, divergence-free \& partial derivative view]
\leavevmode\\[-1.5em]
\begin{enumerate}[label=(\roman*),leftmargin=*,itemsep=1pt,topsep=2pt]
\item The off-diagonal entries must be non-negative, i.e., $[\R_t]_{ij}\ge 0$ for $i\ne j$.
\item $\R_t$ must satisfy the divergence-free property, i.e., its columns must sum to zero: $\mathbf 1^\top \R_t=\mathbf 0^\top$. This ensures probability conservation: for each state, the total outgoing transition rate (encoded in the diagonal) equals the sum of incoming transition rates (off-diagonal entries), $\R_t(y,y) = -\sum_{x \neq y} \R_t(x,y)$.
\item  In some contexts, it is helpful to use the equivalent form of \cref{eq:def-transition-rate-matrix}: \newline 
$\R_t(x,y)=\lim_{s \to t^-} \partial_t q(x_t = x \cond x_s=y)
=\lim_{s\to t^-}
\tfrac{q_{t\cond s}(x\cond y)-\delta_{x,y}}{t-s}$.
\end{enumerate}
\end{remark}

We now consider a discrete state space process with transition matrix parameterised via convex combination (\cref{eq:interpolation-discretediff}), providing a natural bridge between the discrete and continuous time formulations.

\noindent

\discrete[Rate matrix for interpolation-parameterised transitions]{
For $\Q_t=\alpha_t \I + (1-\alpha_t)\,\pnoise \mathbf{1}^{\top}$ (\cref{eq:forward-transition-interpolation-discretediff}), and a non-increasing differentiable noise schedule $\alpha_t$, the marginal
\begin{align}
q(x_t \cond x_0)=\Cat \bigl(x_t \ ;\ \Q_t\,\e_{x_0}\bigr)
\label{eq:marginal-with-Q}
\end{align}
induces the rate matrix:
\begin{align}
\R_t  =  \frac{\alpha_t'}{\alpha_t}\,\bigl(\I - \pnoise \mathbf{1}^{\top}\bigr) \ ,
\label{eq:Rate-matrix-expression}
\end{align}
where $\alpha_t'=\frac{\diff}{\diff t}\alpha_t$ (\citeplain[][Appendix~A.6]{austin2021structured}; \citeplain{sahoo_simple_2024,shi_simplified_2025,zhao_unified_2024}). The detailed proof is provided in \cref{apx:discrete-to-continuous-time-discrete-state} and follows directly from \cref{eq:interpolation-discretediff,eq:def-transition-rate-matrix}.
}

\noindent

Having established the continuous-time limit for both continuous state spaces (SDEs, \cref{eq:cont-time-limit-csp}) and discrete state spaces (CTMCs with rate matrices, \cref{eq:def-transition-rate-matrix,eq:Rate-matrix-expression}), we now leave the discrete-time formulation behind and work directly in continuous time. In the following section, we develop the continuous-time theory systematically: deriving the Kolmogorov forward equations that govern marginal evolution, establishing the form of the reverse (generative) process, and discussing the parameterisations used in practice.

\section{Continuous-time diffusion\label{sect:continuous-diffusion}}

In this section, we develop the continuous-time formulation of diffusion models, where the forward noising process evolves continuously over time $t \in [0,1]$ rather than through discrete steps. This continuous perspective provides both theoretical elegance and practical advantages, enabling the use of powerful tools from stochastic calculus and partial differential equations.

\subsection{Forward process and Kolmogorov forward equation}

We consider a time-indexed stochastic process $\{\x_t\}_{t \in [0,1]}$ where $t=0$ represents data and $t=1$ represents the prior (noise). The process evolves through either:

\emph{Continuous state spaces:} $\x_t \in \RR^d$ evolves via a stochastic differential equation (SDE)
\begin{align}
\diff \x_t = \f_t(\x_t)\,\diff t + g_t\,\diff \w_t\ , \qquad \x_0 \sim \qdata
\label{eq:sde-continuous-diff}
\end{align}
where $\f_t: \RR^d \to \RR^d$ is the drift coefficient, $g_t \in \RR^+$ is the diffusion coefficient, and $\w_t$ is a standard $d$-dimensional Brownian motion.

\emph{Discrete state spaces:} $\x_t \in \calX^d$ (where $\calX$ is a finite alphabet) evolves as a continuous-time Markov chain (CTMC) with transition probabilities
\begin{align}
q(\x_t = \y \cond \x_{t-\Delta t} = \x) = \delta_{\x,\y} + \Rseq_t(\y,\x)\,\Delta t + o(\Delta t)\ ,
\label{eq:ctmc-transition}
\end{align}
where $\Rseq_t: \calX^d \times \calX^d \to \RR$ is the multi-dimensional rate matrix with $\Rseq_t(\x,\x) = -\sum_{\y \neq \x} \Rseq_t(\y,\x)$ ensuring probability conservation.

The forward process is designed such that $q_0(\x)$ matches the data distribution $\qdata(\x)$ and $q_1(\x)$ matches a tractable prior $p_\text{noise}(\x)$ ($q_1 \approx p_\text{noise}$). We denote by $q_t(\x)$ the marginal distribution at time $t$, and by $q(\x_t = \x \cond \x_s = \y)$ or $q_{t \cond s} (\x \cond \y)$, the conditional distribution (transition kernel) from state $\y$ at time $s$ to state $\x$ at time $t$. In practice, the forward process is designed such that one can easily sample from $q_{t\cond 0}$.

In this section we focus on how the marginal distributions $\{q_t\}_{t \in [0,1]}$ evolve under a Markov process. Their evolution is governed by a \textit{partial differential equation (PDE)}, the \emph{Kolmogorov forward equation (KFE)}, also known as the \emph{Fokker--Planck equation} in continuous state spaces and \emph{master equation} in discrete state spaces.

\continuous[Kolmogorov forward (Fokker--Planck) equation]{
For the process $\{\x_t\}_{t\in[0,1]}$ governed by the SDE in \cref{eq:sde-continuous-diff}, the Kolmogorov forward (Fokker–Planck) equation for the marginals is
\begin{align}
\partial_t q_t(\x) 
= - \underbrace{\nabla_\x \!\cdot \bigl(\f_t(\x)\, q_t(\x)\bigr)}_{\text{drift term}} + \underbrace{\tfrac{1}{2}\,g_t^2\,\Delta_\x q_t(\x) }_{\text{diffusion term}} \ .
\label{eq:KFE-continuous-sp}
\end{align}
With $g_t= 0$ we recover the \emph{continuity equation}, and with $\f_t= 0$ the \emph{heat equation}. 
}

\begin{proof}[Derivation sketch:]
We first apply the law of total probability and the Markov property to the transition kernel (Chapman--Kolmogorov equation).\footnote{The Chapman--Kolmogorov equation says that for a Markov process to get from state $\y$ at time $s$ to state $\x$ at time $t$, you can take any intermediate route through state $\z$ at time $r$ (where $s<r<t$). The equation sums over all possible intermediate states, weighting each path by the probability of going $\y\to\z\to\x$.}
For any $t > r > s$:
\begin{align}
q(\x_t = \x \cond \x_s = \y) = \int_{\RR^d} q(\x_t = \x \cond \x_r = \z) q(\x_r = \z \cond \x_s = \y) \, \mathrm{d}\z \ .
\label{eq:CK}
\end{align}
From the SDE $\diff \x_t = \f_t(\x_t)\,\diff t + g_t\,\diff \w_t$, we consider an infinitesimal time step $\Delta t = t - r$ and write the Wiener increment as $\Delta \w_t = \sqrt{\Delta t}\, \epsilonb$ with $\epsilonb \sim \mathcal{N}(0, \I)$. The infinitesimal change becomes
\begin{align}
\Delta \x_t = \f_r(\z)\,\Delta t + g_r\,\sqrt{\Delta t}\,\epsilonb \ ,
\end{align}
yielding the infinitesimal moments
\begin{align}
    \E[\Delta \x_t \cond \x_r = \z] = \f_r(\z) \,\Delta t + o(\Delta t)\ , \qquad
    \E[(\Delta \x_t)(\Delta \x_t)^\top \cond \x_r = \z] = g_r^2 \,\I\, \Delta t + o(\Delta t)\ .
    \label{eq:diffusion-moments}
\end{align}
We multiply \cref{eq:CK} by a smooth test function $\phi(\x)$ with compact support and integrate w.r.t. $\x$:
\begin{align*}
\int_{\mathbb{R}^d} \phi(\x) q(\x_t = \x \cond \x_s = \y) \, \mathrm{d}\x 
= \int_{\mathbb{R}^d} \int_{\mathbb{R}^d} \phi(\x) q(\x_t = \x \cond \x_r = \z) q(\x_r = \z \cond \x_s = \y) \, \mathrm{d}\x \, \mathrm{d}\z\ .
\end{align*}
Taylor-expanding $\phi(\x)$ around $\z$ and using \cref{eq:diffusion-moments}, the RHS inner integral becomes:
\begin{align*}
\int_{\RR^d} \phi(\x) q(\x_t = \x \cond \x_r = \z) \, \mathrm{d}\x 
= \phi(\z) + \nabla_\z \phi(\z)^\top \f_r(\z)\, (t-r) + \tfrac{1}{2} g_r^2 \Delta_\z \phi(\z)\, (t-r) + o(t-r)\ .
\end{align*}
For the left-hand side, we expand the short-time evolution:
\begin{align}
\int_{\RR^d} \phi(\x) q(\x_t = \x \cond \x_s = \y) \, \mathrm{d}\x 
= &\int_{\RR^d} \phi(\z) q(\x_r = \z \cond \x_s = \y) \, \mathrm{d}\z \\
&+ (t-r) \int_{\RR^d} \phi(\z) \partial_t q(\x_t = \z \cond \x_s = \y)\big|_{t=r} \, \mathrm{d}\z+ o(t-r)\ .
\end{align}
Equating both sides in the limit $r \to t^-$ gives
\begin{align*}
\int_{\RR^d} \phi(\z) \partial_t q(\x_t = \z \cond \x_s = \y) \, \mathrm{d}\z 
= \int_{\RR^d} \bigl[\nabla_\z \phi(\z)^\top \f_t(\z) + \tfrac{1}{2} g_t^2 \Delta_\z \phi(\z)\bigr] q(\x_t = \z \cond \x_s = \y) \, \mathrm{d}\z\ .
\end{align*}
Integrating by parts (with vanishing boundary terms) and using the fundamental lemma of calculus of variations yields
\begin{align}
\partial_t q(\x_t = \x \cond \x_s = \y) = -\nabla_\x \cdot (\f_t(\x) q(\x_t = \x \cond \x_s = \y)) + \tfrac{1}{2} g_t^2 \Delta_\x q(\x_t = \x \cond \x_s = \y).
\end{align}
Marginalising over $\x_s = \y$ with $s=0$ recovers \cref{eq:KFE-continuous-sp}.
\end{proof}

\discrete[Kolmogorov forward (master) equation]{
For a CTMC defined with \cref{eq:ctmc-transition} with a multi-dimensional rate matrix $\Rseq_t$, such that $\Rseq_t(\x,\y) \geq 0$ for $\x \neq \y$, and for all $\x \in \calX^d$: $\Rseq_t(\x,\x) = - \sum_{\y \neq \x} \Rseq_t(\y,\x)$.\footnote{This ensures that the total \textit{probability flow} out of each state is balanced by the sum of incoming rates.} Under our convention (destination = row, source = column), the Kolmogorov forward (master) equation for the marginals is
\begin{align}
\partial_t q_t(\x) 
 =  \sum_{\y \in \calX^d} \Rseq_t(\x,\y)\, q_t(\y).
\label{eq:KFE-discrete-sp}
\end{align}
In vector form, i.e., denoting $q_t$ by the probability vector $\q_t$,
\begin{align}
\partial_t \q_t  =  \Rseq_t\,\q_t.
\end{align}
}
\begin{proof}[Derivation sketch:]  Similar to continuous state spaces, let $t \geq s \geq 0$ and apply the Chapman--Kolmogorov equation to the transition kernel, for any $t > r > s$:
\begin{align}
    q(\x_t = \x \cond \x_s = \y) = \sum_{\w \in \calX^d} q(\x_t = \x \cond \x_r = \w)q(\x_r = \w \cond \x_s = \y).
\end{align}
As this equation holds for any $t>r>s$, we probe the system's time evolution by differentiation with respect to $t$ and taking $r$ infinitesimally close to $t$. This yields: 
\begin{align} 
\partial_t q(\x_t = \x \cond \x_s = \y) &= \sum_{\w \in \calX^d} \partial_t q(\x_t = \x \cond \x_r = \w)q(\x_r = \w \cond \x_s = \y) \\ 
\partial_t q(\x_t = \x \cond \x_s = \y) &= \sum_{\w \in \calX^d} \lim_{r \to t^-} \partial_t q(\x_t = \x \cond \x_r = \w) q(\x_r = \w \cond \x_s = \y) \\ 
\partial_t q(\x_t = \x \cond \x_s = \y) &= \sum_{\w \in \calX^d} \Rseq_t(\x,\w) q(\x_t = \w \cond \x_s = \y),
\end{align} 
where $\Rseq_t(\x,\w) \coloneqq \lim_{r \to t^-} \partial_t q(\x_t = \x \cond \x_r = \w)$ is the sequence-level rate matrix as defined in \cref{eq:seq-rate-matrix}. By setting $s=0$ and marginalising, we recover the Kolmogorov forward (master) equation for the marginals: 
\begin{align} 
\partial_t q(\x_t = \x) &= \sum_{\w \in \calX^d} \Rseq_t(\x,\w) q(\x_t = \w). 
\end{align}
\end{proof}
\begin{remark}[Dimensional factorisation]
If forward per-coordinate jumps are independent, the joint rate matrix is the \emph{Kronecker sum} of per-dimension generators, 
\begin{align}\Rseq_t=\bigoplus_{k=1}^d \R^{(k)}_t=\sum_{k=1}^d
    \Bigl(\I^{\otimes (k-1)} \otimes \R_t^{(k)} \otimes \I^{\otimes (d-k)}\Bigr),
    \label{eq:seq-rate-matrix}
\end{align}
so only single-coordinate jumps occur at $O(\Delta t)$ (detail provided in \cref{apx:rate-matrix-sequence-lvl}).
\end{remark}

Usually in discrete state diffusion literature, the rate matrix and KFE are derived in a single dimension and later extended to the multidimensional case. In our approach, when the forward noising process is independent per dimension, we recover the evolution of each coordinate marginal through a single-dimensional KFE with a single-dimensional rate matrix.\footnote{This result can be seamlessly extended to continuous-state spaces.}

\discrete[Single-dimension master equation (independent, discrete)]{
Let $\R_t$ be the rate matrix on $\calX$ (identical across coordinates in the i.i.d.\ case). For the coordinate-$k$ marginal probability vector $\q_t^{(k)}\in\Delta^{K-1}$,
\begin{equation}
\partial_t \q_t^{(k)}  =  \R_t\,\q_t^{(k)},
\qquad k=1,\dots,d.
\label{eq:kfe-1d-discrete}
\end{equation}
When the coordinate is fixed or clear from context, we
often suppress the index $k$ and write $\q_t$ for $\q_t^{(k)}$. It does not imply that the coordinate marginals are
identical.
}

\subsection{Reverse process}

We have seen that the KFE provides a unified view of marginal evolution in both continuous and discrete spaces. Viewed through the KFE, reversing the generative dynamics amounts to running the marginal evolution backwards in time. Concretely, if $\{q_t\}_{t\in[0,1]}$ solves the forward KFE, then with the time change $s\mapsto 1-s$ the reversed marginals $\hat q_s\coloneqq q_{1-s}$ satisfy a \emph{reversed} KFE. In continuous state spaces, this corresponds to a reverse SDE with a drift correction involving the score $\nabla\log q_{1-s}$ (see \cref{eq:reverse-SDE}); in discrete state spaces, it corresponds to a reverse CTMC with a rate matrix tilted by the instantaneous marginals (see \cref{eq:reverse-rate-matrix}). We now derive this in detail.

The detailed formal derivation of the corresponding reverse (denoising) process from the forward equations can be found in the literature on SDEs and Markov processes \cite{anderson1982reverse,oksendal2003stochastic}. In diffusion generative modeling, we seek to define a process $\{\hat{\x}_s\}_{s \in [0,1]}$ with marginals $\{\hat{q}_s\}_{s \in [0,1]}$ such that, for all times $s$, $\hat{q}_s = q_{1-s}$. Consequently, the boundary conditions become $\hat{q}_0 = q_1 \approx p_{\text{noise}}$ and $\hat{q}_1 = q_0 = \qdata$, which ensures that at the end of the reverse process we recover the clean data distribution.\footnote{It is important to emphasise that the reverse process is not defined with respect to the forward filtration, but rather with respect to the natural filtration of the reversed process $\{\sigma(\x_{1-t})\}_{t\in[0,1]}$. Intuitively, the past of $\hat{\x}_s$ corresponds to the future of $\x_{1-s}$. Although these two processes differ in the information each “sees” at intermediate steps, they coincide \textit{in law} over their entire path after time reversal pushforward. If $\rho$ denotes the path reversal, $(\rho \omega)_s \coloneqq \omega_{1-s}$, then in terms of path measures $\hat \Qpath = \rho_\# \Qpath$ (cf. \cref{table:notation1}).}

\begin{remark}[Hat convention for time reversal]
\label{rem:hat-convention}
The hat notation $\hat{\cdot}$ indicates quantities associated with the \emph{reverse process}. The fundamental definition is the time-reversed process itself:
\begin{equation}
\hat\x_s  \coloneqq  \x_{1-s}, \qquad s \in [0,1].
\end{equation}
All quantities derived from the reverse process consequently inherit the hat. For marginal distributions, this reduces to simple index substitution: $\hat q_s = q_{1-s}$. However, for operators governing the reverse dynamics, such as rate matrices and other derived quantities, the relationship to their forward counterparts is more complex and will be derived in the following sections.
\end{remark}

\continuous[Reverse KFE and reverse SDE]{
With the time change $t\mapsto 1-s$, the reversed marginals $\{\hat q_s=q_{1-s}\}_{s \in [0,1]}$ satisfy
\begin{align}
\partial_s \hat q_s(\x)
= -\,\nabla_\x  \cdot \Bigl(\bigl[g_{1-s}^2 \nabla_\x \log \hat q_s(\x) - \f_{1-s}(\x)\bigr]\hat q_s(\x)\Bigr)
+ \tfrac12\, g_{1-s}^2\,\Delta_\x \hat q_s(\x)\;.
\label{eq:reverse-Fokker-Planck}
\end{align}
This corresponds to the reverse-time SDE:
\begin{align}
\diff \hat{\x}_s
= \Bigl[-\,\f_{1-s}(\hat{\x}_s) + g_{1-s}^{2}\,\nabla_\x \log \hat q_{s}(\hat{\x}_s)\Bigr]\diff s
+ g_{1-s}\,\diff \hat{\w}_s\;.
\label{eq:reverse-SDE}
\end{align}
}

\begin{proof}[Derivation sketch.] Given the KFE for the forward process marginals $q_t$, we derive the KFE for the reversed marginals $\hat{q}_s = q_{1-s}$ using the time change $t \mapsto 1-s$:
\begin{align}
\partial_s\,\hat{q}_s(\x)
&= -\partial_t\,q_{t}(\x)\Big|_{t=1-s} \\
&= \nabla_\x \cdot\left(\f_{1-s}(\x) \hat{q}_{s}(\x)\right) - \frac{1}{2}g_{1-s}^2 \Delta_\x \hat{q}_{s}(\x).
\end{align}
The negative diffusion term prevents direct interpretation as a forward-in-$s$ SDE. To resolve this, we add and subtract $\text{C}_s(\x) \, g_{1-s}^2 \, \Delta_\x \hat{q}_{s}(\x)$ for an arbitrary scalar-valued function $\text{C}_s \geq 0$:
\begin{align}
\partial_s\hat{q}_s(\x)
&= \nabla_\x \cdot\left(\f_{1-s}(\x) \hat{q}_{s}(\x)\right) - \left(\text{C}_s(\x)+\frac{1}{2}\right)g_{1-s}^2\Delta_\x\hat{q}_{s}(\x) + \text{C}_s(\x) g_{1-s}^2\Delta_\x\hat{q}_{s}(\x).
\label{eq:reverseKFE-trick}
\end{align}
We arbitrarily set $\text{C}_s(\x) \coloneqq \frac{1}{2}$. Using the Laplacian-divergence identity $\Delta_\x \hat{q}_s = \text{div}(\nabla_\x \hat q_s)= \nabla_\x \cdot [\nabla_\x \hat{q}_s]$ in the term of the RHS and factoring out $\hat{q}_s$ with $\nabla_\x\,\log\hat{q}_s = \tfrac{\nabla_\x\,\hat{q}_s}{\,\hat{q}_s}$, we recover the reverse KFE:
\begin{align}
    \partial_s\,\hat{q}_s(\x)
     = 
    \underbrace{-\;\nabla_\x \cdot\Bigl(\bigl[g_{1-s}^2\,\nabla_\x \log\hat{q}_s(\x) - \f_{1-s}(\x)\bigr]\,\hat{q}_s(\x)\Bigr)}_{\text{drift term}}
    \;+\;
    \underbrace{\frac12\,g_{1-s}^2\,\Delta_\x\!\bigl[\hat{q}_s(\x)\bigr],}_{\text{diffusion term}}
    \label{eq:reverse-FPE}
\end{align}
with the associated reverse SDE:
\begin{align}
    \diff \hat{\x}_s 
    = \left[-\f_{1-s}(\hat{\x}_s) + g_{1-s}^2 \nabla_\x \log q_{1-s}(\hat{\x}_s)\right] \diff s + g_{1-s} \diff \hat{\w}_s\ .
\end{align}
\end{proof}

While reversing the KFE only ensures that the marginals satisfy $\hat{q}_s = q_{1-s}$, for all $s\in [0,1]$, it does not by itself guarantee a full time reversal of the process. In fact, the SDE \eqref{eq:reverse-SDE} \emph{does} define the time-reversed dynamics of the forward process, meaning that $\{\hat{\x}_s \}_{s\in [0,1]}\overset{\text{law}}{=} \{\x_{1-s}\}_{s\in [0,1]}$. Establishing this equality in law over entire trajectories, however, requires a rigorous reversal of the SDE using more sophisticated stochastic methods. For the detailed derivation, we refer the reader to \citet{anderson1982reverse, haussmann_time_1986, cattiaux_time_2023}.

\continuous[Probability–flow ODE and exact likelihood]{
The reverse SDE in \cref{eq:reverse-SDE} is not unique: alternative decompositions of drift and diffusion terms yield different SDEs with identical marginal evolution. By absorbing all stochasticity into the drift, we obtain a deterministic process---the \textit{probability-flow ODE} \cite{chen_neural_2019}:
\begin{align}
    \diff \x_t
     = 
    \ub_t(\x_t)\,\diff t\ ,
    \qquad
    \text{where}\quad
    \ub_t(\x)
     \coloneqq \f_{t}(\x) -
    \frac{g_{t}^2}{2}\,\nabla_\x \log q_{t}(\x) \;.
    \label{eq:reverse-ODE}
\end{align}
This deterministic trajectory enables exact likelihood computation. For any path $\{\x_t\}_{t \in [0,1]}$ satisfying \cref{eq:reverse-ODE}, the log-likelihood is:
\begin{align}
\log q_0(\x_0)
 = 
\log q_1(\x_1)
 + 
\int_0^1 \nabla_\x\!\cdot \ub_t(\x_t)\,\diff t\ .
\label{eq:exact-likelihood-main}
\end{align}
}

\begin{proof}[Derivation sketch.]
To compute the \textit{probability-flow ODE}, we return to the reverse SDE derivation \cref{eq:reverseKFE-trick}. In this equation, the choice of $\text{C}_s(\x)\coloneqq\tfrac 12$, which inserts the diffusion term into the drift, was arbitrary. 
Setting instead $\text{C}_s(\x) \coloneqq 0$ absorbs the diffusion term entirely into the drift, yielding the following Fokker--Planck (continuity) equation:
\begin{align}
    \partial_s\,\hat{q}_s(\x)
     = 
    -\nabla_\x \cdot\Bigl[\hat\ub_s(\x) \, \hat{q}_s(\x)\Bigr],
    \quad
    \text{where}\quad
    \hat \ub_s(\x)
     = 
    \frac{g_{1-s}^2}{2}\,\nabla_\x \log q_{1-s}(\x)
     - 
    \f_{1-s}(\x)\ .
    \label{eq:reverse-ODE-KFE}
\end{align}
While this formulation suffices for sampling from the ODE, exact likelihood computation requires expressing the dynamics in forward time $t \in [0,1]$. Substituting $t \coloneqq 1-s$ into \cref{eq:reverse-ODE-KFE} and using $\partial_s\,\hat{q}_s=-\partial_t q_t$ yields:
\begin{align}
   \partial_t q_{t}=
    -\nabla_\x \cdot\Bigl[\ub_{t}(\x) \, q_{t}(\x)\Bigr],
    \quad
    \text{where}\quad
     \ub_t(\x)
     = \f_{t}(\x)-
    \frac{g_{t}^2}{2}\,\nabla_\x \log q_{t}(\x)
    \ .
    \label{eq:forward-ODE-KFE}
\end{align}
This corresponds to the deterministic ODE $\diff \x_t=\ub_t(\x_t)\,\diff t$, which is the probability-flow ODE. 

For the \textit{exact likelihood}, we compute the total time derivative of $\log q_t(\x_t)$ along the deterministic trajectory:
\begin{align}
    \frac{\diff}{\diff t} q_t(\x_t) &= \partial_t q_t(\x_t) +\left[ \nabla_\x  q_t(\x_t) \right] \transp \frac{\diff \x_t }{\diff t}\\
    &= -\nabla_\x \cdot\Bigl[\ub_t(\x_t)  q_t(\x_t)\Bigr] + \left[\nabla_\x q_t(\x_t) \right] \transp   \ub_t(\x_t) \\
    &= -[\nabla_\x \cdot \ub_t(\x_t) ]   \,q_t(\x_t) \qquad \text{ with } \ \nabla_\x \cdot\Bigl[\ub_t  q_t\Bigr] =\nabla_\x \cdot \ub_t q_t+\left[\nabla_\x  q_t \right]\transp \ub_t  \\ 
    \Rightarrow  \frac{\diff}{\diff t}  \log  &q_t(\x_t) = - \nabla_\x \cdot \ub_t(\x_t)\;.
\end{align}
This leads to the following first-order differential equation:
\begin{align}
\frac{\diff}{\diff t}\left[\begin{array}{c}\x_t \\ \log q_t(\x_t)\end{array}\right]=\left[\begin{array}{c}\ub_t(\x_t)  \\ - \nabla_\x \cdot \ub_t(\x_t)\end{array}\right],
\end{align}
which integrates to the exact likelihood identity:
\begin{align}
  \log q_0(\x_0)
  =
  \log q_1(\x_1)
  +
  \int_{0}^{1}\nabla_\x\!\cdot
  \ub_{t\,}\!\bigl(\x_t\bigr)\,\diff t\;.
  \label{eq:exact-likelihood}
\end{align}
\end{proof}

In summary, various decompositions of the drift and diffusion terms are valid in the reverse KFE, each corresponding to different SDE formulations for the reverse process, all sharing the same evolution of
marginals.

We now turn to discrete state spaces, where an analogous construction applies.

\discrete[Reverse KFE and reverse CTMC]{
Let $\Rseq_t$ be the forward rate matrix.
The reversed marginals $\hat q_s=q_{1-s}$ satisfy:
\begin{align}
 \partial_s \hat{\boldsymbol q}_s=\hat{\Rseq}_s\,\hat{\q}_s\;,
 \label{eq:reverse-CTMC}
\end{align}
where the reverse rate matrix $\hat{\Rseq}_{s}$ satisfies: 
\begin{align}
    \hat{\Rseq}_{s}(\x,\z) =
    \begin{cases}
        \Rseq_{1-s}(\z,\x)\,\dfrac{q_{1-s}(\x)}{q_{1-s}(\z)} & \x \neq \z \ ,\\[1mm]
        -\displaystyle\sum_{\y \neq \x} \Rseq_{1-s}(\x,\y)\,\dfrac{q_{1-s}(\y)}{q_{1-s}(\x)} & \x = \z \ .
        \label{eq:reverse-rate-matrix}
    \end{cases}
\end{align}
Similar to the discrete-time formulation \cref{eq:reversal-austin}, one recovers the posterior-based parameterisation \cite{austin2021structured,campbell_continuous_2022}. For $\x \neq \y \in \calX^d$:
\begin{align}
\hat{\Rseq}_s(\x,\y)= \frac{q_{1-s}(\x)}{q_{1-s}(\y)}\Rseq_{1-s}(\y,\x) 
= \Rseq_{1-s}(\y,\x) \sum_{\x_0} \frac{q_{1-s\cond 0}(\x \cond \x_0)}{q_{1-s\cond 0}(\y \cond \x_0)} \, q_{0 \cond 1-s}(\x_0 \cond \y).
\label{eq:Reverse-R-with-posterior}
\end{align}
This yields the reverse CTMC through \cref{eq:def-transition-rate-matrix}.
}

\begin{proof}[Derivation sketch.] 
In discrete state spaces, using the forward-time KFE, the reversed marginal dynamics satisfy \cite{nagasawa1964time,kelly_reversibility_2011,campbell_continuous_2022}:
\begin{align}
    \partial_s\,\hat{q}_s(\x)
\;&=
-\left.\partial_t q_t(\x)\right|_{t=1-s}\\
&= -\sum_{\y \in \calX^d} \Rseq_{1-s}(\x,\y)q_{1-s}(\y)\\
&=\underbrace{-\sum_{\y \neq \x} \Rseq_{1-s}(\x,\y)q_{1-s}(\y)}_{\text{non diagonal terms}} + \underbrace{\sum_{\z \neq \x}\Rseq_{1-s}(\z,\x)q_{1-s}(\x).}_{-\Rseq(\x,\x)q_{1-s}(\x)\text{ diagonal term}}
\end{align}
We now define the \textit{rate matrix} of the reverse process, denoted by $\hat{\R}_s$, such that it yields the same evolution of the marginals and satisfies the standard CTMC conditions (i.e. non-negative off-diagonals and columns summing to zero). A natural \textit{minimal} way is to ``switch'' the roles of diagonal and off-diagonal contributions:\footnote{This can be understood intuitively as a reversal of probability flow: during the reverse dynamics, what was previously incoming probability becomes outgoing, and conversely, outgoing becomes incoming.}
\begin{align}
    \hat{\Rseq}_{s}(\x,\z)\hat{q}_{s}(\z)=
    \begin{cases}
        \Rseq_{1-s}(\z,\x)q_{1-s}(\x) & \z \neq \x, \\[0.5em]
        -  \sum_{\y \neq \x} \Rseq_{1-s}(\x,\y)q_{1-s}(\y) & \z=\x.
    \end{cases}
\end{align}
Thus,
\begin{align*}
\hat{\Rseq}_{s}(\x,\z) =
\begin{cases}
\Rseq_{1-s}(\z,\x)\,\dfrac{q_{1-s}(\x)}{q_{1-s}(\z)} & \text{if } \x \neq \z\\[1mm]
-\displaystyle\sum_{\y \neq \x} \Rseq_{1-s}(\x,\y)\,\dfrac{q_{1-s}(\y)}{q_{1-s}(\x)} & \text{if } \x = \z
\end{cases}
\end{align*}
which satisfies off-diagonal non-negativity and the divergence-free property.
\end{proof}

Similar to continuous state-spaces, where distinct SDEs, with different diffusion coefficients, can induce the same marginal evolution, distinct CTMCs can share the same family of one-time marginals.

\footnote{While infinitely many processes reproduce the reversed marginals $\hat q_s = q_{1-s}$, only one is the true time reversal of the forward process. It satisfies the stronger \emph{path-level} identity $\hat\Qpath=\rho_\#\Qpath$, where $\rho$ is the path-reversal map $(\rho\omega)_s\coloneqq\omega_{1-s}$ (cf.\ \cref{table:notation1}). The reversed process is generally not adapted to the original
forward filtration:
although $\hat\x_s=\x_{1-s}$ is $\mathcal F_{1-s}$-measurable, the
family $(\mathcal F_{1-s})_{s\in[0,1]}$, is decreasing. Intuitively, the forward
and reverse processes differ in the information each ``sees'': the past
of $\hat\x_s$ corresponds to the future of $\x_{1-s}$.}

The key insight is that any rate matrix $\C_s$ satisfying \emph{detailed balance} with respect to the reversed marginal $\hat{q}_s$ has no net effect on the marginal evolution and can therefore be added to $\hat{\Rseq}_s$ \cite{campbell_continuous_2022,campbell_generative_2024}:
\begin{align}
    \hat{q}_s(\x)\,\C_s(\y,\x) = \hat{q}_s(\y)\,\C_s(\x,\y), \quad \forall\, \x \neq \y\ .
\label{eq:detailed-balance-discrete}
\end{align}
Intuitively, detailed balance ensures that the probability mass flowing from $\x$ to $\y$ exactly equals the return flow from $\y$ to $\x$, resulting in zero net probability transfer. Consequently, we obtain a family of admissible reverse rate matrices parameterised by a scalar $\gamma \ge 0$:
\begin{align}
    \hat{\Rseq}_s^\gamma(\x,\y)  \coloneqq  \hat{\Rseq}_s(\x,\y) \;+\; \gamma\, \C_s(\x,\y)\ .
    \label{eq:reverse-rate-family}
\end{align}
The parameter $\gamma\in \RR^{+}$ controls the \emph{stochasticity} of the reverse process. Setting $\gamma=0$ yields the ``minimal'' reverse process that generates the fewest jumps necessary to match the marginals. 

Increasing $\gamma$ introduces additional reversible jumps, increasing the stochasticity of the process trajectory. While an optimal stochasticity level may be task-dependent, $\gamma=0$ is often preferred in practice \cite{campbell_generative_2024}.

\begin{remark}[Connection to Helmholtz--Hodge decomposition]
\label{rem:helmholtz-hodge}
The non-uniqueness of dynamics with a prescribed marginal evolution, in both continuous, \cref{eq:reverse-SDE,eq:reverse-ODE}, and discrete spaces, \cref{eq:reverse-rate-family}, can be understood through the Helmholtz--Hodge decomposition \cite{bhatia2012helmholtz}. Any probability dynamics can be written as a conservation equation
\begin{align}
    \partial_s \hat{q}_s(\x) = -\nabla_\x \cdot(\J_s(\x)) \quad\ \text{or} \quad\ \partial_s \hat{q}_s(\x) = -\sum_{\y\in \calX^d} \J_s(\y,\x) \quad \forall \ \x \in\calX^d\ ,
\end{align}
where $\J_s$ is the probability current. The Helmholtz--Hodge principle decomposes any current into a \emph{gradient} (curl-free) component and a \emph{rotational} (divergence-free) component. Since divergence-free currents satisfy $\nabla_\x \cdot(\J_s) = 0$, or $\sum_{\y \in \calX^d}\J_s(\y,\x)=0$ for all $\x$ in discrete spaces, they do not contribute to the evolution of marginals and can be freely added or removed.
\end{remark}

\discrete[Dimensional factorisation of the reverse process]{
Let the forward rate matrix factorise as a Kronecker sum (independent noising across coordinates \cref{eq:seq-rate-matrix}), for $\x\in\calX^d$:
\begin{align}
\Rseq_t(\y,\x)
= \sum_{k=1}^{d}\R^{(k)}_t(y^{(k)},x^{(k)})\,\delta_{\y^{\setminus k},\,\x^{\setminus k}} \ ,
\qquad
q(\x_t\cond \x_s)=\prod_{k=1}^d q\bigl(x^{(k)}_t \cond x^{(k)}_s\bigr) \ .
\label{eq:forward-decomp}
\end{align}
The reverse rate matrix generally does \emph{not} Kronecker-decompose because the tilt $q_{1-s}(\x)/q_{1-s}(\z)$ in \cref{eq:reverse-rate-matrix} couples all coordinates.
A naive evaluation therefore requires $K^d$ operations to compute all transitions $\x \to \y$. However, in continuous time, simultaneous multi-coordinate jumps occur with negligible probability over an interval $\Delta t$ (see \cref{apx:rate-matrix-sequence-lvl}). As a result, the reverse dynamics reduce to dimension-wise updates.

Concretely, transitions are only evaluated for pairs $(\x,\y)$ that differ in at most one coordinate (Hamming distance one). For a given coordinate $k$, transitions are computed conditional on $\x^{\setminus k} = \y^{\setminus k}$. This reduces the computational cost per state $\x$ from $K^d$ to $(K - 1) \,d + 1$:
\begin{align}
    \hat{\Rseq}_s(\x,\y)
    = \sum_{k=1}^{d} \hat \R^{(k)}_s\!\bigl(\x,\y\bigr)\,
    \delta_{\x^{\setminus k},\,\y^{\setminus k}},
    \qquad
    \hat \R^{(k)}_s\!\bigl(\x,\y\bigr)
    \coloneqq \R^{(k)}_{1-s}(y^{(k)},x^{(k)})\,\frac{q_{1-s}(\x)}{q_{1-s}(\y)}\ .
    \label{eq:seq-reverse-hamming1}
\end{align}
Equivalently, using the dimension-wise posterior,
\begin{align}
    \hat \R^{(k)}_s\!\bigl(\x, \y\bigr)
    =
    \R^{(k)}_{1-s}(y^{(k)},x^{(k)})
    \sum_{x_0^{(k)}\in \calX} 
    \frac{q_{1-s\cond 0}\!\bigl(x^{(k)} \cond x_0^{(k)}\bigr)}
         {q_{1-s\cond 0}\!\bigl(y^{(k)} \cond x_0^{(k)}\bigr)} q_{0 \cond1-s}\bigl(x_0^{(k)} \cond \x\bigr)  .
    \label{eq:seq-reverse-posterior}
\end{align}
In practice, we often set $\R^{(k)}_t\coloneqq \R_t$ for all $k$.
}

\begin{remark}[Notation \& Time indexing]
\leavevmode\\[-1.5em]
\begin{enumerate}[label=(\roman*),leftmargin=*,itemsep=1pt,topsep=2pt]
\item A common approach in prior work is to develop the theory first in the single-coordinate setting and then extend it to multidimensional product spaces with coordinatewise generators, for which each CTMC jump changes only one coordinate. Accordingly, \cref{eq:reverse-rate-matrix} is often used without explicitly introducing its multidimensional counterpart, \cref{eq:seq-reverse-hamming1}.
\item Some works alternatively reparameterise the reverse process as $\left\{\hat{q}_t\right\}_{t \in[0, 1]}$, with time $t$ flowing backwards (from $1$ to $0$). This avoids the notational asymmetry introduced by the time reversal $t \mapsto 1-s$, and aligns the time indices of the forward and reverse processes.
\end{enumerate}
\end{remark}

\subsection{Approximating the reverse process}

In the previous section, we characterised the exact time reversals---the reverse SDE, \cref{eq:reverse-SDE}, and reverse CTMC, \cref{eq:reverse-rate-matrix}---together with extended formulations, \cref{eq:reverse-ODE,eq:reverse-rate-family}, which have the same one-time marginals but generally a different path law. In all cases, $\hat q_s=q_{1-s}$. In practice, however, we never have analytic access to these marginals---only samples from $\qdata$ and tractable forward transitions $q_{t\cond 0}$. Approximating the reverse process therefore amounts to replacing the intractable quantities that appear in the exact formulas (scores $\nabla \log q_t$, marginal ratios $q_t(\y)/q_t(\x)$, or posteriors $q_{0\cond t}$) with \emph{learned} surrogates, typically parameterised by neural networks ($\mathrm{NN}^\thetab(\x_t,t)$), and trained using objectives based on maximum likelihood. 

The goal of this section is to make this approximation step explicit
by introducing several equivalent parameterisations that predict different but interconvertible targets (e.g., scores, clean data, noise, or velocity in continuous spaces, and posteriors or discrete ``scores'' in discrete spaces). Plugging these learned quantities back into the reverse processes yields implementable generative samplers that approximate the true reverse dynamics.

\paragraph{Continuous state spaces.}
The reverse SDE \eqref{eq:reverse-SDE} and reverse KFE \eqref{eq:reverse-FPE} are fully determined once the \emph{(Stein) score} $\nabla_\x \log q_t$ is known: if we had access to this score for all $t$, we could sample exactly from the reverse process. In practice, however, the marginal $q_t$ is typically a complex non-Gaussian distribution, so its score cannot be computed in closed form. To make the problem tractable, we usually define $\f_t$ to be linear in the state and $g_t$ to be a state-independent scalar, such that the forward transitions $q_{t\cond 0}$ are simple (see \citet[Section~5.5]{sarkka2019applied}):
\begin{align}
    q(\x_t \cond \x_0) = \mathcal{N}\bigl(\x_t \; ; \; \alpha_t \x_0, \sigma_t^2 \I\bigr) \ ,
    \label{eq:ct-forward-marginal}
\end{align}
where the time-dependent coefficients $\alpha_t$ and $\sigma_t$ satisfy the differential relations derived in \cref{eq:cont-time-limit-csp}:
\begin{align}
\f_t(\x_t) = \frac{\diff \log \alpha_t}{\diff t} \x_t\ , 
\qquad
g_t^2 = \alpha_t^2 \frac{\diff}{\diff t}\!\left[\frac{\sigma_t^2}{\alpha_t^2}\right] \ .
\label{eq:ct-alpha-sigma-relation}
\end{align}
These allow us to work equivalently with the SDE coefficients $(\f_t, g_t)$ or the parameters $(\alpha_t, \sigma_t)$.

The following result, links the tractable conditional score $\nabla_{\x_t} \log q(\x_t \cond \x_0)$ to the intractable marginal score $\nabla_{\x_t}\log q_t(\x_t)$ (see \cref{apx:tweedie-identity} for a proof).

\continuous[Score identity]{
The \emph{score identity} expresses the marginal score as a conditional expectation of the tractable conditional score:
\begin{align}
    \nabla_{\x_t} \log q_t(\x_t)=\mathbb{E}_{q(\x_0\cond\x_t)}\left[\nabla_{\x_t} \log q(\x_t \cond \x_0)\right].
    \label{eq:Tweedie-identity-ct}
\end{align}
See \cref{apx:tweedie-identity} for a proof.
}

This insight motivates the following parameterisations, all of which learn approximations of tractable conditional quantities that can be converted into the required score function. 
Writing $\x_t=\alpha_t\x_0+\sigma_t\epsilonb$ with $\epsilonb\sim\mathcal N(\zerob,\I)$ and defining $\mathbf v_t\coloneqq\alpha_t\epsilonb-\sigma_t\x_0$ \cite{salimans2022progressive}, we recall from \cref{eq:score-mean-eps-parameterisation} the following identities which relate different prediction targets:

\begin{align}
\nabla_{\x_t}\log q(\x_t\cond\x_0)
&=-\frac{\epsilonb}{\sigma_t},
\notag\\[-1mm]
\x_0
&=
\frac{\x_t-\sigma_t\epsilonb}{\alpha_t}
=\frac{\sigma_t^2\nabla_{\x_t}\log q(\x_t\cond\x_0)+\x_t}{\alpha_t}
=\frac{\alpha_t\x_t-\sigma_t\mathbf v_t}{\alpha_t^2+\sigma_t^2}.
\label{eq:ct-param-equivalences}
\end{align}

\continuous[Parameterisation of the reverse process]{
In continuous state spaces, we parameterise the reverse process by learning a neural network that predicts one of the following quantities. They are constructed from the conditional forward transition  $q(\x_t\cond\x_0)$ and can be converted into the marginal score. After converting the network output into a score estimate $\mathbf s^\thetab$, the learned reverse SDE is
\begin{align}
\diff \x_{1-t} = \Bigl[-\,\f_{t}(\x_{1-t}) + g_{t}^{2}\,\mathbf{s}^\thetab(\x_{1-t},t)\Bigr]\diff t + g_{t}\,\diff \w_t.
\label{eq:ct-score-param}
\end{align}

\begin{enumerate}[label=(\roman*),leftmargin=*,itemsep=1pt,topsep=2pt]
\item \emph{Score prediction:} Train $\mathbf{s}^\thetab(\x_t,t)$ by regressing on the tractable conditional score \newline $\nabla_\x \log q(\x_t \cond \x_0)$. By \cref{eq:Tweedie-identity-ct} the minimiser of this objective is the marginal score $\nabla_\x \log q_t(\x)$ \cite{song2019generative,song_score-based_2021}.

\item \emph{Clean-data prediction ($\x_0$-prediction):} Train $\x_0^\thetab(\x_t,t)$ to predict $\x_0$ and set
\begin{align}
\mathbf{s}^\thetab(\x_t,t) = \frac{\alpha_t\x_0^\thetab(\x_t,t) - \x_t}{\sigma_t^2}.
\label{eq:ct-x0-to-score}
\end{align}
This parameterisation is particularly stable in low-noise regimes \cite{karras2022elucidating}.

\item \emph{Noise prediction ($\epsilonb$-prediction):} Train $\epsilonb^\thetab(\x_t,t)$ to predict $\epsilonb$ and set
\begin{align}
\s^\thetab(\x_t,t) = -\frac{\epsilonb^\thetab(\x_t,t)}{\sigma_t}.
\label{eq:ct-eps-to-score}
\end{align}
This is the parameterisation used in DDPM \cite{ho_denoising_2020}.

\item \emph{Velocity prediction ($\mathbf{v}$-prediction):} Train $\mathbf{v}^\thetab(\x_t,t)$ to predict $\mathbf{v}_t = \alpha_t \epsilonb - \sigma_t \x_0$ and set
\begin{align}
\s^\thetab(\x_t,t) = -\frac{\sigma_t \x_t + \alpha_t \mathbf{v}^\thetab(\x_t,t)}{\sigma_t(\alpha_t^2+\sigma_t^2)}.
\label{eq:ct-v-to-score}
\end{align}
This parameterisation can provide better balance across noise levels \cite{salimans2022progressive,ho_imagen_2022,gao2025diffusionmeetsflow}.
\item \emph{Probability-flow ODE velocity ($\ub$-prediction):} Train $\ub^\thetab(\x_t,t)$ to predict the field $\ub_t(\x_t)=\f_t(\x_t)-\tfrac12g_t^2\nabla_{\x_t}\log q_t(\x_t)$ and integrate \cref{eq:reverse-ODE}.  When $g_t>0$, the score can also be recovered through
\begin{align}
\mathbf{s}^\thetab(\x_t,t) = \frac{2}{g_t^2}\Bigl[\f_t(\x_t)-\ub^\thetab(\x_t,t)  \Bigr].
\label{eq:ct-u-to-score}
\end{align}
This parameterisation is particularly useful for exact likelihood computation and deterministic sampling \cite{chen_neural_2019,lipman_flow_2024}.
\end{enumerate}
For schedules derived from the discrete-to-continuous limit, the relationship \eqref{eq:ct-alpha-sigma-relation} connects the SDE coefficients to the marginal parameters. Substituting into the reverse SDE and using any of the above parameterisations yields a practical sampling algorithm.
}

\paragraph{Discrete state spaces.}
In discrete diffusion, the neural network rarely outputs the full $K^d\times K^d$ reverse transition or rate matrix, as it becomes rapidly intractable for large vocabularies.
Instead, it predicts a tractable tokenwise quantity conditioned on the complete noisy sample, and the known forward process converts these predictions into the reverse transition rule.

We parameterise the reverse rate matrix $\hat{\Rseq}_t$ appearing in the reverse master equation $\partial_t \hat{\q}_t = \hat{\Rseq}_t \hat{\q}_t$. All parameterisations below define the off-diagonal rates, the diagonal is then fixed by
$\hat{\Rseq}_t(\y,\y)=-\sum_{\x\neq\y}\hat{\Rseq}_t(\x,\y)$.

For a forward noising process independent across coordinates, \cref{eq:forward-decomp}, the infinitesimal off-diagonal rates have Hamming-one support: sequences $\x$ and $\y$ in $\calX^d$ differ at exactly one coordinate. This does not imply that a practical decoding step updates only one coordinate: finite-step samplers may update several coordinates in parallel \cite{campbell_continuous_2022}.

\discrete[Parameterisation of the reverse process]{
In discrete state spaces, \emph{three main parameterisations} have been introduced. We build on the dimensional factorisation \cref{eq:seq-reverse-hamming1}.

\begin{enumerate}
\item \emph{Direct rate matrix:} Directly parameterise $\hat \Rseq_t^\thetab(\x,\y)$ such that $\hat \Rseq_t^\thetab(\x,\y) \approx \hat\Rseq_t(\x,\y)$ and set the approximate reverse process evolution to $\partial_t\p^\thetab_t=\hat \Rseq_t^\thetab\,\p^\thetab_t$, which can be integrated under regularity assumptions. Under i.i.d.\ forward noising with a shared rate matrix ($\R^{(k)}_t=\R_t$ for all $k$), the sequence-level rate matrix is:
\begin{align}
\hat{\Rseq}^\thetab_t(\x,\y) \coloneqq \sum_{k=1}^{d} \hat \R_t^{\thetab}\!\bigl(\x,\y\bigr)\,
\delta_{\x^{\setminus k},\,\y^{\setminus k}}\ .
\end{align}
This parameterisation can be practical for small alphabets, such as nucleotide sequences, but its cost grows linearly with the vocabulary size and becomes quickly intractable for language modeling.
\item \emph{Posterior-based (clean-data prediction):} For each coordinate $k$ and clean token $a\in \calX$, learn $p^\thetab(x_0^{(k)}=a \cond \x_t , t) \approx q(x_0^{(k)}=a \cond \x_t)$, conditioned on the full noisy sequence. Substituting into \cref{eq:seq-reverse-posterior}, gives for $\x \neq \y$ \cite{austin2021structured,campbell_continuous_2022}:
\begin{align}
\hat{\Rseq}^\thetab_t(\x,\y)
\coloneqq \sum^d_{k=1} \R_{1-t}(y^{(k)},x^{(k)}) \delta_{\x^{\setminus k},\,\y^{\setminus k}} \sum_{a\in\calX}\frac{q_{1-t \cond 0} (y^{(k)}\cond a)}{q_{1-t \cond 0}(x^{(k)}\cond a)}\,p^\thetab(x_0^{(k)}=a \cond \y, t).
\end{align}
This approach is analogous to $\x_0$-prediction in continuous spaces and is widely used in practice.

\item \emph{Concrete score (ratio-based):} Building on \textit{ratio matching} \cite{hyvarinen_extensions_2007, lyu_interpretation_2012}, learn the probability ratio $\s^\thetab(\x,t)$ such that $s^\thetab(\x,t)^{(y^{(k)})} \approx \frac{q_{1-t}(\y)}{q_{1-t}(\x)}\delta_{\x^{\setminus k},\,\y^{\setminus k}}$ for $\y\neq \x$ and form \cite{sun_score-based_2023, meng_concrete_2023, lou_discrete_2024}:
\begin{align}
\hat{\Rseq}^\thetab_t(\x,\y)\coloneqq
\begin{cases}\displaystyle\sum_{k=1}^d \s^\thetab(\y,t)^{(x^{(k)})} \,\R^{(k)}_{1-t}(y^{(k)},x^{(k)})
 \delta_{\x^{\setminus k},\,\y^{\setminus k}}, & \y\neq \x,\\[3mm]
-\displaystyle\sum_{\z\neq \x}\hat{\Rseq}^\thetab_t(\z,\x), & \y=\x.
\end{cases}
\end{align}
This parameterisation mirrors the role of the continuous Stein score and has strong theoretical connections to score matching \cite{meng_concrete_2023}.
\end{enumerate}
}

\paragraph{From reverse dynamics to learning objectives.}
We have now established the reverse-time dynamics for both SDEs and CTMCs, along with their parameterisations. This naturally raises the question: how do we train models to learn these approximations? The answer lies in maximum likelihood estimation via the Evidence Lower Bound (ELBO), which we develop in the next section.

\section{Maximum likelihood and ELBO: derivation of the loss \label{sect:ELBO}}

In generative modeling, our primary objective is to maximise the likelihood of observed data under the learned model. As direct likelihood maximisation is often intractable, we instead maximise a variational lower bound on the log-likelihood: the \textit{evidence lower bound} (ELBO). This strategy aligns with the approach used in Variational Autoencoders (VAEs) \cite{kingma2014auto}. 

Following \cite{kingma_variational_2021},  we first derive the discrete-time ELBO using standard variational inference techniques, then take the continuous-time limit to obtain the corresponding continuous-time objective. A direct derivation in continuous time using path measures is provided in \cref{sect:generator-ELBO} \cite{song_how_2021,benton_denoising_2024,campbell_continuous_2022,lou_discrete_2024}.

\subsection{Discrete-time ELBO}

We consider observations $\x_0 \sim \qdata$, where $\qdata$ denotes an unknown data distribution. For each data point, we define a generative model over a time-indexed trajectory $\x_{0:T}$ via a joint distribution $p^\thetab(\x_{0:T})$ with $t \in \{0,\ldots,T\}$ factorised in reverse time as in \cref{eq:elbo_kl_form}, with state-space-specific kernels parameterised as in \cref{eq:condition-reversal-nn-cont,eq:condition-reversal-nn-disc}. The model is interpreted as generating trajectories in reverse time: the terminal marginal $p(\x_T)$ is a prescribed noise distribution, and the variables $\x_{1:T}$ serve as latent variables mediating the reverse-time generation of $\x_0$. To approximate the intractable posterior over latent trajectories, we introduce a variational distribution $q(\x_{1:T} \cond \x_0)$ conditioned on the observed data.
The log marginal likelihood for a data point $\x_0$ can be rewritten as follows:
\begin{align}
    \log p^\thetab(\x_0) 
    &= \int q(\x_{1:T}\cond \x_0) \log p^\thetab(\x_0) \,\diff \x_{1:T}\\
    &= \int  q(\x_{1:T}\cond \x_0) \log \frac{ p^\thetab(\x_{0:T})}{p^\thetab(\x_{1:T}\cond \x_0)}\,\diff \x_{1:T}\\
    &= \int  q(\x_{1:T}\cond \x_0) \log \frac{ q(\x_{1:T} \cond \x_0) p^\thetab(\x_{0:T})}{ q(\x_{1:T} \cond \x_0) p^\thetab(\x_{1:T}\cond \x_0)}\,\diff \x_{1:T} \\
   &= \underbrace{\E_{q(\x_{1:T}\cond \x_0)} \left[ \log p^\thetab(\x_0 \cond \x_{1:T}) \right]}_{\text{reconstruction term}} 
   - \underbrace{\KL \left( q(\x_{1:T} \cond \x_0) \| p^\thetab (\x_{1:T} ) \right)}_{\text{path-space KL}} \notag \\
  & \qquad  + \underbrace{\KL\left(q(\x_{1:T} \cond \x_0) \| p^\thetab(\x_{1:T} \cond \x_0)\right)}_{\geq 0}.
\end{align}
Since the final KL term is non-negative, we obtain a lower bound by discarding it:
\begin{equation}
\log p^\thetab(\x_0) 
\;\geq\; 
\E_{q(\x_{1:T}\cond \x_0)} \left[ \log p^\thetab(\x_0 \cond \x_{1:T}) \right]
- \KL \left( q(\x_{1:T} \cond \x_0) \| p^\thetab (\x_{1:T} ) \right).
\end{equation}
Exploiting the Markovian structure of both forward process $q(\x_{1:T} \cond \x_0)$ and reverse process $p^\thetab(\x_{T:0})$, the path-space KL divergence decomposes into a sum of step-wise terms (detailed proof in \cref{apx:discrete-time-ELBO}):
\general[Discrete-time ELBO decomposition]{
The negative log-likelihood for a data point $\x_0$ is upper-bounded by:
\begin{align}
    -\log p^\thetab (\x_0)
        &\leq \mathcal{L}_T(\x_0) + \mathcal{L}^\thetab_{\text{diff}}(\x_0) + \mathcal{L}^\thetab_0(\x_0) 
        \label{eq:discrete-elbo-decomp}
\end{align}
where
\begin{align}
    \mathcal{L}_T(\x_0)
        & \coloneqq D_{\mathrm{KL}}\bigl(q(\x_T \cond \x_0)\,\|\,p(\x_T)\bigr) \\
    \mathcal{L}^\thetab_{\text{diff}}(\x_0) \label{eq:discrete-training-objective}
        &\coloneqq
        \sum_{t=2}^T
        \E_{\x_t\sim q(\cdot\cond\x_0)}
        \Bigl[
        D_{\mathrm{KL}}\!\left(q(\x_{t-1}\cond \x_t,\x_0)\,\|\,p^\thetab(\x_{t-1}\cond \x_t)\right)
        \Bigr] \\
    \mathcal{L}^\thetab_0(\x_0)
        &\coloneqq \E_{\x_1 \sim q(\cdot \cond \x_0)} \left[-\log p^\thetab(\x_0\cond \x_1)\right] \ .
\end{align}
In practice, training diffusion models typically amounts to minimising the diffusion loss $\mathcal{L}_{\text{diff}}$, since $\mathcal{L}_T$ is constant and $\mathcal{L}_0$ is often assumed to be negligible compared to the diffusion loss $\mathcal{L}^\thetab_{\text{diff}}$.
}

\begin{remark}
The ELBO decomposes into three types of terms:
\begin{enumerate}[label=(\roman*),leftmargin=*,itemsep=1pt,topsep=2pt]
\item \textit{Prior matching ($\mathcal{L}_T$):} Measures the discrepancy between the forward process endpoint $q(\x_T \cond \x_0)$ and the noise prior $p(\x_T)$. This term is constant with respect to $\thetab$ when the forward process and prior are fixed \cite{ho_denoising_2020,song2021maximum}, and vanishes when the forward process sufficiently destroys information (i.e., $q(\x_T \cond \x_0) \approx p(\x_T)$ for all $\x_0$).

\item \textit{Diffusion loss ($\mathcal{L}_{\mathrm{diff}}$):} Measures how well the learned reverse transitions $p^\thetab(\x_{t-1} \cond \x_t)$ match the true posterior transitions $q(\x_{t-1} \cond \x_t, \x_0)$ (which condition on the data $\x_0$). This term is trainable and forms the core of the optimization objective.

\item \textit{Reconstruction ($\mathcal{L}_0$):} Evaluates the decoder quality at the final denoising step. We consider this term negligible as it goes to $0$ in the continuous time limit \cite{kingma_variational_2021,ho_denoising_2020, campbell_continuous_2022}.
\end{enumerate}
\end{remark}
We now specialise this general framework to continuous and discrete state spaces.

\subsection{Continuous state space}

In continuous state spaces, each denoising term of $\mathcal{L}_{\mathrm{diff}}$ involves the KL divergence between two Gaussian distributions (cf.\ \cref{eq:condition-reversal-nn-cont,eq:discrete-time-reversal-cont,eq:discrete_time_reverse_mu_x0}), which can be computed in closed form:\footnote{For two Gaussians in \(\RR^D\), $\KL\bigl(\mathcal N(\mub_q,\Sigmab_q)\,\|\,\mathcal N(\mub_p,\Sigmab_p)\bigr)
=\frac12\bigl[\log\frac{\det\Sigmab_p}{\det\Sigmab_q}-D+\mathrm{tr}(\Sigmab_p^{-1}\Sigmab_q)+(\mub_p-\mub_q)^\top\Sigmab_p^{-1}(\mub_p-\mub_q)\bigr]$.
Here, $\mub_q=\mub_{t-1\cond t}(\x_t,\x_0)$, $\mub_p=\mub^\thetab(\x_t,t)$, and $\Sigmab_q = \Sigmab_p = \sigma^2_{t-1\cond t}\I$. Note that, while in continuous time the variances must be equal for the ELBO to be finite \citep{archambeau2007gaussian}, in discrete time this does not have to be the case. In fact, the optimal $\Sigmab_p$ is different from $\Sigmab_q$ \citep{nielsen2024diffenc}.}
\begin{align}
    D_{\mathrm{KL}} \big( q_{t-1 \cond t, 0}\,\|\,p^{\thetab}_{t-1 \cond t} \big)
    &=
    \frac12\,
    \frac{\bigl\|\mub_{t-1\cond t}(\x_t,\x_0)-\mub^\thetab(\x_t,t)\bigr\|^2}{\sigma^2_{t-1\cond t}} \ .
\end{align}

\continuous[Per-step loss under common parameterisations]{
The diffusion loss $\mathcal{L}_{\mathrm{diff}}$ becomes a mean-squared error under every standard parameterisation:
\begin{alignat*}{3}
    &\text{Clean-data: } \quad
    &&\E_{\x_t\sim q(\cdot\cond\x_0)}\left[\frac12 \left(\frac{\alpha_{t-1}^2}{\sigma_{t-1}^2}-\frac{\alpha_t^2}{\sigma_t^2}\right)
    \bigl\|\x_0-\x_0^\thetab(\x_t,t)\bigr\|^2 \right] \ ,\\
    &\text{Noise: } \quad
    &&\E_{\x_t\sim q(\cdot\cond\x_0)}\left[\frac12 \left(\frac{\alpha_{t-1}^2 \sigma_t^2}{\alpha_t^2 \sigma_{t-1}^2}-1\right)
    \bigl\|\epsilonb-\epsilonb^\thetab(\x_t,t)\bigr\|^2\right] \ ,\\
    &\text{Score: } \quad
    &&\E_{\x_t\sim q(\cdot\cond\x_0)}\left[\frac12 \left(\frac{\alpha_{t-1}^2}{\sigma_{t-1}^2}-\frac{\alpha_t^2}{\sigma_t^2}\right)
    \frac{\sigma_t^4}{\alpha_t^2} \; 
    \left\|\nabla_{\x_t}\log q(\x_t\cond \x_0)-\s^\thetab(\x_t,t)\right\|^2\right] \ .
\end{alignat*}
}

To derive the continuous-time diffusion loss, we take the limit as
$T \to \infty$. As in
\cref{sect:discrete-to-continuous-time-csp}, set
$\Delta t = 1/T$ and $t_i = i \Delta t$, and define
$\rho_t \coloneqq \sigma_t^2/\alpha_t^2$. For the score parameterisation:
\begin{align}
    \mathcal{L}^\thetab_{\mathrm{diff}}(\x_0)
    &=
    \sum_{i=2}^{T}
    \E_{\x_{t_i}\sim q(\cdot\cond\x_0)}
    \left[
    \frac12\,
    \left(
    \frac{1}{\rho_{t_i-\Delta t}}
    -
    \frac{1}{\rho_{t_i}}
    \right) \alpha_{t_i}^2\rho_{t_i}^2
    \left\|
    \nabla_{\x_{t_i}}\log q(\x_{t_i}\cond\x_0)
    -
    \s^\thetab(\x_{t_i},t_i)
    \right\|_2^2
    \right]
    \nonumber
        \\
        &= \sum_{i=2}^{T}
        \E_{\x_{t_i} \sim q(\cdot \cond \x_0)}
        \left[
        \frac12\,
        \alpha_{t_i}^2
        \frac{\rho_{t_i}}{\rho_{t_i-\Delta t}}
        \frac{\rho_{t_i}-\rho_{t_i-\Delta t}}{\Delta t}
        \bigl\|
        \nabla_{\x_{t_i}}\log q(\x_{t_i}\cond\x_0)
        -
        \s^\thetab(\x_{t_i},t_i)
        \bigr\|_2^2
        \right]
        \Delta t
        \nonumber
        \\
        &= \sum_{i=2}^{T}
        \E_{\x_{t_i} \sim q(\cdot \cond \x_0)}
        \left[
        \frac12\,
        \left(
        \alpha_{t_i}^2
        \frac{\rho_{t_i}-\rho_{t_i-\Delta t}}{\Delta t}
        +o(1)
        \right)
        \bigl\|
        \nabla_{\x_{t_i}}\log q(\x_{t_i}\cond\x_0)
        -
        \s^\thetab(\x_{t_i},t_i)
        \bigr\|_2^2
        \right]
        \Delta t,
        \nonumber
        \\
    \intertext{where
    $\rho_{t_i}/\rho_{t_i-\Delta t}=1+o(1)$ for fixed $t_i>0$, and the
    linear Gaussian forward process satisfies
    $g_t^2=\alpha_t^2\dot{\rho}_t
    =\alpha_t^2\frac{\mathrm d}{\mathrm dt}
    \bigl(\sigma_t^2/\alpha_t^2\bigr)$. Hence, under the usual
    regularity and integrability assumptions, this expression
    is a Riemann sum and converges to:}
    \lim_{\Delta t \to 0}
    \mathcal{L}^\thetab_{\text{diff}}(\x_0)
        &=
        \frac{1}{2}
        \int_{0}^{1}
        \E_{\x_t \sim q(\cdot \cond \x_0)}
        \left[
        g_t^2
        \left\|
        \nabla_{\x_t}\log q(\x_t\cond\x_0)
        -
        \s^\thetab(\x_t,t)
        \right\|_2^2
        \right]
        \,\dt.
    \label{eq:discrete-to-continuous-time-csp-ELBO}
\end{align}

\continuous[Denoising score matching loss]{
In the limit $T\to\infty$, the diffusion loss becomes:
\begin{align}
    \mathcal{L}^\thetab_{\text{diff}}(\x_0)
    =
    \frac{1}{2}\,\int_0^1
    g^2_t \; \E_{\x_t\sim q(\cdot\cond\x_0)}
    \left[\left\|\nabla_{\x_t}\log q_t(\x_t\cond\x_0) - \mathbf s^\thetab(\x_t,t)\right\|_2^2 \right]
    \,\diff t \ ,
    \label{eq:elbo-score-matching}
\end{align}
which is a time-weighted \emph{denoising score matching} objective \cite{vincent_connection_2011}.}

\subsection{Discrete state space}
In discrete state spaces, the intermediate KL divergence terms in $\mathcal{L}_{\mathrm{diff}}$ are:
\begin{align}
    D_{\mathrm{KL}}\bigl(q_{t-1 \cond t, 0}\,\|\,p^{\thetab}_{t-1 \cond t}\bigr)
    = \sum_{\x_{t-1}} q(\x_{t-1} \cond \x_t , \x_0) \log \frac{q(\x_{t-1} \cond \x_t , \x_0)}{p^\thetab(\x_{t-1} \cond \x_t)},
\end{align}
for $t = 2, \ldots, T$.
Applying Bayes' rule and exploiting the Markov property of the forward process for $ q(\x_{t-1} \cond \x_t , \x_0)$, we obtain:
\begin{align}
  D_{\mathrm{KL}}\bigl(q_{t-1 \cond t, 0}\,\|\,p^{\thetab}_{t-1 \cond t}\bigr)  = \sum_{\x_{t-1}}\frac{q(\x_{t-1} \cond \x_0)}{q(\x_t \cond \x_0)}q(\x_{t} \cond \x_{t-1}) \left(  \log \frac{q(\x_{t} \cond \x_{t-1})}{p^\thetab(\x_{t-1} \cond \x_t)} +  \log \frac{q(\x_{t-1} \cond \x_0)}{q(\x_t \cond \x_0) } \right) .
\end{align}
Using the same timestep notation as in continuous spaces (\cref{eq:discrete-to-continuous-time-csp-ELBO}), with $\Delta t = 1/T$ and $t_i = i \Delta t$, the \emph{diffusion loss} becomes:
\begin{align}
    \mathcal{L}^\thetab_{\text{diff}}(\x_0)  = \sum_{i=2}^T \E_{\x_{t_i} \sim q( \cdot \cond \x_0)} \left[ \sum_{\x_{t_{i-1}}}\frac{q(\x_{t_{i-1}} \cond \x_0)}{q(\x_{t_i} \cond \x_0)} q(\x_{t_i} \cond \x_{t_{i-1}})\, \mathcal{F}_i\right] \ ,
    \label{eq:elbo-dsp-discrete-time-ti}
\end{align}
with
\begin{align}
     \mathcal{F}_i =\log q(\x_{t_i} \cond \x_{t_{i-1}}) -  \log p^\thetab(\x_{t_{i-1}} \cond \x_{t_i}) 
+  \log \frac{q(\x_{t_{i-1}} \cond \x_0)}{q(\x_{t_i} \cond \x_0) } \ .
\end{align}
To derive the continuous-time limit, we analyse each term of $\mathcal{F}_i$ separately.
For two discrete states $\x, \y \in \calX^d$ with $\lim_{\Delta t \to 0}\x_{t_{i-1}}= \y$ and $\lim_{\Delta t \to 0}\x_{t_{i}}= \x$, we have:
\begin{align}
    \lim_{\Delta t \to 0} \frac{q(\x_{t_{i-1}} \cond \x_0)}{q(\x_{t_i} \cond \x_0)}= \delta_{\x,\y} + (1-\delta_{\x,\y}) \frac{q(\y \cond \x_0)}{q(\x \cond \x_0)} \ .
    \label{eq:quotient-taylor-discrete-elbo}
\end{align}
For the transitions, following \citet{campbell_continuous_2022} and \citet{von2025generalized}, and using the jump equation \eqref{eq:def-transition-rate-matrix}, the logarithm expands as:
\begin{align}
    \log q(\x_{t_i}\cond\x_{t_{i-1}}) = \begin{cases}
    \Rseq_t(\x_{t_i},\x_{t_{i-1}})\,\Delta t + o(\Delta t), & \text{if } \x_{t_i} = \x_{t_{i-1}}\ ,\\
    \log \bigl(\Rseq_t(\x_{t_i},\x_{t_{i-1}})\,\Delta t\bigr) + o(1), & \text{if } \x_{t_i} \neq \x_{t_{i-1}} \ .
    \end{cases}
\end{align}
An analogous expression holds for $\log p^\thetab(\x_{t_{i-1}} \cond \x_{t_i})$ with $\Rseq_t$ replaced by $\hat \Rseq_t^\thetab$.
After combining the Taylor expansions and keeping only first-order terms in $\Delta t$, the product term becomes:
\begin{align}
     q(\x_{t_i} \cond \x_{t_{i-1}}) \Bigl( \log q(\x_{t_i} \cond \x_{t_{i-1}}) -  \log p^\thetab(\x_{t_{i-1}} \cond \x_{t_i}) \Bigr) = \mathcal{T}_{\text{diag}} + \mathcal{T}_{\text{off-diag}} + o(\Delta t),
     \label{eq:Tdiag-Toffdiag-discrete-elbo}
\end{align}
where
\begin{align}
\mathcal{T}_{\text{diag}} &= \delta_{\x_{t_i},\x_{t_{i-1}}}\, \Bigl[
         \Rseq_t(\x_{t_i},\x_{t_{i-1}})
         - \hat\Rseq^\thetab_t(\x_{t_{i-1}},\x_{t_i})\Bigr] \Delta t\ ,\\
\mathcal{T}_{\text{off-diag}} &= (1-\delta_{\x_{t_i},\x_{t_{i-1}}})\,\Rseq_t(\x_{t_i},\x_{t_{i-1}})\,
\log \!\frac{\Rseq_t(\x_{t_i},\x_{t_{i-1}})} {\hat\Rseq^\thetab_t(\x_{t_{i-1}},\x_{t_i})} \,\Delta t\ .
\end{align}

Combining \cref{eq:quotient-taylor-discrete-elbo,eq:Tdiag-Toffdiag-discrete-elbo}, taking the limit $\Delta t \to 0$ and using the column-sum-to-zero property $\Rseq_t(\x,\x) = -\sum_{\y \neq \x} \Rseq_t(\y,\x)$, we recover the following ELBO \cite{campbell_continuous_2022,lou_discrete_2024,chen_convergence_2024,von2025generalized}:
\begin{align}
     \mathcal{L}^\thetab_{\text{diff}}(\x_0) = \int_0^1 \E_{\x_t \sim q(\cdot\cond \x_0)} \Bigg[ &\sum_{\y \neq \x_t} \Bigg( \hat\Rseq^\thetab_t(\y,\x_t)
         - \Rseq_t(\y,\x_t) \\&+\Rseq_t(\x_t,\y)\frac{q(\y \cond \x_0)}{q(\x_t \cond \x_0)} \left[ \log \frac{\Rseq_t(\x_t,\y)} {\hat\Rseq^\thetab_t(\y,\x_t)}
         +   \log \frac{q(\y \cond \x_0)}{q(\x_t \cond \x_0)} \right] \Bigg) \Bigg] \dt \ ,
 \end{align}

Grouping constants and focusing on learnable terms yields the \emph{Denoising Score Entropy} objective.

\discrete[Denoising score entropy loss]{

The continuous-time ELBO yields the \emph{denoising score entropy} objective, which equals the diffusion loss up to terms independent of $\thetab$:
\begin{align}
    \mathcal{L}^\thetab_\text{DSE}(\x_0)
    =
    \int_0^1
    \E_{\x_t\sim q(\cdot\cond\x_0)}
    \Bigg[\sum_{\y\neq \x_t}
    \Bigl(\hat\Rseq_t^\thetab(\y,\x_t)
    -\Rseq_t(\x_t,\y)\tfrac{q(\y\cond\x_0)}{q(\x_t\cond\x_0)}\log \hat\Rseq_t^\thetab(\y,\x_t)\Bigr)\Bigg]
    \,\dt \ ,
    \label{eq:seq-elbo}
\end{align}
which can be rewritten in terms of concrete score \cite{lou_discrete_2024,meng_concrete_2023}:
\begin{align*}
    \mathcal{L}^\thetab_\text{DSE}(\x_0) = \int_0^1 \E_{\x_t \sim q(\cdot \cond \x_0)} \left[ \sum_{\y \neq \x_t} \Rseq_t\!\left(\x_t, \y\right)\left( \langle \s^\thetab\left(\x_t, t\right), \e_\y \rangle - \frac{q(\y \cond \x_0)}{q(\x_t \cond \x_0)} \log\langle \s^\thetab(\x_t, t), \e_\y \rangle\right) \right] \dt \ .
\end{align*}
Or posterior:
\begin{align}
    \mathcal{L}^\thetab_\text{DSE}(\x_0)
    =
    \int_0^1\!
    \E_{\x_t\sim q(\cdot\cond\x_0)}
    \Biggl[\sum_{\y\neq \x_t} \Rseq_t(\x_t,\y)
    \Bigl(\sum_{\tilde\x_0}\tfrac{q(\y\cond\tilde\x_0)}{q(\x_t\cond\tilde\x_0)}p^\thetab(\tilde\x_0\cond \x_t,t)
    \nonumber \\-\tfrac{q(\y\cond\x_0)}{q(\x_t\cond\x_0)}\log \Bigl[\sum_{\tilde\x_0}\tfrac{q(\y\cond\tilde\x_0)}{q(\x_t\cond\tilde\x_0)}p^\thetab(\tilde\x_0\cond \x_t,t)\Bigr] \Bigr) \Biggr]
    \,\dt \ .
\end{align}
Only transitions with non-zero forward rate $\Rseq_t(\x,\y)$ contribute, enabling sums over local moves (e.g., single-token changes).
}

Since this sum is weighted by the transition rates $\Rseq_t(\x,\y)$, only transitions with non-zero probability contribute to the loss. Therefore, the sum can be restricted to local transitions (e.g., Hamming distance 1 or single-token edits) \cite{campbell_continuous_2022, lou_discrete_2024}.

\discrete[Reduction to token-level objective (used in practice)]{
Since $\Rseq_t$ can be factorized for Hamming-1 transitions,
\begin{align}
\Rseq_t(\x_t,\y)
 = 
\sum_{k=1}^{d}\R^{(k)}_t(x_t^{(k)},y^{(k)})\,
\delta_{\x_t^{\setminus k},\,\y^{\setminus k}}.
\end{align}
Moreover, the forward marginals factorise as
$q(\x_t\cond \x_0)=\prod_{k=1}^d q(x_t^{(k)}\cond x_0^{(k)})$,
so that the ratio tilt cancels on unchanged coordinates:
\begin{align}
\frac{q(\y\cond \x_0)}{q(\x_t\cond \x_0)}
=
\frac{q(y^{(k)}\cond x_0^{(k)})}{q(x_t^{(k)}\cond x_0^{(k)})}
\quad
\text{whenever $\y$ and $\x_t$ differ only at coordinate $k$.}
\end{align}
Expanding the sequence-level sum in \eqref{eq:seq-elbo} thus yields the token-level objective:
\begin{align}
    {\mathcal{L}_{\text{DSE}}^{\text{seq}^\thetab}}(\x_0)
    =
    \int_0^1
    \E_{\x_t\sim q(\cdot\cond \x_0)}
    \Biggl[
    &\sum_{k=1}^{d}
    \sum_{y^{(k)}\neq x_t^{(k)}}
    \Biggl(
    \hat{\R}_t^{\thetab,(k)}(y^{(k)},\x_t)
    \\& -\R_t(x_t^{(k)},y^{(k)})\,
    \tfrac{q(y^{(k)}\cond x_0^{(k)})}{q(x_t^{(k)}\cond x_0^{(k)})}
    \log \hat{\R}_t^{\thetab,(k)}(y^{(k)},\x_t)
    \Biggr)
    \Biggr]
    \diff t,
    \label{eq:token-elbo}
\end{align}
which is the form implemented in practice (typically with shared per-token rates $\R^{(k)}_t\coloneqq \R_t$ for all $k$), where the parameterised reverse rate matrix remains conditioned on the full noised sequence $\x_t$.}

We now specialise to (absorbing) masked diffusion with interpolation marginals.
$
q_t(x\!\cond\!x_0)=\alpha_t\,\delta_{x_0}(x)+(1-\alpha_t)\,\delta_{\text{[MASK]}}(x),
$
for which the only off-diagonal forward moves are token $\to$ [MASK].

The ELBO simplifies to a time-weighted masked-token cross-entropy, recovering the MLM loss used in practice \cite{ou_your_2024,devlin_bert_2018}.

\discrete[ELBO reduction to MLM loss.]{Under the (absorbing) masked diffusion process with interpolation marginals
$
q_t(x|x_0)=\alpha_t\,\delta_{x_0}(x)+(1-\alpha_t)\,\delta_{\text{[MASK]}}(x)
$, the token level loss \cref{eq:token-elbo} reduces to weighted the MLM loss:
\begin{align}
    \mathcal{L}^\thetab_\text{MLM}(\x_0)
    =
    \int_0^1
    \E_{\x_t\sim q(\cdot\cond\x_0)}\frac{\alpha'_t}{1-\alpha_t}\sum_{k=1}^d
    \delta_{\text{[MASK]}}(x_t^{(k)})\log p^\thetab(x_0^{(k)} \cond \x_t,t)
    \,\diff t.
    \label{eq:mlm-ELBO}
\end{align}
}

\begin{proof}[Derivation sketch:]

We consider the token-level ELBO using posterior prediction
, up to terms independent of $\thetab$,
\begin{align}
    \mathcal{L}^\thetab_\text{DSE}(\x_0)
    &=
    \int_0^1
    \E_{\x_t\sim q(\cdot\cond\x_0)}
    \sum_{k=1}^d
    \sum_{y\neq x^{(k)}_t}
    \Biggl[
    \underbrace{
    \R_t(x^{(k)}_t,y)
    \sum_{\tilde x_0}
    \frac{q_t(y\cond\tilde x_0)}
         {q_t(x^{(k)}_t\cond\tilde x_0)}
    \,
    p^\thetab(x_0^{(k)}=\tilde x_0\cond\x_t,t)
    }_{\text{(A)}}
    \nonumber\\[-0.2em]
    &\hspace{8em}
    -
    \underbrace{
    \R_t(x^{(k)}_t,y)
    \frac{q_t(y\cond x_0^{(k)})}
         {q_t(x^{(k)}_t\cond x_0^{(k)})}
    \log\!\Bigl[
    \sum_{\tilde x_0}
    \frac{q_t(y\cond\tilde x_0)}
         {q_t(x^{(k)}_t\cond\tilde x_0)}
    \,p^\thetab(x_0^{(k)}=\tilde x_0\cond\x_t,t)
    \Bigr]
    }_{\text{(B)}}
    \Biggr]
    \,\diff t.
    \label{eq:token-posterior-ELBO}
\end{align}

Under the (absorbing) masked diffusion with interpolation marginals
\[
q_t(x\cond x_0)
=
\alpha_t\,\delta_{x_0}(x)
+
(1-\alpha_t)\,\delta_{\text{[MASK]}}(x),
\]
the only off-diagonal forward moves are token $\to$ \text{[MASK]}:
\begin{align*}
\R_t(\text{[MASK]},y)
&=\beta_t,
\qquad y\neq\text{[MASK]},\\
\R_t(x,y)
&=0
\quad
\text{for all other off-diagonal pairs }(x,y),\\
\beta_t
&\coloneqq-\frac{\alpha'_t}{\alpha_t}\geq0.
\end{align*}

We are thus restricted to
$x_t^{(k)}=\text{[MASK]}$ and
$y\neq\text{[MASK]}$. So we have that the sum in (A) is reduced to a $\thetab$-independent term:
\begin{align*}
\sum_{y\neq\text{[MASK]}}
\R_t(\text{[MASK]},y)
\sum_{\tilde x_0\neq\text{[MASK]}}
\frac{q_t(y\cond\tilde x_0)}
     {q_t(\text{[MASK]}\cond\tilde x_0)}
p^\thetab(x_0^{(k)}=\tilde x_0\cond\x_t,t)
&=
\beta_t\frac{\alpha_t}{1-\alpha_t}
\sum_{y\neq\text{[MASK]}}
p^\thetab(x_0^{(k)}=y\cond\x_t,t)
\\
&=
\beta_t\frac{\alpha_t}{1-\alpha_t}
=
-\frac{\alpha'_t}{1-\alpha_t}.
\end{align*}

For (B), the interpolation formula gives
\[
\frac{q_t(y\cond x_0^{(k)})}
     {q_t(\text{[MASK]}\cond x_0^{(k)})}
=
\frac{\alpha_t}{1-\alpha_t}
\,\delta_{x_0^{(k)}}(y).
\]
and,
\begin{align*}
\sum_{\tilde x_0\neq\text{[MASK]}}
\frac{q_t(y\cond\tilde x_0)}
     {q_t(\text{[MASK]}\cond\tilde x_0)}
\,p^\thetab(x_0^{(k)}=\tilde x_0\cond\x_t,t)
=
\frac{\alpha_t}{1-\alpha_t}
\,p^\thetab(x_0^{(k)}=y\cond\x_t,t).
\end{align*}
It follows that
\begin{align*}
&\sum_{y\neq\text{[MASK]}}
\R_t(\text{[MASK]},y)
\frac{q_t(y\cond x_0^{(k)})}
     {q_t(\text{[MASK]}\cond x_0^{(k)})}
\log\!\left[
\sum_{\tilde x_0\neq\text{[MASK]}}
\frac{q_t(y\cond\tilde x_0)}
     {q_t(\text{[MASK]}\cond\tilde x_0)}
\,p^\thetab(x_0^{(k)}=\tilde x_0\cond\x_t,t)
\right]
\\
&\qquad=
\beta_t\frac{\alpha_t}{1-\alpha_t}
\log\!\left[
\frac{\alpha_t}{1-\alpha_t}
p^\thetab(x_0^{(k)}\cond\x_t,t)
\right].
\end{align*}

Accounting for the minus sign before (B), and dropping the $\thetab$-independent terms, we recover the
masked language modeling (MLM) objective:
\begin{align}
    \mathcal{L}^\thetab_\text{MLM}(\x_0)
    =
    \int_0^1
    \E_{\x_t\sim q(\cdot\cond\x_0)}
    \frac{\alpha'_t}{1-\alpha_t}
    \sum_{k=1}^d
    \delta_{\text{[MASK]}}(x_t^{(k)})
    \log p^\thetab(x_0^{(k)}\cond\x_t,t)
    \,\diff t.
    \label{eq:mlm-loss-final}
\end{align}
\end{proof}

Having established the training objectives, we now have all components necessary to model and train a generative process that learns the data distribution $\qdata$, enabling the generation of new samples.

\section{Unified view under the infinitesimal generator} \label{sect:generator-perspective}

\begin{figure}[h!]
    \centering
    \includegraphics[width=1\linewidth]{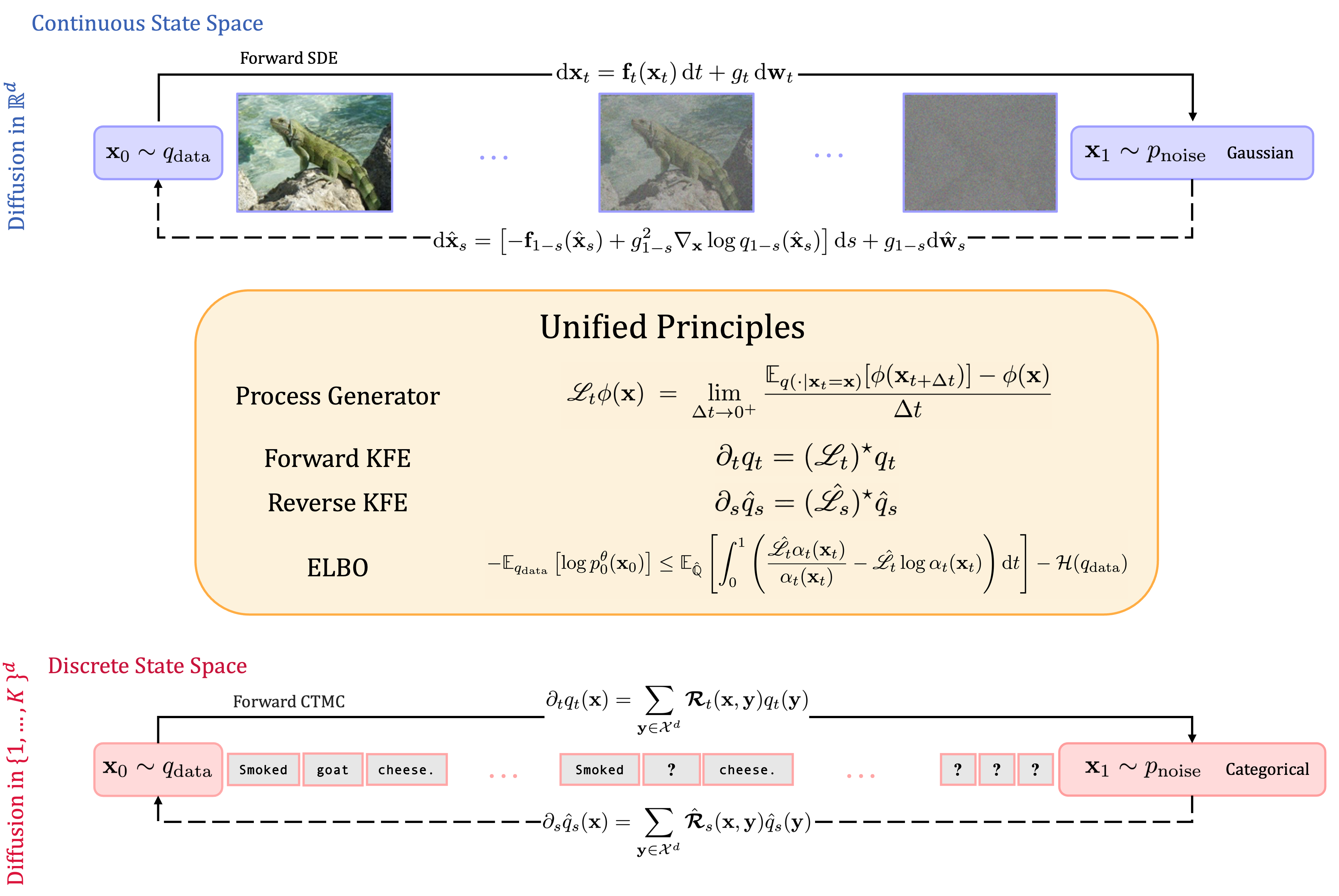}
    \caption{Unified perspective on diffusion models in continuous and discrete state spaces.
    \textcolor{figblue}{(Top)}~An image $\x_0 \sim q_{\text{data}}$ is corrupted via a forward SDE to Gaussian noise $\x_1 \sim p_{\text{noise}}$; a reverse SDE reconstructs the clean data. 
    \textcolor{figred}{(Bottom)}~A discrete sequence is corrupted via a forward CTMC through masking; a reverse CTMC recovers it. 
    \textcolor{general}{(Middle)}~Both processes are governed by infinitesimal generators $\Lgen_t$ and Kolmogorov forward equations, providing a unified theoretical framework.}
    \label{fig:main-fig}
\end{figure}

In this section we generalise the results from \cref{sect:continuous-diffusion,sect:ELBO} and show that it is also possible to define a generative process directly through a powerful tool called the Markov process \emph{infinitesimal generator}. We then reverse the process by identifying the generator of the time-reversed process. This operator-theoretic perspective not only unifies continuous and discrete diffusion under a single formalism but also offers an alternative, more direct route to the ELBO that clarifies the fundamental structure of time reversal in Markov processes. Moreover, this framework naturally extends to broader classes of stochastic processes beyond SDEs and CTMCs, such as jump-diffusions and piecewise-deterministic processes \cite{bertazzi2024piecewise,han2025distillkac}. Recent work exploring this perspective include \cite{holderrieth_generator_2024,benton_denoising_2024,ren2025unified,rojas2025diffuse}.

Consider a potentially time-inhomogeneous Markov process $\{\x_t\}_{t\in [0,1]}$ on a state space $\calX^d$.\footnote{A Markov process is called time-inhomogeneous if the transition probabilities $q_{t|s}$ depend on both the source time $s$ and target time $t$; it is time-homogeneous if $q_{t|s}=q_{t-s}$ for all $s,t$.} The \emph{(spatial) infinitesimal generator} $\Lgen_t$, is the fundamental operator characterising the (spatial) instantaneous local dynamics of the process.  It acts on a class of test functions $\phi$ belonging to its domain $\calD(\Lgen_t)$, which we assume includes all functions of interest.\footnote{In continuous spaces, $\calD(\Lgen_t)$ typically consists of twice-differentiable functions with compact support $C_0^2(\calX^d)$; in discrete spaces, $\calD(\Lgen_t) = \ell_\infty(\calX^d)$ suffices.} For a comprehensive analysis of these operators, we refer the reader to \citet{ruschendorf2016comparison,pavliotis_stochastic_2014,karlin_second_1981}.

\general[Infinitesimal generator]{
For a test function $\phi \in \calD(\Lgen_t)$, the \emph{(spatial) infinitesimal generator} $\Lgen_t$ is defined as:
\begin{equation}
\Lgen_t \phi(\x_t)
\coloneqq
\lim_{\Delta t \to 0^+}
\frac{\E_{q(\x_{t + \Delta t} \cond \x_t)}\!\left[\phi(\x_{t+\Delta t})\right] - \phi(\x_t)}{\Delta t}\ ,
\label{eq:gen-def}
\end{equation}
whenever the limit exists. Intuitively, $\Lgen_t\phi(\x_t)$ quantifies the instantaneous rate of change of the expected value of the test function $\phi$ when the process is at state $\x_t$ at time $t$.}

\begin{remark}[Extended generator for inhomogeneous processes]
For a time-inhomogeneous process $\{\x_t\}_{t\in [0,1]}$, we note that it can be turned into a time-homogeneous process by defining augmented states $\w_t \coloneqq (\x_t,t)$, effectively expanding its state space to $\calX^d \times [0,1]$. In this expanded formulation, the test function $\phi$ acts on $\calX^d \times [0,1]$, and we define  $\phi_t(\x) \coloneqq \phi(\x,t)$. In the general case of time-inhomogeneous processes and time-dependent test functions $\phi_t(\x)$, the entire evolution is captured by the \emph{extended generator} (see \cref{apx:time-inh-generator-decomposition} for a proof of this decomposition):
\begin{equation}
\Ggen_t \phi_t(\x) 
 =  
\partial_t \phi_t(\x) + \Lgen_t \phi_t(\x) \ .
\label{eq:extended-generator}
\end{equation}
In this work, we simplify the presentation by focusing solely on the spatial part ($\Lgen_t$) of the infinitesimal generator ($\Ggen_t$).
\end{remark}

This formalism provides a \emph{local} description of how expectations evolve under the Markov process. The generator framework is powerful because the same operator, $\Lgen_t$, governs both the evolution of expectations and, through its adjoint, the \emph{global} evolution of probability distributions. We now introduce these concepts and show this connection.

\subsection{Forward process}

The infinitesimal generator $\Lgen_t$ characterises local dynamics through expectations of test functions. To obtain a global picture of how the \emph{marginal distribution} $q_t$ evolves, we leverage the adjoint operator, \cref{eq:adjoint-definition}, and exploit the duality between forward and backward evolution.

Under standard regularity and smoothness assumptions, allowing for the exchange of limit and expectation, one can simply derive the evolution of the marginal expectation of a test function $\phi$ by differentiating and applying the law of total expectation:
\begin{align}
    \frac{\mathrm{d}}{\mathrm{d}t}\, \E_{q(\x_t)}[\phi(\x_t)] 
    &= \lim_{\Delta t \to 0^+} \frac{\E[\phi(\x_{t+\Delta t})] - \E[\phi(\x_t)]}{\Delta t} \nonumber \\
    &= \lim_{\Delta t \to 0^+} \frac{1}{\Delta t}\, \E_{q(\x_t)}\Big[\E_{q(\x_{t+\Delta t} \cond \x_t)}\big[\phi(\x_{t+\Delta t})\big] - \phi(\x_t) \Big] \nonumber
    \\
    &= \E_{q(\x_t)}\bigg[ \underbrace{\lim_{\Delta t \to 0^+} \frac{\E_{q(\x_{t+\Delta t} \cond \x_t)}[\phi(\x_{t+\Delta t})] - \phi(\x_t)}{\Delta t}}_{\Lgen_t \phi(\x_t)}\bigg] \ .
    \label{eq:dynkin-step2}
\end{align}
Integrating \cref{eq:dynkin-step2} up to a time $\tau$, for any test function $\phi \in \cap_{s \leq \tau}\calD(\Lgen_s)$ yields a simplified \emph{Dynkin's formula}:
\begin{align}
    \E[\phi(\x_\tau)] - \E[\phi(\x_0)]
    = \int_0^\tau\E\left[\Lgen_t \phi(\x_t)\right]  \dt\ . 
    \label{eq:dynkin-formula-lose}
\end{align}
The general formulation is provided below.

\general[Dynkin's formula]{
For any time $\tau \in [0,1]$ and test function $\phi \in \cap_{s \leq \tau}\calD(\Lgen_s)$, \emph{Dynkin's formula} \cite{kallenberg1997foundations,dynkin1965markov} expresses the change in $\phi$ along a sample path as:
\begin{align}
\phi(\x_\tau) - \phi(\x_0) = \int_0^\tau(\partial_t+\Lgen_t )\phi(\x_t) \dt + M_\tau (\x_\tau), \label{eq:general-Dynkin}
\end{align}
where $M_\tau$ is a martingale.\footnote{A martingale is a stochastic process whose expected future value, conditioned on all past information, equals its present value. Intuitively, it models a fair game with no systematic drift.} Taking expectations and using the fact that $\E_{q_\tau(\x_\tau\cond\x_0)}[M_\tau(\x_\tau)] = 0$ yields: 
\begin{align}
    \E_{q_t(\x_\tau\cond\x_0)}[\phi(\x_\tau)] - \phi(\x_0)
    = \E\left[\int_0^\tau\Lgen_t \phi(\x_t) \dt\right] .
    \label{eq:dynkin-formula}
\end{align}
Informally, \cref{eq:dynkin-formula} can be recovered from \cref{eq:dynkin-step2} starting from the conditional expectation w.r.t. $\x_0$ and exchanging the time and space integrals.
}

By the definition of the adjoint operator $(\Lgen_t)^*$ with respect to the inner product $\langle f, g \rangle = \int f(\x)\, g(\x)\, \diff\mu(\x)$, we have:
\begin{equation}
    \langle \Lgen_t \phi, q_t \rangle 
    =
    \langle \phi, (\Lgen_t)^* q_t \rangle\ .
    \label{eq:adjoint-definition}
\end{equation}
Substituting into \cref{eq:dynkin-step2}, for all $\phi \in \calD(\Lgen_t)$, we obtain two equivalent formulations known as the \emph{Kolmogorov backward equation} (KBE) and \emph{Kolmogorov forward equation} (KFE):
\begin{align}
    \frac{\diff}{\diff t} \langle \phi, q_t \rangle 
    &=
    \langle\Lgen_t  \phi, q_t \rangle
    \quad &\text{(KBE)},\\
    &= 
    \langle \phi, (\Lgen_t)^* q_t \rangle
    \quad &\text{(KFE)}.
    \label{eq:duality-argument}
\end{align}
Since this identity holds for all test functions $\phi$, by the fundamental lemma of calculus of variations we conclude that $q_t$ must satisfy the differential equation $\partial_t q_t = (\Lgen_t)^* q_t$. This is formalised below.

\general[General Kolmogorov forward equation]{Assuming sufficient regularity of $q_t$ in $\x$ and $t$, and given an initial distribution $\qdata$:
\begin{equation}\label{eq:KFE-eq}
\partial_t q_t  =  (\Lgen_t)^* q_t,
\qquad
q_0=\qdata.
\end{equation}
Moreover, while often used for the marginals, the \textit{KBE} and \emph{KFE} also extend to transition distributions. For any times $0 \leq s<t \leq 1$ and $\x,\y \in \calX^d$:
\begin{align}
\partial_t\, q(\x_t=\x \cond \x_s=\y) &= (\Lgen_t)^*\, q(\x_t=\x \cond \x_s=\y), \nonumber
\\& \qquad \qquad \text{(forward, acts on terminal variable $\x$).} \label{eq:KFE-eq-transition}\\[0.8em]
\partial_s\, q(\x_t=\x \cond \x_s=\y) &= -\, \Lgen_s\, q(\x_t=\x \cond \x_s=\y), \nonumber
\\& \qquad \qquad \text{(backward, acts on initial variable $\y$)}. \label{eq:backward-transition-KBE}
\end{align}
}

\begin{remark}[Backward Kolmogorov equation]
Despite its name, the backward equation \eqref{eq:backward-transition-KBE} does \textbf{not} reverse the stochastic process, it merely describes how transition probabilities evolve when we differentiate with respect to the initial time $s$ rather than the terminal time $t$. Both equations still characterise the same forward-time process. The actual reverse process, which generates samples by running backward in time from noise to data, requires a different construction discussed in \cref{sec:reverse-generator}.
\end{remark}

\paragraph{From forward diffusion to reverse generation.} Having established the forward process through the KFE \eqref{eq:KFE-eq}, which describes how an initial data distribution $\qdata$ evolves toward a noise distribution under the adjoint generator $\Lgen_t^\star$, we now address the central challenge of generative modeling: \emph{reversing} this process. Given a sample from the terminal noise distribution $q_1$, how can we generate samples from $\qdata$ by running the process backward in time?

The key insight is that time reversal can be characterised entirely through the generator formalism. Under appropriate regularity conditions, the reverse-time dynamics are themselves Markovian and can be expressed through a \emph{reverse generator} $\hat{\Lgen}_t$. This reverse generator depends on both the forward generator $\Lgen_t$ and the forward marginals $\{q_t\}_{t \in [0,1]}$, providing a unified framework for backward generation across continuous and discrete state spaces.

\subsection{Reverse Process}\label{sec:reverse-generator}
The natural starting point for time reversal is to define the \emph{reversed process} $\hat{\x}_s \coloneqq \x_{1-s}$ for $s \in [0,1]$, which runs from noise at $s=0$ to data at $s=1$. Its marginal distributions are denoted $\hat{q}_s \coloneqq q_{1-s}$. By the chain rule and the forward KFE \eqref{eq:KFE-eq}, these reversed marginals satisfy:
\begin{align}
    \partial_s \hat{q}_s
    = -(\Lgen_{1-s})^{*}\,\hat{q}_s \ .
    \label{eq:time-reversal-marginal}
\end{align}
However, this cannot directly be used to \emph{simulate} the reverse process.
To see why, observe that the forward generator $\Lgen_{1-s}$, when expressed in terms of $\hat{\x}_s$, describes transitions to the \emph{previous} reversed time:
\begin{equation}
    \Lgen_{1-s} \phi(\hat{\x}_s) = \lim_{\Delta s\to0^+}\frac{\E_{q( \hat{\x}_{s-\Delta s} \cond \hat{\x}_s)}\left[\phi\big(\hat{\x}_{s-\Delta s}\big) \right]-\phi(\hat{\x}_s)}{\Delta s}\ .
\end{equation}
Thus, this operator computes the expected value at the previous reversed time $s - \Delta s$ given the current state at time $s$.

To obtain a tractable simulation scheme, we need to compute transitions in the \emph{reverse} direction: given the current state $\hat{\x}_s$, we must predict the state at the \emph{next} reversed time $s + \Delta s$. For this, we seek a \emph{reversed generator} $\hat{\Lgen}_s$ such that the reversed process $\{\hat{\x}_s\}_{s\in[0,1]}$ satisfies the standard generator definition with forward-in-$s$ conditioning:
\begin{equation}\label{eq:reversed-generator-def}
\hat{\Lgen}_s \phi(\hat{\x}_s)  =  \lim_{\Delta s\to0^+}\frac{\E_{q(\hat{\x}_{s+\Delta s} \cond \hat{\x}_{s})}\left[\phi\big(\hat{\x}_{s+\Delta s}\big) \right]-\phi(\hat{\x}_s)}{\Delta s}\ .
\end{equation}
The following result provides an explicit formula for $\hat{\Lgen}_s$ in terms of the forward generator $\Lgen_t$ and the marginal density $q_t$.

\general[Reversed generator]{\label{prop:reversed-generator}
The reversed process $\{\hat{\x}_s\}_{s\in[0,1]}$ with $\hat{\x}_s = \x_{1-s}$ has infinitesimal generator \cite{palmowski2002technique}:
\begin{equation}\label{eq:reversed-generator}
\hat{\Lgen}_s \phi(\x)
= \frac{1}{q_{1-s}(\x)}\Big((\Lgen_{1-s})^\star\left[q_{1-s}\phi\right](\x) - \phi(\x)\,(\Lgen_{1-s})^\star \left[q_{1-s}\right](\x) \Big)\ .
\end{equation}
}

We can verify through the KFE that this reversed generator yields the correct evolution of the marginals of the reverse process. While this offers a useful consistency check, it is weaker than establishing that this operator is in fact the generator of the reverse process itself (as opposed to merely generating a process with matching marginals). That derivation is provided in \cref{apx:reverse-generator}.

\begin{proof}[Derivation sketch:]
The adjoint of $\hat{\Lgen}_s$ satisfies for all test function $\phi \in \mathcal{D}(\Lgen_{1-s})$:
\begin{align}
    \int [(\hat{\Lgen}_s)^* \hat{q}_s](\x)\,\phi(\x)\,\diff\mu(\x)
    &= \int \hat q_s(\x)\,[\hat{\Lgen}_s\phi](\x)\,\diff\mu(\x) \nonumber \\
    &= \int \frac{\hat q_s(\x)}{q_{1-s}(\x)}\,(\Lgen_{1-s})^\star\left[q_{1-s}\phi\right](\x)\,\diff\mu(\x)\\
     & \qquad  - \int \frac{\hat q_s(\x)}{q_{1-s}(\x)}\,\phi(\x)\,[(\Lgen_{1-s})^\star q_{1-s}](\x)\,\diff\mu(\x) \nonumber \\
    &= \underbrace{\int (\Lgen_{1-s})^\star\left[q_{1-s}\phi\right](\x)\,\diff\mu(\x)}_{\mathcal I_1}
     - \underbrace{\int \phi(\x)\,(\Lgen_{1-s})^\star q_{1-s}(\x)\,\diff\mu(\x)}_{\mathcal I_2}\ .
    \label{eq:v1}
\end{align}
For the first term $\mathcal I_1$, we use adjointness with the constant function $\mathbf 1$ and notice that the generator applied to a constant test function is null:
\begin{align}
    \mathcal I_1
    = \int (\Lgen_{1-s})^\star\left[q_{1-s}\phi\right]\,\diff\mu
    = \int (q_{1-s}\phi)\,(\Lgen_{1-s}\mathbf 1)\,\diff\mu = 0\ .
\end{align}
For the second term,
\begin{align}
\mathcal I_2 = \int \phi\,[(\Lgen_{1-s})^\star q_{1-s}]\,\diff\mu.
\end{align}
Combining the two, \cref{eq:v1} becomes
\begin{align}
    \int [(\hat{\Lgen}_s)^* \hat{q}_s]\,\phi\,\diff\mu
    = - \int \phi\,[(\Lgen_{1-s})^\star q_{1-s}]\,\diff\mu
    = - \int \phi\,[(\Lgen_{1-s})^\star \hat q_s]\,\diff\mu\ ,
\end{align}
and therefore $(\hat{\Lgen}_s)^* \hat{q}_s = -(\Lgen_{1-s})^\star \hat{q}_s$, which is exactly \cref{eq:time-reversal-marginal}. 
\end{proof}

\subsection{Evidence Lower Bound (ELBO)\label{sect:generator-ELBO}}

Having established the reverse-generator formalism, we now introduce a tractable, parameterised reverse process intended to approximate the true (but intractable) time-reversal dynamics, and which can subsequently be used for data generation. As in \cref{sect:ELBO}, our derivation proceeds by constructing a variational lower bound on the log-likelihood---the evidence lower bound (ELBO)---which serves as the training objective for diffusion models. In contrast to \cref{sect:ELBO}, however, the present treatment is formulated directly at the level of generators, thereby providing a unified perspective on the ELBO that applies uniformly to both continuous and discrete state spaces.

Consider a Markov process $\{\x_t\}_{t\in [0,1]}$ on a state space $\calX^d$ with generator $\Lgen_t$ and path measure $\Qpath$. We denote by $\{\hat\x_s\}_{s\in [0,1]}$ the time-reversal of $\{\x_t\}_{t\in [0,1]}$ with generator $\hat\Lgen_s$ and path measure $\hat \Qpath$ such that its marginals satisfy $\hat q_{s}=q_{1-s}$ for all $s\in[0,1]$.

From the expression of the true reverse generator in \cref{eq:reversed-generator}, we see that it depends on the intractable marginals $q_t$. To obtain tractable reverse dynamics, we introduce approximate marginals $\hat{p}_s^\thetab \approx \hat{q}_s = q_{1-s}$ and define the \emph{approximate reverse generator} by directly substituting $\hat{p}_s^\thetab$ for $q_{1-s}$ in the closed-form expression of the true reverse generator:
\begin{align}
    \hat{\Lgen}^\thetab_s \phi(\x)
    \coloneqq \frac{1}{\hat{p}^\thetab_s(\x)}\Big((\Lgen_{1-s})^\star\left[\hat{p}_s^\thetab\phi\right](\x) - \phi(\x)\,(\Lgen_{1-s})^\star \hat{p}_s^\thetab(\x) \Big) \ ,\qquad \forall \x \in \calX^d\ .
\end{align}
The generator $\hat{\Lgen}^\thetab_s$ induces an \emph{approximate reverse process} $\{\hat{\x}^\thetab_s\}_{t\in[0,1]}$ with corresponding path measure $\hat{\Ppath}^\thetab$. When the approximation is exact, i.e., $\hat{p}_s^\thetab = \hat{q}_s$ for all $s$, we recover the true reverse generator $\hat{\Lgen}_s$ and the true reverse path measure $\hat{\Qpath}$.

Our goal is to quantify the discrepancy between the approximate reverse process (with path measure $\hat{\Ppath}^\thetab$) and the true one (with path measure $\hat{\Qpath}$). We introduce the strictly positive ratio :
\begin{align}
\alpha_s \coloneqq \frac{\hat{p}^\thetab_s}{\hat q_s}.
\end{align}

A key question is whether we can assess this discrepancy at the generator level. We establish the relationship between the approximate generator $\hat{\Lgen}^\thetab_s$ and the true reverse generator $\hat{\Lgen}_s$ using $\alpha_s$.

For any test function $\phi\in\mathcal{D}(\hat \Lgen_s)$ and all $\x \in \calX^d$, we can express the approximate generator as:
\begin{align}
    \hat{\Lgen}^\thetab_s \phi(\x)
= \frac{1}{\alpha_s(\x)\hat q_{s}(\x)}\Big(\underbrace{(\Lgen_{1-s})^\star\left[\alpha_s\hat q_{s}\phi\right](\x)}_{(A)} - \underbrace{\phi(\x)\,(\Lgen_{1-s})^\star\left[ \alpha_s \hat q_{s}\right](\x) }_{(B)}\Big) \ .
\end{align}
To relate $(A)$ and $(B)$ to the true reverse generator, we apply the definition of $\hat \Lgen_s$ from \cref{eq:reversed-generator}. Under smoothness assumptions ensuring that $\alpha_s \phi$ and $\alpha_s$ belong to $\mathcal{D}(\hat\Lgen_s)$ for all $s$, we obtain:
\begin{align}
    \hat \Lgen_s[\alpha_s \phi](\x) &= \frac{1}{\hat q_s(\x)} \Bigl( (\Lgen_{1-s})^\star\left[\hat q_{s}\alpha_s\phi\right](\x)- \alpha_s(\x) \phi(\x) (\Lgen_{1-s})^\star \hat q_s(\x) \Bigr)\ ,\\
    \phi(\x)\hat \Lgen_s\alpha_s(\x) &= \frac{1}{\hat q_s(\x)} \Bigl( \phi(\x)(\Lgen_{1-s})^\star\left[\hat q_{s}\alpha_s\right](\x)-  \phi(\x) \alpha_s(\x) (\Lgen_{1-s})^\star \hat q_s(\x) \Bigr)\ .
\end{align}
Rearranging these equations to express $(A)$ and $(B)$:
\begin{align}
    (A) &= \hat q_s(\x) \hat \Lgen_s\left[\alpha_s \phi\right](\x) + \alpha_s(\x) \phi(\x) (\Lgen_{1-s})^\star \hat q_s(\x)\ , \\
    (B) &= \hat q_s(\x)\phi(\x)\hat \Lgen_s\alpha_s(\x) + \phi(\x) \alpha_s(\x) (\Lgen_{1-s})^\star\hat q_s(\x)\ .
\end{align}
Subtracting $(B)$ from $(A)$, yields:
\begin{align}
    \hat{\Lgen}^\thetab_s \phi(\x) = \frac{1}{\alpha_s(\x)} \Bigl( \hat \Lgen_s\left[\alpha_s \phi\right](\x) - \phi(\x)\hat \Lgen_s\alpha_s(\x) \Bigr).
\label{eq:parametrized-generator-relation}
\end{align}

This key identity expresses the approximate generator $\hat{\Lgen}^\thetab_s$ in terms of the true reverse generator $\hat{\Lgen}_s$ via the density ratio $\alpha_t$. We now leverage this connection via the Girsanov theorem to derive the ELBO. The following result is adapted from \citet[Theorem 4.2]{palmowski2002technique}.

\general[General Girsanov formula via the generator]{
Let $\{\x_s\}_{s\in[0,1]}$ be a Markov process with infinitesimal spatial generator $\Lgen_s$ and marginal densities $q_s$ under measure $\Qpath$, and denote by $\x_{[0,1]} = (\x_s)_{s\in[0,1]}$ a sample path up to time $1$. Suppose $\alpha_s>0$ is a strictly positive function satisfying regularity conditions ensuring that the exponential martingale
\begin{equation}
M_s^\alpha \coloneqq \frac{\alpha_s(\x_s)}{\alpha_0(\x_0)} \exp\left\{-\int_0^s \frac{(\partial_\tau + \Lgen_\tau) \alpha_\tau(\x_\tau)}{\alpha_\tau(\x_\tau)} \diff \tau\right\},
\end{equation}
is a true martingale.\footnote{Sufficient conditions (Proposition 3.2 of 
\citet{palmowski2002technique}) include: (i) both $\alpha_s$ and 
$\frac{(\Lgen_s\alpha_s)}{\alpha_s}$ are bounded; or (ii) both $\alpha_s$ and 
$\Lgen_s\alpha_s$ are bounded with $\inf_\x \alpha_s(\x) > 0$.} Then there exists a probability measure $\Ppath$ absolutely continuous with respect to $\Qpath$ with Radon-Nikodym derivative:
\begin{equation}
\frac{\diff \Ppath}{\diff \Qpath}(\x_{[0,1]}) = M_1^\alpha = \frac{\alpha_1(\x_1)}{\alpha_0(\x_0)} \exp\left\{-\int_0^1 \frac{(\partial_s + \Lgen_s) \alpha_s(\x_s)}{\alpha_s(\x_s)} \diff s\right\}.
\label{eq:general-girsanov}
\end{equation}
Under $\Ppath$, the process $\{\x_s\}_{s\in[0,1]}$ is Markovian with the transformed spatial generator:
\begin{equation}
\hat{\Lgen}_s \phi(\x) = \frac{1}{\alpha_s(\x)}\Big[\Lgen_s[\alpha_s \phi](\x) - \phi(\x)\,\Lgen_s \alpha_s(\x)\Big],
\label{eq:transformed-generator}
\end{equation}
for all $\phi \in \mathcal{D}(\Lgen_s)$, where $\mathcal{D}(\hat{\Lgen}_s) = \mathcal{D}(\Lgen_{s})$.
}

We now connect this result to the ELBO computation. As discussed in \cref{sect:ELBO}, the goal in diffusion generative modeling is to maximise the likelihood of observed data under the learned model. We will show that the negative log-likelihood can be upper-bounded by a KL divergence between path measures.

\paragraph{From marginal to path-space KL divergence.} 
We begin by expressing the negative log-likelihood in terms of the KL divergence of the data distribution $\qdata$ from the model-induced sample distribution at time 0, denoted $p_0^\thetab$ (endpoint marginal induced by $\hat\Ppath^\thetab$):
\begin{align}
    - \E_{\qdata}\left[\log p_0^\thetab(\x_0)\right] 
    = \KL(\qdata\|p_0^\thetab) + \underbrace{\E_{\qdata}\left[-\log\qdata(\x_0)\right],}_{\mathcal{H}(\qdata)}
\end{align}
where $\mathcal{H}(\qdata)$ is the entropy of the data distribution and is constant w.r.t. $\thetab$.
By the data processing inequality (detailed in \cref{apx:data_processing_inequality}), the marginal KL is bounded by the path-space KL:
\begin{align}
    \KL(\qdata\|p_0^\thetab) \leq \KL(\hat \Qpath \| \hat \Ppath^\thetab) = \E_{\hat \Qpath}\left[\log \frac{\diff \hat \Qpath}{\diff \hat \Ppath^\thetab}\right] \ ,
\end{align}
where $\hat\Qpath$ and $\hat\Ppath^\thetab$ denote the path measures of the true and approximated reverse processes, respectively.
This yields the following upper bound on the negative log-likelihood:
\begin{align}
     - \E_{\qdata}\left[\log p_0^\thetab(\x_0)\right] 
     &\leq \E_{\hat \Qpath}\left[\log \frac{\diff \hat \Qpath}{\diff \hat \Ppath^\thetab}\right] + \mathcal{H}(\qdata)\ .
\end{align}

\paragraph{Applying the Girsanov formula.}
Recall that $\alpha_s= \frac{\hat p_s^\thetab}{\hat q_s}$ relates the approximated and true reverse marginals. Assuming sufficient regularity conditions such that the Girsanov formula \eqref{eq:general-girsanov} applies with $\Qpath = \hat\Qpath$, $\Ppath = \hat\Ppath^\thetab$, and $\Lgen_s = \hat\Lgen_s$, we obtain:
\begin{align}
\frac{\diff \hat \Ppath^\thetab}{\diff \hat \Qpath}(\hat \x_{[0,1]}) 
= \frac{\alpha_1(\hat \x_1)}{\alpha_0(\hat \x_0)} \exp\left\{-\int_0^1 \frac{(\partial_s + \hat\Lgen_s) \alpha_s(\hat \x_s)}{\alpha_s(\hat \x_s)} \diff s\right\}.
\end{align}

Taking the negative logarithm and expectation under $\hat\Qpath$ gives the negative ELBO.

\general[ELBO via the generator]{
The negative log-likelihood is upper-bounded by the following negative ELBO expressed in terms of the infinitesimal generator:
\begin{align}
    - \E_{\qdata}\left[\log p_0^\thetab(\x_0)\right] 
    &\leq \E_{\hat \Qpath}\left[\int_0^1 \frac{(\partial_s+\hat\Lgen_s) \alpha_s(\hat \x_s)}{\alpha_s(\hat \x_s)} \diff s 
    - \log\left(\frac{\alpha_1(\hat \x_1)}{\alpha_0(\hat \x_0)}\right) \right] + \mathcal{H}(\qdata)\ .
    \label{eq:ELBO-generator}
\end{align}

Applying Dynkin's formula \cref{eq:general-Dynkin} to absorb the time derivative into the generator expectation, this can be rewritten as:
\begin{align}
    - \E_{\qdata}\left[\log p_0^\thetab(\x_0)\right] 
    &\leq \E_{\hat \Qpath}\left[\int_0^1 \left(\frac{\hat\Lgen_s \alpha_s(\hat \x_s)}{\alpha_s(\hat \x_s)}  - \hat \Lgen_s\log\alpha_s(\hat \x_s)\right)\diff s \right] + \mathcal{H}(\qdata)\ .
    \label{eq:ELBO-generator-Dynkin}
\end{align}
}

\begin{proof}[Derivation of equivalence \eqref{eq:ELBO-generator-Dynkin}]
Using Dynkin's formula \eqref{eq:general-Dynkin} with time-dependen test function $\phi_s(\x) \coloneqq \log\alpha_s(\x)$ and taking expectations on both sides, we obtain:
\begin{align}
    \E_{\hat \Qpath} \left[\log\alpha_1(\hat\x_1) - \log\alpha_0(\hat\x_0) \right]
    = \E_{\hat \Qpath} \left[\int_0^{1}(\partial_s+\hat \Lgen_s)\log\alpha_s(\hat\x_s)\,\diff s \right]\, .
\end{align}
Substituting this into \eqref{eq:ELBO-generator}, the time-derivative terms cancel since $\partial_s \log \alpha_s = \frac{\partial_s \alpha_s}{\alpha_s}$. The remaining terms form the generator difference in \eqref{eq:ELBO-generator-Dynkin}:
\begin{equation}
    \frac{(\partial_s + \hat{\Lgen}_s)\alpha_s}{\alpha_s} - (\partial_s + \hat{\Lgen}_s)\log\alpha_s 
    = \frac{\hat{\Lgen}_s\alpha_s}{\alpha_s} - \hat{\Lgen}_s\log\alpha_s.
\end{equation}
This quantity is non-negative (a consequence of Jensen's inequality) and represents the local discrepancy between the true and approximated reverse processes.
\end{proof}
Having derived the general reverse generator applicable to any Markov process, we now specialise to the canonical cases of continuous and discrete state spaces, recovering the familiar forward/reverse SDE and CTMC formulas.

\subsection{Connection to standard diffusion processes}

\subsubsection{Forward and reverse process}

The abstract generator framework developed in \cref{sec:reverse-generator} applies to arbitrary Markov processes. We now demonstrate how it \emph{specialises} to the two fundamental cases: stochastic differential equations (SDEs) in continuous state spaces, and continuous-time Markov chains (CTMCs) in discrete state spaces. This connection reveals that the familiar continuous and discrete diffusion formulas are not independent derivations, but rather concrete instantiations of the unified reverse generator from \cref{eq:reversed-generator}. While we focus on these canonical examples, the framework extends naturally to more general processes such as jump-diffusions and piecewise-deterministic Markov processes (see \citet{palmowski2002technique,benton_denoising_2024,ren2025unified} for applications to these settings). The detailed derivations are provided in \cref{apx:derivation-infinitesimal-generator}.

\continuous[The forward \& reverse generator]{
Let $\{\x_t\}_{t\in[0,1]}$ be a diffusion process taking values in $\calX^d \subseteq \RR^d$ and governed by the SDE $\dx_t = \f_t(\x_t) \,\dt + g_t \, \dw_t$ with state-independent diffusion coefficient $g_t$ (see \cref{eq:sde-continuous-diff}). Then the adjoint generator acting on a test function $\phi$ is given by:
\begin{align}
    (\Lgen_t)^*\phi(\x) 
    = - \nabla_\x  \cdot  \bigl( \f_t(\x)\,\phi(\x) \bigr) + \Delta_\x \left( \tfrac{1}{2} g_t^2\, \phi(\x) \right) \;, \qquad \x \in \calX^d \; .
\end{align}
Therefore, the Kolmogorov forward (Fokker–Planck) equation for the marginals is
\begin{align}
    \partial_t q_t(\x) 
    = - \underbrace{\nabla_\x \!\cdot \bigl(\f_t(\x)\, q_t(\x)\bigr)}_{\text{drift term}} + \underbrace{\tfrac{1}{2}\,g_t^2\,\Delta_\x q_t(\x) }_{\text{diffusion term}} \ .
    \label{eq:generator-KFE-continuous}
\end{align}
which corresponds to \cref{eq:KFE-continuous-sp}. Applying the reverse generator formula \eqref{eq:reversed-generator} yields:
\begin{align}
    \hat{\Lgen}_s \phi(\x)
    = \tfrac{1}{2}\,g_{1-s}^2\,\Delta_\x \phi(\x)
   + \Big[-\f_{1-s}(\x) + g_{1-s}^2\,\nabla_\x \log q_{1-s}(\x)\Big] \cdot \nabla_\x \phi(\x) \ .
    \label{eq:reversed-generator-closed-form}
\end{align}
This expression exactly recovers the reverse Fokker--Planck equation \cref{eq:reverse-Fokker-Planck}.
}

\begin{proof}[Derivation sketch:]
Starting from \cref{eq:reversed-generator} with $g_{1-s}$ independent of $\x$:
\begin{align}
\hat{\Lgen}_s \phi(\x) 
&= \underbrace{q^{-1}_{1-s}(\x)\,(\Lgen_{1-s})^\star\!\big(q_{1-s}\phi\big)(\x)}_{\mathcal I_1}
   \;\underbrace{-\, q^{-1}_{1-s}(\x)\,\phi(\x)\,(\Lgen_{1-s})^\star q_{1-s}(\x)}_{\mathcal I_2}.
\end{align}
We study both terms separately:
\begin{align}
\mathcal I_2
&= -\, q^{-1}_{1-s}(\x)\,\phi(\x)\,
    \Big[\Delta_\x\!\Big(\tfrac{1}{2} g_{1-s}^2\, q_{1-s}(\x)\Big)
          - \nabla_\x  \cdot  \bigl( \f_{1-s}(\x)\,q_{1-s}(\x) \bigr)\Big] \nonumber\\
&=  - \tfrac{1}{2} g_{1-s}^2 \frac{\Delta_\x q_{1-s}}{q_{1-s}}\,\phi
  + (\nabla_\x  \cdot  \f_{1-s})\,\phi
  + \f_{1-s} \cdot \frac{\nabla_\x q_{1-s}}{q_{1-s}}\,\phi \ ,
\\[4pt]
\mathcal I_1
&= \; q_{1-s}^{-1}(\x)\,
     \Big[\Delta_\x \!\Big( \tfrac{1}{2} g_{1-s}^2\, q_{1-s}(\x)\phi(\x) \Big)
           - \nabla_\x  \cdot  \bigl( \f_{1-s}(\x)\,q_{1-s}(\x)\phi(\x) \bigr)\Big] \nonumber \\
&= \tfrac{1}{2} g_{1-s}^2 \Delta_\x\phi
   + g_{1-s}^2 \frac{\nabla_\x q_{1-s}}{q_{1-s}}  \cdot  \nabla_\x\phi
   + \tfrac{1}{2} g_{1-s}^2 \frac{\Delta_\x q_{1-s}}{q_{1-s}}\,\phi \nonumber\\
   &\quad - \f_{1-s} \cdot \nabla_\x\phi
   - (\nabla_\x  \cdot  \f_{1-s})\,\phi
   - \f_{1-s} \cdot \frac{\nabla_\x q_{1-s}}{q_{1-s}}\,\phi \ ,
\end{align}
using the product rules
$\Delta_\x(uv)=u\,\Delta_\x v+2\nabla_\x u\cdot\nabla_\x v+v\,\Delta_\x u$ and
$\nabla_\x\!\cdot(u\mathbf{v})=\nabla_\x u\cdot\mathbf{v}+u\,\nabla_\x\!\cdot\mathbf{v}$\,.
All zeroth-order terms (in $\phi$) cancel in $\mathcal I_1+\mathcal I_2$, leaving
\begin{align*}
\hat{\Lgen}_s \phi
= \tfrac{1}{2}\,g_{1-s}^2\,\Delta_\x \phi
  + \big(-\f_{1-s} + g_{1-s}^2\,\nabla_\x \log q_{1-s}\big) \cdot \nabla_\x \phi \ .
\end{align*}
\end{proof}

\discrete[The forward \& reverse generator]{
Let $\{\x_t\}_{t \in [0,1]}$ be a continuous-time Markov chain (CTMC) with state space $\calX^d$ and time-dependent rate matrix $\Rseq_t$, defined under the convention that $\Rseq_t(\x, \y)$ denotes the transition rate from $\y$ to $\x$. The adjoint generator acting on a test function $\phi$ is given by:\footnote{In some literature, the adjoint generator equals the transpose of the rate matrix. The difference stems from our indexing convention introduced in \cref{eq:def-transition-matrix}.}
\begin{align}
    (\Lgen_t)^* \phi(\x) 
    = \sum_{\y \in \calX^d} \Rseq_t(\x,\y)\,\phi(\y).
\end{align}
The Kolmogorov forward (master) equation for the marginals is
\begin{align}
    \partial_t q_t(\x) 
    = \sum_{\y \in \calX^d} \Rseq_t(\x,\y)\, q_t(\y)\ ,
\end{align}
which recovers \cref{eq:KFE-discrete-sp}. Applying the reverse generator formula \eqref{eq:reversed-generator} yields:
\begin{align}
    \hat{\Lgen}_s\phi(\x) = \sum_{\y}\Rseq_{1-s}(\x,\y) \frac{q_{1-s}(\y)}{q_{1-s}(\x)} \big[\phi(\y)-\phi(\x)\big]\ .
    \label{eq:reversed-generator-discrete-sp}
\end{align}
}

\begin{proof}[Derivation sketch:]
Working with counting measure so that $(\Lgen_{1-s}^\star h)(\x)=\sum_{\y\in\calX^d}\Rseq_{1-s}(\x,\y)\,h(\y)$:
\begin{align}
\hat{\Lgen}_s\phi(\x) 
&= \frac{1}{q_{1-s}(\x)}\Bigg[\sum_{\y}\Rseq_{1-s}(\x,\y) q_{1-s}(\y) \phi(\y) - \phi(\x)\sum_{\y}\Rseq_{1-s}(\x,\y) q_{1-s}(\y)\Bigg] \nonumber\\
&= \sum_{\y}\frac{q_{1-s}(\y)}{q_{1-s}(\x)} \Rseq_{1-s}(\x,\y) \big[\phi(\y)-\phi(\x)\big].
\end{align}
Hence the off-diagonal rates of the reversed generator are
$\hat{\Rseq}_s(\y,\x)=\dfrac{q_{1-s}(\y)}{q_{1-s}(\x)}\,\Rseq_{1-s}(\x,\y)$ for $\y\neq\x$,
and the diagonal is determined by the column-sum-to-zero property:
$\hat{\Rseq}_s(\x,\x)=-\sum_{\y\neq \x}\hat{\Rseq}_s(\y,\x)$.
\end{proof}

\begin{remark}[Generator factorisation under independent noising] Assume the forward noising is \emph{independent across coordinates}. 
Then, for $\x_t = (x_t^{(1)},\dots,x^{(d)}_t) \in \calX^d$, 
the sequence-level spatial generator of the Markov process \emph{decomposes additively}:
\begin{equation}
\Lgen_t =  \sum_{k=1}^d \Lgen_t^{(k)}, 
\label{eq:generator-dimensional-sum}
\end{equation}
where each $\Lgen_t^{(k)}$ acts only on the $k$-th coordinate $x_t^{(k)}$ while leaving other coordinates unchanged.
That is, for any function $\phi$,
\begin{equation}
(\Lgen_t^{(k)} \phi)(\x)  =  (\Lgen_t^{(k)} \phi)(x^{(1)}_t,\dots,x^{(d)}_t)
\end{equation}
depends on the action of the generator only through coordinate $x^{(k)}_t$.
Consequently, the \emph{marginal density of each coordinate}, 
$q_t(x^{(k)}) \coloneqq \int_{\calX^{d-1}} q_t(\x) \, \diff\mu(\x^{\setminus k})$,
evolves independently according to a \emph{single-dimension} KFE:
\begin{equation}
\partial_t q_t(x^{(k)})  =  (\Lgen_t^{(k)})^* q_t(x^{(k)}),
\label{eq:generator-factorisation-dimension}
\end{equation}
where $\x^{\setminus k}$ denotes all coordinates except the $k$-th (proof in \cref{apx:generator-decomposition-dimension}). This recovers the single-dimension KFE presented in \cref{eq:kfe-1d-discrete}.
\end{remark}

\subsubsection{ELBO}

Having established the general ELBO via the generator framework in \cref{eq:ELBO-generator-Dynkin}, we now specialise this result to continuous and discrete state spaces, recovering the familiar score-matching and rate-matrix objectives as concrete instances of the unified formula.
We start from the ELBO expression in \cref{eq:ELBO-generator-Dynkin}, reproduced here for convenience:
\begin{align}
    - \E_{\qdata}\left[\log p_0^\thetab(\x_0)\right] 
    &\leq \E_{\hat \Qpath}\Bigg[\int_0^1 \bigg(\underbrace{\frac{\hat\Lgen_s \alpha_s(\hat \x_s)}{\alpha_s(\hat \x_s)}}_{(A)}  - \underbrace{\hat \Lgen_s\log\alpha_s(\hat \x_s)}_{(B)}\bigg)\diff s \Bigg] + \mathcal{H}(\qdata) \ .
    \label{eq:csp-elbo-generator-terms}
\end{align}
Recall that we denote the ratio between the approximated and true reverse marginals by $\alpha_s \coloneqq \frac{\hat p^\thetab_s}{\hat q_s}$.
To evaluate the ELBO, we must explicitly compute the two generator terms $(A)$ and $(B)$, both of which involve the reverse-time generator $\hat{\Lgen}_s$ given in \cref{eq:reversed-generator-closed-form}.
In what follows, we compute $(A)$ and $(B)$ under the continuous-time diffusion dynamics. An analogous derivation for discrete-state Markov chains will be presented afterward.

\paragraph{Continuous state space.}
\emph{Computing term (A):}
Using the closed-form reverse generator \eqref{eq:reversed-generator-closed-form}, we obtain:
\begin{align}
    (A) 
    &= \frac{1}{\alpha_s(\hat\x_s)}\left[ \tfrac{1}{2}\,g_{1-s}^2\,\Delta_\x \alpha_s(\hat\x_s)
   + \Big[-\,\f_{1-s}(\hat\x_s) + g_{1-s}^2\,\nabla_\x \log \hat q_s(\hat\x_s)\Big] \cdot \nabla_\x \alpha_s(\hat\x_s) \right] \notag\\
   &= \tfrac{1}{2}\,g_{1-s}^2\,\frac{\Delta_\x \alpha_s(\hat\x_s)}{\alpha_s(\hat\x_s)}
   + \Big[-\,\f_{1-s}(\hat\x_s) + g_{1-s}^2\,\nabla_\x \log \hat q_s(\hat\x_s)\Big] \cdot \nabla_\x \log\alpha_s(\hat\x_s) \ .
    \label{eq:term-A-expansion}
\end{align}
To simplify the Laplacian quotient, we introduce $\gamma_s(\hat\x_s) \coloneqq \log \alpha_s(\hat\x_s)$, so that $\alpha_s(\hat\x_s) = e^{\gamma_s(\hat\x_s)}$. By the chain rule:
\begin{align}
    \Delta_\x \alpha_s
    &= \nabla_\x \cdot\left(e^\gamma_s \nabla_\x \gamma_s\right)
    = e^\gamma_s\left(\Delta_\x \gamma_s + \|\nabla_\x \gamma_s\|^2\right)
    = \alpha_s\left(\Delta_\x \log \alpha_s + \| \nabla_\x \log \alpha_s \|^2\right)\ .
\end{align}

\emph{Computing term (B):}
Applying the reverse generator to $\log\alpha_s(\hat \x_s)$ yields:
\begin{align}
    (B) 
    &= \hat \Lgen_s\log\alpha_s(\hat\x_s) \notag\\
    &= \tfrac{1}{2}\,g_{1-s}^2\,\Delta_\x \log\alpha_s(\hat\x_s)
   + \Big[-\,\f_{1-s}(\hat\x_s) + g_{1-s}^2\,\nabla_\x \log \hat q_s(\hat\x_s)\Big] \cdot \nabla_\x \log\alpha_s(\hat\x_s) \ .
    \label{eq:term-B-expansion}
\end{align}

\emph{Combining terms:}
Subtracting \eqref{eq:term-B-expansion} from \eqref{eq:term-A-expansion}, we find that all first-order gradient terms cancel, leaving only:
\begin{align}
(A) - (B)
&= \tfrac{1}{2} g_{1-s}^2 \| \nabla_\x \log \alpha_s(\hat\x_s)\|^2 \notag\\
&= \tfrac{1}{2} g_{1-s}^2 \|\nabla_\x \log \hat p^\thetab_s(\hat\x_s) - \nabla_\x\log \hat q_s(\hat\x_s)\|^2.
\label{eq:score-diff-squared}
\end{align}

Substituting \eqref{eq:score-diff-squared} into \eqref{eq:csp-elbo-generator-terms} and applying the Fubini--Tonelli theorem to exchange the order of integration:
\begin{align}
- \E_{\qdata}\left[\log p_0^\thetab(\x_0)\right] 
&\leq \int_0^1 \E_{\hat \Qpath}\left[\tfrac{1}{2} g_{1-s}^2 \|\nabla_\x \log \hat p^\thetab_s(\hat \x_s) - \nabla_\x\log \hat q_s(\hat \x_s)\|^2 \right] \diff s + \mathcal{H}(\qdata).
\end{align}
Since the expectation under the path measure $\hat{\Qpath}$ of a function depending only on $\hat \x_s$ reduces to the marginal expectation under $\hat{q}_s$, and recalling that $\hat{q}_s = q_{1-s}$, we obtain:
\begin{align}
- \E_{\qdata}\left[\log p_0^\thetab(\x_0)\right] 
&\leq \int_0^1 \E_{\hat q_s}\left[\tfrac{1}{2} g_{1-s}^2 \|\nabla_\x \log \hat p^\thetab_{s}(\hat \x_s) - \nabla_\x\log \hat q_s(\hat\x_s)\|^2 \right] \diff s + \mathcal{H}(\qdata).
\end{align}

Changing variables from reverse time $s$ to forward time $t = 1-s$, dropping the constant entropy term, and applying the score identity (\cref{eq:Tweedie-identity-ct}), we recover the standard denoising score matching loss. %
Here we denote by $\mathbf s^\thetab(\x,t)$ the learned approximation to the score of the true marginals $q_t$.

\continuous[Generator-based ELBO for continuous spaces]{
The general generator ELBO \eqref{eq:ELBO-generator-Dynkin} specialises to the denoising score matching objective:
\begin{align}
    \mathcal{L}^\thetab_{\text{DSM}}
    &= \E_{\x_0 \sim \qdata}\int_0^1 \E_{\x_t \sim q(\cdot|\x_0)}\left[\tfrac{1}{2} g_t^2 \|\mathbf{s}^\thetab(\x_t,t) - \nabla_{\x_t}\log q(\x_t|\x_0)\|^2 \right] \dt\ ,
\end{align}
which recovers the standard training loss in \cref{eq:elbo-score-matching}.
}

\paragraph{Discrete state spaces.} 
\emph{Computing term (A):}
Applying \cref{eq:reversed-generator-discrete-sp} to $\alpha_s(\hat \x_s)$:
\begin{align}
(A) 
&= \frac{\hat\Lgen_s \alpha_s(\hat\x_s)}{\alpha_s(\hat\x_s)} \notag\\
&= \frac{1}{\alpha_s(\hat\x_s)} \sum_{\y \in \calX^d} \frac{\hat q_s(\y)}{\hat q_s(\hat\x_s)} \Rseq_{1-s}(\hat\x_s,\y) \bigl[\alpha_s(\y) - \alpha_s(\hat\x_s)\bigr] \notag\\
&= \sum_{\y \neq \hat\x_s} \frac{\hat q_s(\y)}{\hat q_s(\hat\x_s)} \Rseq_{1-s}(\hat\x_s,\y) \left[\frac{\alpha_s(\y)}{\alpha_s(\hat\x_s)} - 1\right] \notag\\
&= \sum_{\y \neq \hat \x_s} \frac{\hat q_s(\y)}{\hat q_s(\hat \x_s)} \Rseq_{1-s}(\hat \x_s,\y) \left[\frac{\hat p^\thetab_s(\y) \hat q_s(\hat \x_s)}{\hat p^\thetab_s( \hat \x_s) \hat q_s(\y)} - 1\right],
\label{eq:discrete-term-A}
\end{align}
where we used $\Rseq_{1-s}(\hat\x_s,\hat\x_s) = -\sum_{\y\neq\hat\x_s}\Rseq_{1-s}(\y,\hat \x_s)$ (column-sum-to-zero).

\emph{Computing term (B):}
For the logarithm term:
\begin{align}
(B) 
&= \hat\Lgen_s \log\alpha_s(\hat\x_s) \notag\\
&= \sum_{\y \neq\hat  \x_s} \frac{\hat q_s(\y)}{\hat q_s(\hat\x_s)} \Rseq_{1-s}(\hat\x_s,\y) \bigl[\log\alpha_s(\y) - \log\alpha_s(\hat\x_s)\bigr] \notag\\
&= \sum_{\y \neq \hat\x_s} \frac{\hat q_s(\y)}{\hat q_s(\hat\x_s)} \Rseq_{1-s}(\hat\x_s,\y) \log\!\left(\frac{\hat p^\thetab_s(\y) \hat q_s(\hat\x_s)}{\hat p^\thetab_s(\hat\x_s) \hat q_s(\y)}\right).
\label{eq:discrete-term-B}
\end{align}

\emph{Combining terms:}
Subtracting \cref{eq:discrete-term-B} from \cref{eq:discrete-term-A}, the $\hat q_s$ ratios cancel, yielding:
\begin{align}
(A) - (B)
&= \sum_{\y \neq \hat \x_s} \frac{\hat q_s(\y)}{\hat q_s(\hat \x_s)} \Rseq_{1-s}(\hat \x_s,\y) \left[\frac{\hat p^\thetab_s(\y) \hat q_s(\hat\x_s)}{\hat p^\thetab_s(\hat \x_s) \hat q_s(\y)} - 1 - \log\!\left(\frac{\hat p^\thetab_s(\y) \hat q_s(\hat \x_s)}{\hat p^\thetab_s(\hat \x_s) \hat q_s(\y)}\right)\right] \notag\\
&= \sum_{\y \neq \hat \x_s} \frac{\hat q_s(\y)}{\hat q_s(\hat \x_s)} \Rseq_{1-s}(\hat \x_s,\y) \bigl[r_s(\y,\hat \x_s) - 1 - \log r_s(\y,\hat \x_s)\bigr],
\label{eq:discrete-combined}
\end{align}
where $r_s(\y,\hat \x_s) \coloneqq \frac{\hat p^\thetab_s(\y)/\hat p^\thetab_s(\hat \x_s)}{\hat q_s(\y)/\hat q_s(\hat \x_s)}$ is the ratio of learned to true density ratios. The expression $r - 1 - \log r$ is non-negative for $r > 0$, vanishing when $r = 1$, and measures the discrepancy between the learned and true distributions.

Substituting into \eqref{eq:ELBO-generator-Dynkin} and changing time variable from reverse time $s$ to forward time $t = 1-s$ (so $\hat q_s = q_{t}$ and $\Rseq_{1-s} = \Rseq_t$), we obtain:
\begin{align}
    - \E_{\qdata}\left[\log p_0^\thetab(\x_0)\right] 
    &\leq \int_0^1 \E_{\x_t \sim q_t}\left[\sum_{\y \neq \x_t} \Rseq_t(\x_t,\y) \frac{q_t(\y)}{q_t(\x_t)} \Bigl(r_t(\y,\x_t) - 1 - \log r_t(\y,\x_t)\Bigr)\right] \diff t + \mathcal{H}(\qdata),
    \label{eq:discrete-elbo-ratio-form}
\end{align}
where now $r_t(\y,\x_t) = \frac{p^\thetab_t(\y)/p^\thetab_t(\x_t)}{q_t(\y)/q_t(\x_t)}$ is the learned-to-true ratio in forward time.

To make explicit the connection to the ELBO derived in \cref{sect:ELBO}, we expand the expression $r_t - 1 - \log r_t$. Recall that the true reverse rate is $\hat{\Rseq}_t(\y,\x_t) = \Rseq_t(\x_t,\y) \frac{q_t(\y)}{q_t(\x_t)}$ (from \cref{eq:reverse-rate-matrix}), and the learned reverse rate approximates this with $\hat{\Rseq}^\thetab_t(\y,\x_t) = \Rseq_t(\x_t,\y) \frac{p^\thetab_t(\y)}{p^\thetab_t(\x_t)}$. Thus:
\begin{align}
\hat{\Rseq}^\thetab_t(\y,\x_t) 
&= \hat{\Rseq}_t(\y,\x_t) \cdot \frac{p^\thetab_t(\y)/p^\thetab_t(\x_t)}{q_t(\y)/q_t(\x_t)}
= \hat{\Rseq}_t(\y,\x_t) \cdot r_t(\y,\x_t).
\label{eq:learned-reverse-rate-ratio}
\end{align}
Since $\hat{\Rseq}_t(\y,\x_t) = \Rseq_t(\x_t,\y) \frac{q_t(\y)}{q_t(\x_t)}$, substituting into \cref{eq:discrete-elbo-ratio-form} and expanding:
\begin{align}
\Rseq_t(\x_t,\y) \frac{q_t(\y)}{q_t(\x_t)} \bigl(r_t - 1 - \log r_t\bigr) &= \hat{\Rseq}_t(\y,\x_t) \bigl(r_t - 1 - \log r_t\bigr) \notag\\
&= \hat{\Rseq}^\thetab_t(\y,\x_t) - \hat{\Rseq}_t(\y,\x_t) - \hat{\Rseq}_t(\y,\x_t) \log r_t \\
&= \hat{\Rseq}^\thetab_t(\y,\x_t) - \hat{\Rseq}_t(\y,\x_t) - \Rseq_t(\x_t,\y) \frac{q_t(\y)}{q_t(\x_t)} \log\!\left(\frac{\hat{\Rseq}^\thetab_t(\y,\x_t)}{\hat{\Rseq}_t(\y,\x_t)}\right). \notag\\
\label{eq:discrete-elbo-expanded-step}
\end{align}

As in the continuous case, we apply an identity analogous to the \emph{score identity} \cref{eq:Tweedie-identity-ct}, to replace intractable marginal ratios with tractable conditional ratios. For the forward process with $q_t(\x_t) = \int q(\x_t|\x_0)\,\qdata(\x_0)\,\diff\x_0$, the marginal ratio satisfies:\footnote{By Bayes' rule,
\(
\E_{q(\cdot\cond \x_t)}
\left[\frac{q(\y\cond \x_0)}{q(\x_t\cond \x_0)}\right]=
\int \frac{q(\x_t\cond \x_0)\,q_{\text{data}}(\x_0)}{q_t(\x_t)}
\cdot \frac{q(\y\cond \x_0)}{q(\x_t\cond \x_0)}\diff \x_0 =
\frac{1}{q_t(\x_t)}
\int q(\y\cond \x_0)q_{\text{data}}(\x_0)\diff \x_0 = \frac{q_t(\y)}{q_t(\x_t)}.
\)}
\begin{align}
    \frac{q_t(\y)}{q_t(\x_t)} 
    = \E_{\x_0 \sim q(\cdot|\x_t)}\left[\frac{q(\y|\x_0)}{q(\x_t|\x_0)}\right]\ .
    \label{eq:discrete-tweedie}
\end{align}
Rewriting the expectation over $q_t$ as a joint expectation over $\qdata$ and $q(\cdot|\x_0)$, and dropping constants independent of $\thetab$, we recover the denoising score entropy (DSE) loss, as shown below.

\discrete[Generator-based ELBO for discrete spaces]{
The general generator ELBO \eqref{eq:ELBO-generator-Dynkin} specialises to the denoising score entropy objective:
\begin{align}
    \mathcal{L}_{\text{DSE}}(\thetab)
    &= \E_{\x_0 \sim \qdata}\int_0^1 \E_{\x_t \sim q(\cdot|\x_0)}\Biggl[\sum_{\y \neq \x_t} \Bigl(\hat{\Rseq}^\thetab_t(\y,\x_t) \notag\\
    &\qquad\qquad- \Rseq_t(\x_t,\y)\frac{q(\y|\x_0)}{q(\x_t|\x_0)} \log \hat{\Rseq}^\thetab_t(\y,\x_t)\Bigr)\Biggr] \dt\ ,
    \label{eq:elbo-discrete-dse}
\end{align}
where $\hat{\Rseq}^\thetab_t(\y,\x_t)$ is the learned reverse rate from $\x_t$ to $\y$, and $\Rseq_t(\x_t,\y)$ is the forward rate from $\y$ to $\x_t$. This matches the DSE loss derived via the discrete-time ELBO in \cref{eq:seq-elbo}.
}

\paragraph{Concluding remarks.}
The infinitesimal-generator framework provides a unified lens for understanding diffusion-based generative modeling. We have shown that both SDEs (continuous state spaces) and CTMCs (discrete state spaces) arise as special cases of general Markov processes characterised by their generators $\Lgen_t$. The reverse generator formula \eqref{eq:reversed-generator} yields the time-reversed dynamics in both settings, and the Girsanov-based ELBO \eqref{eq:ELBO-generator-Dynkin} specialises to the denoising score matching objective for continuous spaces and the denoising score entropy objective for discrete spaces, recovering the same training objectives derived in \cref{sect:ELBO} via discrete-time variational arguments.

This unified perspective confirms that, despite their apparent differences, continuous and discrete diffusion models share the same underlying mathematical structure: a forward process that progressively corrupts data, a reverse process that requires approximating intractable marginal quantities (scores or ratios), and an ELBO-based training objective that enables learning from samples.

With these theoretical foundations, we turn to a practical consideration. Diffusion does not have to operate directly in the data space. In many applications, it is advantageous to first encode data into a learned latent representation and perform diffusion there. This approach, known as \emph{latent diffusion}, offers computational benefits and enables the application of continuous diffusion techniques to discrete data. We explore this paradigm in the following section.

\section{Latent diffusion \label{sect:latent-diffusion}}

\begin{figure}[h!]
    \centering
    \includegraphics[width=0.75\linewidth]{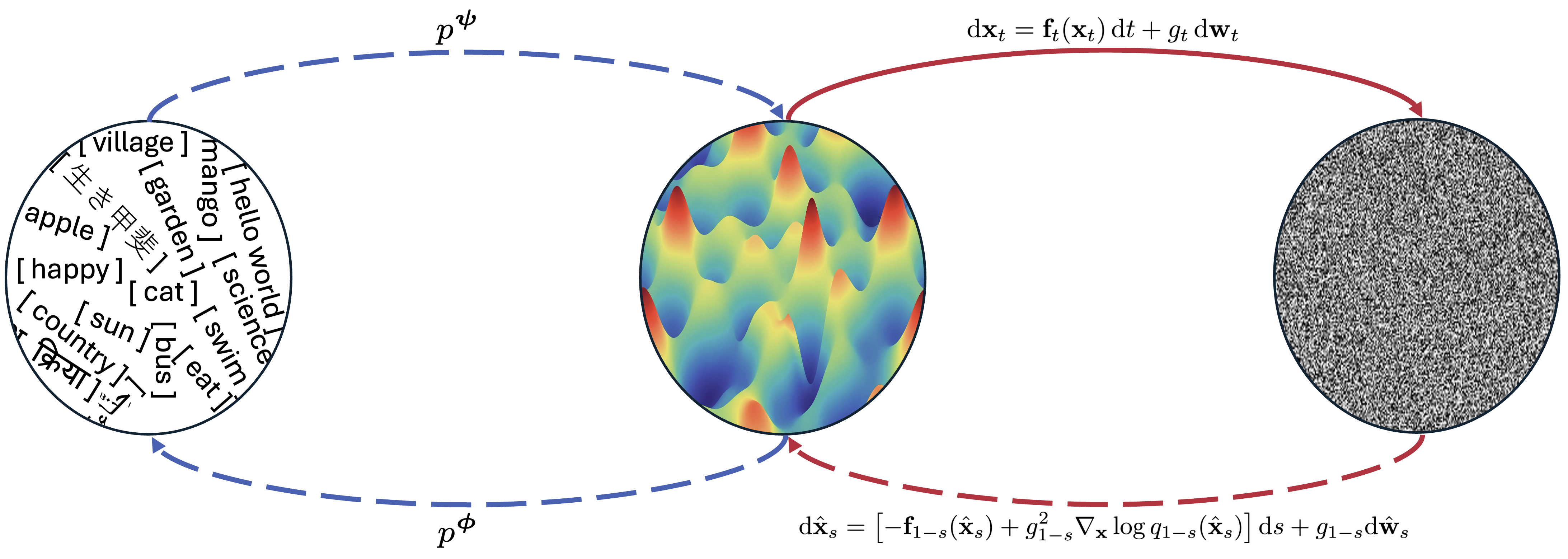}
    \caption{Latent diffusion example. Discrete text tokens are mapped to a continuous latent space via an encoder $q^\psib$. In this latent space, one can perform continuous diffusion, following an SDE, to converge toward the noise distribution. The clean latent representation is then reconstructed using the reverse SDE. Finally, the decoder $p^\phib$ maps the clean latent back to discrete text tokens.}
    \label{fig:latent-diffusion}
\end{figure}

In contemporary diffusion models, the denoising process is often conducted in a learned latent space rather than directly in the high-dimensional data space \cite{liu2023audioldm,podell2023sdxl}. This approach, known as latent diffusion \cite{kingma_variational_2021}, offers significant advantages in terms of scalability and computational efficiency. By operating in a compressed representation space, models can achieve more stable training dynamics and reduced memory requirements.

\subsection{Latent diffusion in general state spaces}

The rationale behind latent diffusion stems from the observation that essential information in data, such as images, text, or biological sequences, is typically concentrated in specific regions. For instance, in images, semantic content often resides in localised areas, a principle already leveraged for a long time by compression algorithms like JPEG and MP3  \cite{dieleman2025latents}. Hence, latent diffusion often improves perceptual fidelity because the model perturbs semantically salient structure rather than raw pixels or tokens.

Formally, we define an encoder function $e: \calX \rightarrow \mathcal{Z}$ that maps data from the original space $\calX$ to a latent space $\mathcal{Z}$, and a decoder $d: \mathcal{Z} \rightarrow \calX$ that reconstructs data back to the original space. The diffusion process is applied within the latent space:
\begin{align}
    \x_0 \xrightarrow{e} \z_0 \xrightarrow{\text{diffusion}} \z_1, \quad
    \z_1 \xrightarrow{\text{reverse diffusion}} \hat  \z_0 \xrightarrow{d} \hat \x_0.
\end{align}

Depending on the generative model employed, the encoder $e$ can be a latent-variable or representation-learning model,\footnote{Example include Generative Adversarial Networks (GANs) \cite{goodfellow2014generative} Variational Autoencoders (VAEs) \cite{kingma2014auto}, Vector Quantized VAEs (VQ-VAEs) \cite{van2017neural}, or VQGANs \cite{esser2021taming}.} resulting in learned conditional models $q^\psib(\z\cond\x)$ for the encoder and $p^\phib(\x\cond\z)$ for the decoder. Alternatively, the encoder and decoder can correspond to deterministic mappings, for instance, projecting finite data spaces onto their probability simplices.\footnote{The probability simplex over a finite set $\calX$ is the set of all probability distributions supported on $\calX$ (i.e. all functions $p: \calX \to [0,1]$ such that $\sum_{x \in \calX}p(x)=1$).}

In the context of learnt mapping, we can write the latent diffusion process as:
\begin{align}
\x_0 \xrightarrow{q^\psib} \z_0 \xrightarrow{q} \cdots \xrightarrow{q} \z_1 \xrightarrow{p^\thetab} \cdots \xrightarrow{p^\thetab} \hat{\z}_0 \xrightarrow{p^\phib} \hat{\x}_0 .
\end{align}

In the context of learnt projections, latent diffusion models can be trained using two different paradigms: joint training or two-stage training. In the joint training paradigm, the encoder, decoder, and diffusion model are optimised simultaneously via a single ELBO objective \cite{vahdat_score-based_2021,wehenkel_diffusion_2021}. Although theoretically well-founded, this approach has seen limited adoption in practice. One hypothesis is that, while improving likelihood consistency, joint training may introduce conflicting gradients between the reconstruction and diffusion terms.

Alternatively, the two-stage training strategy decouples the learning process \cite{rombach_high-resolution_2021}. Initially, an autoencoder is trained to learn the latent representation of the data. Subsequently, the diffusion model is trained on this fixed latent space. This separation offers flexibility in choosing loss functions for each stage and improving generation quality. For instance, the autoencoder is often trained using a combination of \textit{reconstruction loss}, \textit{perceptual loss}, and \textit{adversarial loss}.\footnote{The reconstruction loss is typically a mean squared error (MSE) in input space. Perceptual losses, computed in a learned feature space (e.g., using LPIPS for images \cite{zhang2018unreasonable}, spectral coefficients for audio \cite{yamamoto2020parallel,kong2020hifi}, or GNN-based embeddings for molecules \cite{gilmer2017neural}), preserve perceptual fidelity, especially high-frequency content. Adversarial losses, implemented via a discriminator co-trained with the generator, further enhance realism, as in GANs \cite{goodfellow2020generative,karras2019style}.} The diffusion model can then focus solely on modeling the distribution in the latent space without being influenced by reconstruction constraints.

When jointly training all components, the ELBO must account for the encoder and decoder mappings. We derive the following objective using standard variational upper bound derivation (proof \cref{apx:latent-diff-ELBO}):
\begin{align}
     - \log p^\thetab(\x_0) \leq \underbrace{\KL\left(q^\psib(\z_0 \cond \x_0) \| p^\thetab(\z_0) \right)}_{\text{diffusion loss}} - \underbrace{\E_{q^\psib(\z_0 \cond \x_0)}\left[ \log p^\phib(\x_0 \cond \z_0)\right],}_{\text{reconstruction loss}}
\end{align}
where $q^\psib(\z_0 \cond \x_0)$ is the approximate posterior, $p^\thetab(\z_0)$ the prior in latent space, and $p^\phi(\x_0 \cond \z_0)$ the decoder likelihood.

Applying the data processing inequality to the KL divergence yields:
\begin{align}
    \KL\left(q^\psib(\z_0 \cond \x_0) \| p^\thetab(\z_0) \right) \leq \KL\left(q^\psib(\z_{[0,1]} \cond \x) \| p^\thetab(\z_{[0,1]}) \right)
\end{align}
Thus, the ELBO becomes:
\begin{align}
     - \log p^\thetab(\x_0) \leq \underbrace{\KL\left(q^\psib(\z_{[0,1]} \cond \x_0) \| p^\thetab(\z_{[0,1]}) \right)}_{\text{diffusion loss}} \underbrace{- \E_{q^\psib(\z_0 \cond \x_0)}\left[ \log p^\phib(\x_0 \cond \z_0)\right].}_{\text{reconstruction loss}}
\end{align}
This unified ELBO objective facilitates the simultaneous training of the encoder, decoder, and diffusion process.

In practice, even when diffusion is performed in the data space, the denoised output $\x^\thetab_0$ is considered an approximation of the clean data at a small positive time $\varepsilon > 0$, allowing for a consistent ELBO derivation across both latent and data spaces.

Latent diffusion models have demonstrated significant success in continuous domains, such as image generation, due to their efficient training and sampling processes. In discrete domains, like text or molecular sequences, latent diffusion offers additional benefits. By mapping discrete data into a continuous latent space, standard diffusion techniques such as guidance can be applied more effectively, leveraging the flexibility of continuous models to handle discrete data.

\subsection{Continuous diffusion for discrete data}

An intuitive approach to performing diffusion on discrete data is first to project the data into a continuous space and then run a continuous diffusion process in this transformed space. This approach achieves competitive results, comparable to direct diffusion methods \cite{li_diffusion-lm_2022,gong_diffuseq_2023,dieleman_continuous_2022,strudel_self-conditioned_2022,lin_text_2023,yuan_seqdiffuseq_2022,lovelace_latent_2023, gao_empowering_2024,wu_ar-diffusion_2023}.

A primary challenge of this approach arises in the context of categorical data, where projection into a continuous space is not always straightforward. The difficulty lies in balancing the trade-off between maintaining a continuous representation that closely approximates the discrete data space and employing a more complex, expressive representation that may deviate further away from the original discrete space \cite{dieleman2025latents,jo2025continuous}.

Thus, the success of this approach heavily depends on the choice of the encoder and decoder for mapping between the discrete and continuous state spaces.

Some approaches attempt to tackle this challenge by converting discrete tokens into binary bit strings (analogue bits) and modeling them as real values \cite{chen_analog_2023}, or by performing the diffusion process directly on the token logits (simplex diffusion) \cite{han_ssd-lm_2023,mahabadi_tess_2024}. From a more theoretical perspective, \citet{richemond_categorical_2022}  proposes specific SDE formulations for simplex diffusion using the Cox-Ingersoll-Ross process; however, these formulations lack experimental validation.

Other works have tried to leverage the latent representations of a pre-trained model as continuous representations \cite{gong_diffuseq_2023,li2024discdiff,mittal_symbolic_2021,wang_language_2023,lovelace_latent_2023} or to jointly learn representations and the diffusion process \cite{dieleman_continuous_2022, yuan_seqdiffuseq_2022, liu2024unified}.

The embedding of discrete data is particularly advantageous, as it allows for the application of standard gradient-based guidance methods commonly used in continuous diffusion. While some recent work has explored guidance methods for discrete diffusion \cite{nisonoff2024unlocking, guo_plug-and-play_2024,schiff_simple_2024, lee2025debiasing}, those for continuous diffusion remain more straightforward and well-developed \cite{ho_classifier-free_2022, skreta2025feynman}. For instance, Diffusion-LM \cite{li_diffusion-lm_2022} introduces a novel language model based on continuous diffusion, facilitating complex gradient-based control. Similarly, SED \cite{strudel_self-conditioned_2022} explored self-conditioning within this framework, while DiffuSeq \cite{gong_diffuseq_2023} employs continuous diffusion for text-to-text generation tasks by embedding words in a latent space and conditioning the diffusion process on input text embeddings. SSD-LM \cite{han_ssd-lm_2023}, directly applies continuous diffusion in the vocabulary space using simplex diffusion, allowing the incorporation of classifier guidance using off-the-shelf classifiers. Finally, CDCD \cite{dieleman_continuous_2022} explores continuous-time and state-space diffusion for discrete data and conducts a comprehensive evaluation of guidance methods by jointly learning latent representations and the diffusion process.

\subsection{Bridging continuous and discrete diffusion for discrete data} A growing line of work seeks to explicitly couple \emph{continuous} and \emph{discrete} diffusion for discrete data. Duo \cite{sahoo2025diffusion} shows that uniform-state discrete diffusion can be derived as a marginal of an underlying Gaussian process, enabling the transfer of Gaussian-diffusion techniques (e.g., curriculum schedules, fast few-step sampling) to the discrete setting.
CADD \cite{zheng2025continuously} augments a discrete Markov chain with a paired continuous latent diffusion so that masked tokens carry informative, graded latent states that guide discrete denoising, improving text/image/code generation quality.
CCDD \cite{zhou2025coevolutionary} goes further by defining a joint process on the union of a continuous representation space and a discrete token space, training a single model to co-denoise in both.
Together, these hybrids suggest a promising path to better likelihood training, guidance, and sampling efficiency in discrete generative modeling.

\section*{Acknowledgments}
We thank all members of Prof.\ Dr.\ Stefan Bauer's group, and Andrew Campbell for insightful discussions. This work was partially supported by the Munich Center for Machine Learning (MCML).

\clearpage
\begin{landscape}
    \begin{table}[htbp]
     \centering
    \renewcommand{\arraystretch}{1.5}
    \setlength{\tabcolsep}{6pt}
    \resizebox{\linewidth}{!}{%
    \begin{tabular}{@{} l c c @{}}
        \toprule
        & \textbf{Discrete State Space (CTMC)} & \textbf{Continuous State Space (SDE)} \\
        \midrule
        
        \multicolumn{3}{l}{\textit{Discrete-time formulation (\textsection\ref{sect:discrete_time_diffusion})}} \\[4pt]
        
        \textbf{Forward Marginals}
        &
        $\displaystyle
          q(\x_t \cond \x_0)
          = \prod_{k=1}^{d}\Cat\!\left(x_t^{(k)};\,\p=\Q_t\e_{x_0^{(k)}}\right),
          \qquad
          \Q_t \coloneqq \Q_{t \cond t-1}\Q_{t-1 \cond t-2}\cdots\Q_{1 \cond 0}
        $
        &
        $\displaystyle
          q(\x_t \cond \x_0) 
          = \Normal\bigl(\x_t \,;\, \alpha_t\,\x_0,\, \sigma_t^2\,\I\bigr), 
          \;
          \alpha_t \coloneqq \alpha_{t \cond t-1} \alpha_{t-1}, 
          \;
          \sigma_t^2 \coloneqq \sum_{k=1}^{t} 
          \bigl(\prod_{j=k+1}^{t} \alpha_{j \cond j-1}\bigr)^2
          \sigma_{k \cond k-1}^2
        $
        \\[10pt]
        
        \textbf{Convex Combination Form}
        &
        $\displaystyle
          q(\x_t \cond \x_0)
          = \prod_{k=1}^{d}\Cat\!\left(x_t^{(k)};\,\p=\alpha_t\e_{x_0^{(k)}}+(1{-}\alpha_t)\pnoise\right)
        $
        &
        $\displaystyle
          \x_t = \alpha_t\,\x_0 + \sigma_t\,\epsilonb, \quad \epsilonb\sim \Normal(\zerob,\I)
        $
        \\[10pt]
        
        \textbf{Conditional Reverse Kernel}
        &
        $\displaystyle
          q(\x_{t-1} \cond \x_t, \x_0)
          = \prod_{k=1}^{d}\Cat\!\left(
             x_{t-1}^{(k)};\,
             \p = \frac{\Q_{t|t-1}^\top\e_{x_t^{(k)}} \odot \Q_{t-1}\e_{x_0^{(k)}}}
             {\e_{x_t^{(k)}}^\top\Q_t\e_{x_0^{(k)}}}
           \right)
        $
        &
        $\displaystyle
           q(\x_{t-1} \cond \x_t, \x_0)
          = \Normal\bigl(
            \x_{t-1} \,;\, \mub_{t-1|t}(\x_t, \x_0),\,
            \sigma_{t-1|t}^2\,\I
          \bigr)
          $
        \\[12pt]
        
        \midrule
        \multicolumn{3}{l}{\textit{Continuous-time formulation (\textsection\ref{sect:discrete-to-continuous-time-csp}--\ref{sect:continuous-diffusion})}} \\[4pt]
        
        \textbf{Dynamics} 
        & 
        $\displaystyle
        q_{t \cond t - \Delta t} (\x \cond \y)= \delta_{\x,\y} + \Rseq_t(\x , \y)\, \Delta t + o(\Delta t), \quad \Rseq_t = \bigoplus_{k=1}^d \R_t^{(k)}
        $
        & 
        $\displaystyle
        \diff\x_t = \f_t(\x_t)\, \dt + g_t\, \dw_t
        $ 
        \\[10pt]
            
        \textbf{Rate/Drift (Interpolation)}
        & 
        $\displaystyle
        \R_t = \frac{\alpha^\prime_t}{\alpha_t}\bigl(\I - \pnoise \mathbf{1}\transp\bigr)
        $
        & 
        $\displaystyle
        \f_t(\x_t) = \frac{\alpha^\prime_t}{\alpha_t}\, \x_t, \qquad g_t^2 = \alpha_t^2 \frac{\diff}{\dt} \Bigl[ \frac{\sigma_t^2}{\alpha_t^2} \Bigr]
        $
        \\[10pt]
        
        \textbf{Time Reversal} 
        & 
        $\displaystyle
        \partial_s \hat{\q}_s= \hat{\Rseq}_s\, \hat{\q}_s, \qquad
        \hat{\Rseq}_s(\x, \y)= \Rseq_{1-s}(\y, \x)\,\frac{q_{1-s}(\x)}{q_{1-s}(\y)} \quad (\x \neq \y)
        $
        & 
        $\displaystyle
        \diff \hat{\x}_s = \bigl[-\f_{1-s}(\hat{\x}_s) + g_{1-s}^{2}\,\nabla_\x \log \hat q_{s}(\hat{\x}_s)\bigr]\diff s + g_{1-s}\,\diff \hat{\w}_s
        $
        \\[10pt]
            
        \textbf{Generator Instantiation}
        &
        $\displaystyle
        \Lgen_t \phi(\x) = \sum_{\y \in \calX^d} \Rseq_t(\y, \x)\bigl[\phi(\y) - \phi(\x)\bigr]
        $
        &
        $\displaystyle
        \Lgen_t \phi(\x) = \tfrac{1}{2}g_t^2\,\Delta_\x \phi(\x) + \f_t(\x) \cdot \nabla_\x \phi(\x)
        $
        \\[12pt]
        
        \midrule
        \multicolumn{3}{l}{\textit{Unified generator perspective (\textsection\ref{sect:generator-perspective})}} \\[4pt]
        
        \textbf{Infinitesimal Generator}
        &
        \multicolumn{2}{c}{%
        $\displaystyle
        \Lgen_t \phi(\x)
        \;\coloneqq\;
        \lim_{\Delta t \to 0^+}
        \frac{\E_{q(\cdot \cond \x_t=\x)}\!\left[\phi(\x_{t+\Delta t})\right] - \phi(\x)}{\Delta t}
        $
        }
        \\[10pt]
        
        \textbf{Kolmogorov Forward Equations}
        &
        \multicolumn{2}{c}{%
        $\displaystyle
        \partial_t q_t \;=\; (\Lgen_t)^* q_t, \quad  q_0 = \qdata
         \qquad \qquad \partial_s \hat q_s \;=\; (\hat \Lgen_s)^* \hat q_s, \quad \hat q_0 = p_\text{noise}$
        }
        \\[10pt]
        
        \textbf{Reversed Generator}
        &
        \multicolumn{2}{c}{%
        $\displaystyle
        \hat{\Lgen}_s \phi(\x)
        \;=\; \frac{1}{q_{1-s}(\x)}\Bigl((\Lgen_{1-s})^*\bigl[q_{1-s}\phi\bigr](\x) \;-\; \phi(\x)\,(\Lgen_{1-s})^* q_{1-s}(\x) \Bigr)
        $
        }
        \\[12pt]
        
        \midrule
        \multicolumn{3}{l}{\textit{Training objective (\textsection\ref{sect:ELBO}--\ref{sect:generator-ELBO})}} \\[4pt]
        
        \textbf{ELBO via Generator}
        &
        \multicolumn{2}{c}{%
        $\displaystyle
        - \E_{\qdata}\!\left[\log p_0^\theta(\x_0)\right] 
        \;\leq\; \E_{\hat \Qpath}\!\left[\int_0^1 \!\left(\frac{\hat\Lgen_s \alpha_s(\hat\x_s)}{\alpha_s(\hat\x_s)}  - \hat \Lgen_s\log\alpha_s(\hat\x_s)\right)\diff s \right] + \mathcal{H}(\qdata), \qquad \alpha_s \coloneqq \frac{\hat p_s^\theta}{\hat q_s}
        $
        }
        \\[12pt]
        
        \textbf{Loss Instantiation}
        &
        $\displaystyle
        \mathcal{L}_{\text{DSE}} = \E_{\x_0}\!\int_0^1\!
        \E_{\x_t\cond\x_0}
        \Bigl[\sum_{\y\neq \x_t}
        \Bigl(\hat{\Rseq}_t^\theta(\y,\x_t)
        -\Rseq_t(\x_t,\y)\frac{q(\y\cond\x_0)}{q(\x_t\cond\x_0)}\log \hat{\Rseq}_t^\theta(\y,\x_t)\Bigr)\Bigr]
        \diff t
        $
        &
        $\displaystyle
        \mathcal{L}_{\text{DSM}} = \tfrac{1}{2}\E_{\x_0}\!\int_0^1\!
        g_t^2\,\E_{\x_t\cond\x_0}
        \bigl\|\nabla_{\x_t}\log q(\x_t\cond\x_0)-\mathbf{s}_\theta(\x_t,t)\bigr\|^2
        \diff t
        $
        \\[10pt]
        
        &
        \textit{Denoising Score Entropy (DSE)}
        &
        \textit{Denoising Score Matching (DSM)}
        \\[6pt]
        
        \bottomrule
      \end{tabular}
        }
      \caption{Summary of diffusion generative modelling in discrete and continuous state spaces. The table is organised into four parts: \textit{discrete-time formulation} (forward/reverse kernels), \textit{continuous-time formulation} (dynamics, KFEs, time reversal), \textit{unified generator perspective} (general formulas valid for any Markov process, shown in merged columns), and \textit{training objectives} (general ELBO and state-space specific instantiations). Notation: $\hat{q}_s \coloneqq q_{1-s}$, $\hat{\x}_s \coloneqq \x_{1-s}$ (time-reversed quantities); $(\Lgen_t)^*$ ($L^2$-adjoint of generator); $\hat\Qpath$ (path measure of reversed process); $\Rseq_t$ (rate matrix, $\mathbf{1}^\top\Rseq_t = \mathbf{0}^\top$); $\pnoise$ (noise distribution); $\e_{\x}$ (one-hot); $\odot$ (Hadamard); $\oplus$ (Kronecker sum).}
      \label{tab:diffusion_comparison}
    \end{table}
    \end{landscape}

\clearpage

\printbibliography

\clearpage

\appendix

\section{Derivation of the forward marginals \label{apx:interpolation-marginals}}

\subsection{Continuous state space\label{apx:interpolation-marginals-continuous-diff}}

\begin{lemma}[Forward marginal decomposition \label{lem:forward-marginal-continuous}]
Consider a discrete-time diffusion process with one-step transitions:
\begin{equation}
\x_t = \alpha_{t \cond t-1} \x_{t-1} + \sigma_{t \cond t-1} \epsilonb_t, \qquad \epsilonb_t \sim \mathcal{N}(\mathbf{0}, \I), \quad \epsilonb_t \perp \epsilonb_s \text{ for } t \neq s,
\end{equation}
where $\x_t \in \RR^d$ and $\{\alpha_{t \cond t-1}\}_{t \geq 1}$, $\{\sigma_{t \cond t-1}\}_{t \geq 1}$ are deterministic scalar sequences. Then for any $t \geq 1$:
\begin{equation}\label{eq:forward-decomposition}
\x_t = \alpha_t \x_0 + \sum_{k=1}^{t} \Bigg( \prod_{j=k+1}^{t} \alpha_{j\cond j-1} \Bigg) \sigma_{k \cond k-1} \epsilonb_k,
\end{equation}
where $\alpha_t \coloneqq \alpha_{t \cond t-1} \alpha_{t-1}$ and we use the convention that empty products equal $1$.
\end{lemma}

\begin{proof}
We proceed by induction on $t$.

\textbf{Base Case:} For $t = 1$, we have:

\begin{align}
\x_1 = \alpha_{1 \cond 0} \x_0 + \sigma_{1 \cond 0} \epsilonb_1
\end{align}

We substitute into our proposed formula the first term:
$$\prod_{j=1}^{1} \alpha_{j \cond j-1} = \alpha_{1 \cond 0},$$
and the second term 
$$\sum_{k=1}^{1} \left( \prod_{j=k+1}^{1} \alpha_{j \cond j-1} \right) \sigma_{k \cond k-1} \epsilonb_k = \left( \prod_{j=2}^{1} \alpha_{j \cond j-1} \right) \sigma_{1 \cond 0} \epsilonb_1.$$

Since we define $\prod_{j=2}^{1} \alpha_{j \cond j-1} = 1$ (empty product), we get: $$\x_1 = \alpha_{1 \cond 0} \x_0 + \sigma_{1 \cond 0} \epsilonb_1.$$

This matches the given interpolation relation, so the base case holds.

\textbf{Induction:} Assume the formula holds for some $t = n \geq 1$:
\begin{align}
\x_n = \Bigg( \prod_{j=1}^{n} \alpha_{j \cond j-1} \Bigg) \x_0 + \sum_{k=1}^{n}  \Bigg( \prod_{j=k+1}^{n} \alpha_{j \cond j-1} \Bigg) \sigma_{k \cond k-1} \epsilonb_k
\end{align}
We need to show the formula holds for $t = n + 1$.
Starting with the interpolation relation:
\begin{align}\x_{n+1} = \alpha_{n+1 \cond n} \x_n + \sigma_{n+1 \cond n} \epsilonb_{n+1}\end{align}
Substituting the inductive hypothesis:
\begin{align}\x_{n+1} = \alpha_{n+1 \cond n} \left[ \Bigg( \prod_{j=1}^{n} \alpha_{j \cond j-1} \Bigg) \x_0 + \sum_{k=1}^{n}  \Bigg( \prod_{j=k+1}^{n} \alpha_{j \cond j-1} \Bigg) \sigma_{k \cond k-1} \epsilonb_k \right] + \sigma_{n+1 \cond n} \epsilonb_{n+1}\end{align}
Distributing $\alpha_{n+1 \cond n}$:
\begin{align}\x_{n+1} = \alpha_{n+1 \cond n} \Bigg( \prod_{j=1}^{n} \alpha_{j \cond j-1} \Bigg) \x_0 + \alpha_{n+1 \cond n} \sum_{k=1}^{n}  \Bigg( \prod_{j=k+1}^{n} \alpha_{j \cond j-1} \Bigg) \sigma_{k \cond k-1} \epsilonb_k + \sigma_{n+1 \cond n} \epsilonb_{n+1}\end{align}
For the first term:
\begin{align}
\alpha_{n+1 \cond n} \Bigg( \prod_{j=1}^{n} \alpha_{j \cond j-1} \Bigg) = \prod_{j=1}^{n+1} \alpha_{j \cond j-1}
\end{align}
For the second term, we can factor out $\alpha_{n+1 \cond n}$:
\begin{align}\alpha_{n+1 \cond n} \sum_{k=1}^{n}  \Bigg( \prod_{j=k+1}^{n} \alpha_{j \cond j-1} \Bigg) \sigma_{k \cond k-1} \epsilonb_k = \sum_{k=1}^{n}  \Bigg( \prod_{j=k+1}^{n+1} \alpha_{j \cond j-1} \Bigg) \sigma_{k \cond k-1} \epsilonb_k
\end{align}
For the third term, note that:
\begin{align}\sigma_{n+1 \cond n} \epsilonb_{n+1} = \left( \prod_{j=n+2}^{n+1} \alpha_{j \cond j-1} \right) \sigma_{n+1 \cond n} \epsilonb_{n+1}\end{align}
where $\prod_{j=n+2}^{n+1} \alpha_{j \cond j-1} = 1$ (empty product).
Combining all terms:
\begin{align}\x_{n+1} = \Bigg( \prod_{j=1}^{n+1} \alpha_{j \cond j-1} \Bigg) \x_0 + \sum_{k=1}^{n}  \Bigg( \prod_{j=k+1}^{n+1} \alpha_{j \cond j-1} \Bigg) \sigma_{k \cond k-1} \epsilonb_k + \Bigg( \prod_{j=n+2}^{n+1} \alpha_{j \cond j-1} \Bigg) \sigma_{n+1 \cond n} \epsilonb_{n+1}\end{align}
This can be written as:
\begin{align}\x_{n+1} = \Bigg( \prod_{j=1}^{n+1} \alpha_{j \cond j-1} \Bigg) \x_0 + \sum_{k=1}^{n+1}  \Bigg( \prod_{j=k+1}^{n+1} \alpha_{j \cond j-1} \Bigg) \sigma_{k \cond k-1} \epsilonb_k\end{align}
Therefore, the formula holds for $t = n + 1$.
\textbf{Conclusion}
By induction, the interpolation formula holds for all $t \geq 1$:
\begin{align}\x_t = \Bigg( \prod_{j=1}^{t} \alpha_{j \cond j-1} \Bigg) \x_0 + \sum_{k=1}^{t}  \Bigg( \prod_{j=k+1}^{t} \alpha_{j \cond j-1} \Bigg) \sigma_{k \cond k-1} \epsilonb_k\end{align}

\end{proof}

\begin{corollary}[Gaussian forward marginal]\label{cor:gaussian-forward-marginal}
Under the conditions of \cref{lem:forward-marginal-continuous}, the forward marginal $q(\x_t \cond \x_0)$ is Gaussian:
\begin{equation}
q(\x_t \cond \x_0) = \mathcal{N}\bigl(\x_t \,;\, \alpha_t \x_0,\, \sigma_t^2 \I\bigr),
\end{equation}
where the cumulative signal coefficient $\alpha_t$ and noise variance $\sigma_t^2$ are given by:
\begin{equation}
\alpha_t \coloneqq \alpha_{t \cond t-1} \alpha_{t-1}, \qquad \sigma_t^2 \coloneqq \alpha_{t\cond t-1}^2\sigma_{t-1}^2 + \sigma_{t\cond t-1}^2.
\end{equation}
\end{corollary}

\begin{proof}
Conditioned on $\x_0$, we have:
\begin{equation}
\x_t = \alpha_t \x_0 + \sum_{k=1}^{t} \Bigg( \prod_{j=k+1}^{t} \alpha_{j \cond j-1} \Bigg) \sigma_{k \cond k-1} \epsilonb_k.
\end{equation}

The term $\alpha_t \x_0$ is deterministic given $\x_0$. The second term is a linear combination of independent Gaussian random vectors. Since each $\epsilonb_k \sim \mathcal{N}(\mathbf{0}, \I)$ independently, and a linear combination of independent Gaussians is Gaussian with parameters:
\begin{equation}
\sum_{k=1}^{t} \Bigg( \prod_{j=k+1}^{t} \alpha_{j \cond j-1} \Bigg) \sigma_{k \cond k-1} \epsilonb_k \sim \mathcal{N}\Bigl(\mathbf{0},\, \sigma_t^2 \I\Bigr).
\end{equation}

Combining the deterministic mean and stochastic component yields the result.
\end{proof}

\begin{remark}[Equivalent reparameterisation]
The forward marginal can equivalently be written as:
\begin{equation}
\x_t = \alpha_t \x_0 + \sigma_t \epsilonb, \qquad \epsilonb \sim \mathcal{N}(\mathbf{0}, \I),
\end{equation}
which is the reparameterisation used in the main text (\cref{eq:interpolation-continuousdiff}). This form is computationally convenient as it requires sampling only a single noise vector $\epsilonb$ rather than $t$ independent vectors $\{\epsilonb_k\}_{k=1}^t$.
\end{remark}

\subsection{Discrete state space \label{apx:interpolation-marginals-discrete-state}
}

We derive the closed-form expression for the forward marginals under the convex combination transition matrices introduced in \cref{eq:interpolation-discretediff}.

\begin{lemma}[Cumulative transition matrix]\label{lem:cumulative-transition}
Let $\{\Q_{i \cond i-1}\}_{i=1}^t$ be one-step transition matrices of the form
\begin{equation}
\Q_{i \cond i-1} = \alpha_{i \cond i-1}\,\I + (1 - \alpha_{i \cond i-1})\,\pnoise\,\mathbf{1}^\top,
\end{equation}
where $\alpha_{i \cond i-1} \in [0,1]$, $\I$ is the identity matrix, $\mathbf{1}$ is the all-ones column vector, and $\pnoise \in \Delta^{K-1}$ is the noise distribution. Then the cumulative transition matrix $\Q_t \coloneqq \Q_{t \cond t-1}\Q_{t-1 \cond t-2}\cdots\Q_{1 \cond 0}$ satisfies:
\begin{equation}\label{eq:cumulative-Q-form}
\Q_t = \alpha_t\,\I + (1 - \alpha_t)\,\pnoise\,\mathbf{1}^\top,
\end{equation}
where $\alpha_t \coloneqq \prod_{i=1}^t \alpha_{i \cond i-1}$. Consequently, the forward marginal is:
\begin{equation}
q(x_t \cond x_0) = \Cat\bigl(x_t;\, \Q_t\,\e_{x_0}\bigr) = \Cat\bigl(x_t;\, \alpha_t\,\e_{x_0} + (1-\alpha_t)\,\pnoise\bigr).
\end{equation}
\end{lemma}

\begin{proof}
We proceed by induction on $t$.

\paragraph{Base case:}
For $t = 1$, we have
$$
  \Q_{1 \cond 0}  = \alpha_{1 \cond 0}\,\I \;+\;\bigl(1-\alpha_{1 \cond 0}\bigr)\,\pnoise \mathbf{1}^T.
$$
This matches the required form with $\alpha_1 = \alpha_{1 \cond 0}$ for $t=1$. Hence the base case holds.

\paragraph{Induction:}
Assume that for some $t \ge 1$,
\begin{align}
\Q_t  = \Q_{t \cond t-1}\Q_{t-1 \cond t-2}\cdots\Q_{1 \cond 0} = \alpha_t\,\I \;+\;\bigl(1-\alpha_t\bigr)\, \pnoise\,\mathbf{1}^T
\end{align}
 
We need to show it holds for $t+1$.  We compute
$$
  \Q_{t+1}
   = \;\Q_{t+1 \cond t} \Q_t 
   = 
  \Bigl(\alpha_{t+1 \cond t}\,\I + \bigl(1-\alpha_{t+1 \cond t}\bigr)\,\pnoise \mathbf{1}^T\Bigr)\,
  \Bigl(\alpha_t\,\I + \bigl(1-\alpha_{t}\bigr)\,\pnoise\mathbf{1}^T\Bigr).
$$

Expanding the product gives
$$
  \alpha_{t+1 \cond t}\,\alpha_t\,\I
  \;+\; \alpha_{t+1 \cond t}\bigl(1-\alpha_t\bigr)\,\pnoise\mathbf{1}^T
  \;+\; (1-\alpha_{t+1 \cond t})\,\alpha_t\,\pnoise\mathbf{1}^T
  \;+\; \bigl(1-\alpha_{t+1 \cond t}\bigr)\bigl(1-\alpha_t\bigr)
    \pnoise\mathbf{1}^T\pnoise\mathbf{1}^T.
$$

Since $\pnoise\mathbf{1}^T$ is a matrix whose columns are all equal to the vector $\pnoise$, it follows that 
$$
  \pnoise\mathbf{1}^T \pnoise\mathbf{1}^T
   =  \pnoise\underbrace{(\mathbf{1}^T \pnoise)}_{=1}\mathbf{1}^T =  
  \pnoise\mathbf{1}^T,
$$

because $\pnoise$ is a column probability vector and $\mathbf{1}^\top \pnoise$ sums its components, which equals $1$.

 Grouping terms yields
$$
  \Q_{t+1}
   = 
  \Bigl(\alpha_t\alpha_{t+1 \cond t}\Bigr)\I
  \;+\;
  \bigl[\alpha_{t+1 \cond t}(1-\alpha_t) + (1-\alpha_{t+1 \cond t})\alpha_t + (1-\alpha_{t+1 \cond t})(1-\alpha_t)\bigr]
  \pnoise\mathbf{1}^T.
$$
Using the definition of $\alpha_{t+1} =\alpha_{t+1 \cond t}\alpha_t$, we recognise that the coefficient of $\pnoise\mathbf{1}^T$ simplifies to $1-\alpha_{t+1}$, ensuring that
$$
  \Q_{t+1}  =  \alpha_{t+1}\I + \bigl(1-\alpha_{t+1}\bigr)\pnoise\mathbf{1}^T.
$$
Thus we conclude the induction step, proving the claim.  
\end{proof}

\begin{remark}[Interpretation]
The cumulative transition matrix $\Q_t$ preserves the convex combination structure: with probability $\alpha_t$, the state remains unchanged from $x_0$; with probability $1 - \alpha_t$, it is resampled from the noise distribution $\pnoise$.
\end{remark}
\subsection{Forward multi-step transition in discrete state space \label{apx:multistep-woodburry-identity}}

\begin{lemma}[Invertibility and multi-step transitions]\label{lem:interpolation-inverse}
Let the transition matrix be parameterised as:
\begin{equation}
\Q_t \;=\; \alpha_t\,\I \;+\; (1-\alpha_t)\,\pnoise\,\mathbf{1}^\top,
\end{equation}
where $\alpha_t \in (0,1]$, $\pnoise \in \Delta^{K-1}$ is a probability vector satisfying $\mathbf{1}^\top \pnoise = 1$, and $\mathbf{1} \in \RR^K$ is the all-ones column vector. Then:
\begin{enumerate}
    \item $\Q_t$ is invertible for all $\alpha_t \neq 0$, with inverse:
    \begin{equation}\label{eq:Qt-inverse}
    \Q_t^{-1} \;=\; \frac{1}{\alpha_t}\Bigl(\I \;-\; (1-\alpha_t)\,\pnoise\,\mathbf{1}^\top\Bigr).
    \end{equation}
    
    \item The multi-step transition matrix from time $s$ to time $t$ (with $s < t$) is:
    \begin{equation}\label{eq:Qt|s-formula}
    \Q_{t\cond s} \;=\; \Q_t\, \Q_s^{-1} \;=\; \frac{\alpha_t}{\alpha_s}\,\I \;+\; \Bigl(1-\frac{\alpha_t}{\alpha_s}\Bigr)\,\pnoise\,\mathbf{1}^\top.
    \end{equation}
\end{enumerate}
\end{lemma}

\begin{proof}
We note $\Q_t = \alpha_t\,\I + (1-\alpha_t)\,\pnoise\,\mathbf{1}^\top$. The Sherman–Morrison–Woodbury formula \cite{golub2013matrix} states that for $\alpha \neq 0$ and vectors $\ub, \mathbf{v} \in \RR^K$:
$$
(\alpha\,\I + \mathbf{u}\mathbf{v}\transp)^{-1}
 = 
\frac{1}{\alpha}\,\I 
 - 
\frac{1}{\alpha^2}\;
\frac{\mathbf{u}\mathbf{v}\transp}{1 + \tfrac{1}{\alpha}\,\mathbf{v}^\top\mathbf{u}}.
$$

Here, $\alpha = \alpha_t$, $\mathbf{u} = (1-\alpha_t)\,\pnoise$, and $\mathbf{v} = \mathbf{1}$. Since $\mathbf{1}^\top \pnoise = 1$, we get
\begin{align}
\Q_s^{-1}
 = 
\bigl[\alpha_s\,\I + (1-\alpha_s)\,\pnoise\,\mathbf{1}^\top\bigr]^{-1}
 = 
\frac{1}{\alpha_s}\,\I 
 - 
\frac{(1-\alpha_s)}{\alpha_s}
\;\pnoise\,\mathbf{1}^\top.
\end{align}

We now compute:
\begin{align}
\Q_{t}\,\Q_{s}^{-1}
 = 
\Bigl(\alpha_t\,\I \;+\; (1-\alpha_t)\,\pnoise\,\mathbf{1}^\top\Bigr)\,
\Bigl(\tfrac{1}{\alpha_s}\,\I  -  \tfrac{1-\alpha_s}{\alpha_s}\,\pnoise\,\mathbf{1}^\top\Bigr).
\end{align}
By developing and simplifying, we get:
\begin{align}
\Q_t \, \Q_s^{-1}
 = 
\frac{\alpha_t}{\alpha_s}\,\I
\;+\;
\Bigl(1-\tfrac{\alpha_t}{\alpha_s}\Bigr)\,\pnoise\,\mathbf{1}^\top
\end{align}
\end{proof}

\section{Closed-form of the conditional reverse transition \label{apx:closed-form-reverse-kernel}}

\subsection{Continuous state space \label{apx:continuous-space-reverse-kernel}
}
We derive the closed-form expression for the reverse transition kernel $q(\x_{t-1} \cond \x_t, \x_0)$ in continuous state spaces, as stated in \cref{eq:discrete-time-reversal-cont}.

\begin{lemma}[Gaussian reverse kernel]\label{lem:gaussian-reverse-kernel}
Consider a discrete-time diffusion process in continuous state space $\RR^d$ with forward transitions:
\begin{equation}
q(\x_t \cond \x_{t-1}) = \mathcal{N}(\x_t \,;\, \alpha_{t|t-1} \x_{t-1},\, \sigma_{t|t-1}^2 \I),
\end{equation}
and forward marginals:
\begin{equation}
q(\x_t \cond \x_0) = \mathcal{N}(\x_t \,;\, \alpha_t \x_0,\, \sigma_t^2 \I),
\end{equation}
where $\alpha_t \coloneqq \prod_{j=1}^t \alpha_{j|j-1}$ and $\sigma_{t\cond s}^2 \coloneqq \sigma_t^2 - \alpha_{t|s}^2 \sigma_s^2$ for $s < t$.

Then the reverse kernel conditioned on $\x_0$ is Gaussian:
\begin{equation}
q(\x_{t-1} \cond \x_t, \x_0) = \mathcal{N}\left(\x_{t-1} \,;\, \mub_{t-1|t}(\x_t, \x_0),\, \sigma_{t-1|t}^2 \I\right),
\end{equation}
with mean and variance:
\begin{align}
\mub_{t-1|t}(\x_t, \x_0) &= \frac{\alpha_{t|t-1}\, \sigma_{t-1}^2}{\sigma_t^2}\, \x_t + \frac{\alpha_{t-1}\, \sigma_{t|t-1}^2}{\sigma_t^2}\, \x_0, \label{eq:reverse-mean-apx}\\[0.5em]
\sigma_{t-1|t} &= \sigma_{t|t-1} \cdot \frac{\sigma_{t-1}}{\sigma_t}. \label{eq:reverse-var-apx}
\end{align}

Moreover, the reverse kernel factorises across dimensions:
\begin{equation}
q(\x_{t-1} \cond \x_t, \x_0) = \prod_{k=1}^d q\left(x_{t-1}^{(k)} \cond x_t^{(k)}, x_0^{(k)}\right).
\end{equation}
\end{lemma}

\begin{proof}
We proceed in three steps: (i) establish the Gaussian property via Bayes' rule, (ii) derive the parameters by completing the square, and (iii) verify dimensional factorisation.

\textbf{Part I: Establish Gaussian Property:}

The reverse kernel $q(\x_{t-1} \cond \x_t, \x_0)$ is Gaussian because, using Bayes' rule:
$$q(\x_{t-1} \cond \x_t, \x_0) = \frac{q(\x_t \cond \x_{t-1}, \x_0) \cdot q(\x_{t-1} \cond \x_0)}{q(\x_t \cond \x_0)}.$$
Where:
\begin{enumerate}
    \item $q(\x_t \cond \x_{t-1}, \x_0) = q(\x_t \cond \x_{t-1})$ is Gaussian with linear dependence on $\x_{t-1}$ in the mean.
    \item $q(\x_{t-1} \cond \x_0)$ is Gaussian (from the forward diffusion marginal).
    \item $q(\x_t \cond \x_0)$ is constant with respect to $\x_{t-1}$.
\end{enumerate}
As a function of $\x_{t-1}$, the product of the two Gaussian factors in the numerator is proportional to a Gaussian density.

\textbf{Substitute known distributions:}

From the forward diffusion process, we know: 
\begin{align}q(\x_t \cond \x_{t-1}) &= \mathcal{N}(\x_t;\alpha_{t \cond t-1} \x_{t-1}, \sigma_{t \cond t-1}^2 \mathbf{I}) \qquad
q(\x_{t-1} \cond \x_0) = \mathcal{N}(\x_{t-1}; \alpha_{t-1} \x_0, \sigma_{t-1}^2 \mathbf{I})\\
q(\x_t \cond \x_0) &= \mathcal{N}(\x_t; \alpha_t \x_0,\sigma_t^2 \mathbf{I})
\end{align}
Where $\alpha_{t\cond s} = \prod_{k=s}^t \alpha_{k \cond k-1}$ , $\alpha_t=\alpha_{t\cond0}$, $\sigma_{t\cond s}^2 \coloneqq \sigma_t^2 - \alpha_{t|s}^2 \sigma_s^2$ and $\sigma_t^2 = 1-\alpha_t$ for all times $s<t$.

\begin{align}
&q\left(\x_{t-1} \cond \x_t, \x_0\right)   =q\left(\x_t \cond \x_{t-1}, \x_0\right) \frac{q\left(\x_{t-1} \cond \x_0\right)}{q\left(\x_t \cond \x_0\right)} \\
& \propto \exp \left(-\frac{1}{2}\left(\frac{\|\x_t-\alpha_{t \cond t-1} \x_{t-1}\|^2}{\sigma_{t \cond t-1}^2}+\frac{\|\x_{t-1}-\alpha_{t-1} \x_0\|^2}{\sigma^2_{t-1}}\right)\right) \\
& =\exp \left(-\frac{1}{2}\left(\frac{\|\x_t\|^2-2 \alpha_{t \cond t-1} \x_t\transp   \x_{t-1}+\alpha^2_{t \cond t-1} \|\x_{t-1}\|^2}{\sigma_{t \cond t-1}^2}+\frac{\|\x_{t-1}\|^2-2 \alpha_{t-1} \x_0\transp \x_{t-1}+\alpha^2_{t-1} \|\x_0\|^2}{\sigma^2_{t-1}}\right)\right)
\end{align}

\textbf{Part II: Complete the Square}

The exponent has the form $-\frac{1}{2}(A\|\x_{t-1}\|^2 + \B\transp\x_{t-1} + C)$ where: 

\begin{align}
    A &= \frac{\alpha^2_{t \cond t-1}}{\sigma_{t \cond t-1}^2} + \frac{1}{\sigma_{t-1}^2} \qquad
    \B = -\frac{2\alpha_{t \cond t-1} \x_t}{\sigma_{t \cond t-1}^2} - \frac{2\alpha_{t-1} \x_0}{\sigma^2_{t-1}},\\
    C &= \frac{\|\x_t\|^2}{\sigma_{t \cond t-1}^2} + \frac{\alpha^2_{t-1} \|\x_0\|^2}{\sigma_{t-1}^2}.
\end{align}

Completing the square: $A \|\x_{t-1}\|^2 + \B\transp \x_{t-1} + C = \frac{\|\x_{t-1} - \mub\|^2}{\sigma^2} + \text{Const}_{\x_0,\x_t}$, where $\mub = -\frac{\B}{2A}$ and the variance $\sigma^2 = \frac{1}{A}$.

For the mean, we have:
\begin{align}
    \mub(\x_t, \x_0) = \frac{\frac{\alpha_{t \cond t-1} \x_t}{\sigma_{t \cond t-1}^2} + \frac{\alpha_{t-1} \x_0}{\sigma^2_{t-1}}}{\frac{\alpha^2_{t \cond t-1}}{\sigma_{t \cond t-1}^2} + \frac{1}{\sigma^2_{t-1}}}.
\end{align}
We first compute the denominator:
\begin{align}
\frac{\alpha^2_{t \cond t-1}}{\sigma_{t \cond t-1}^2} + \frac{1}{\sigma^2_{t-1}} &= \frac{\alpha^2_{t \cond t-1}(\sigma^2_{t-1}) + \sigma_{t \cond t-1}^2}{\sigma_{t \cond t-1}^2(\sigma^2_{t-1})}\\
&= \frac{\sigma^2_t}{\sigma_{t \cond t-1}^2 (\sigma_{t-1}^2)}.
\end{align}
Noting $\mub_{t-1 \cond t}\coloneqq \mub$ and $\sigma_{t-1 \cond t}=\sigma$, we obtain:
\begin{align}
    \mub_{t-1 \cond t}(\x_t, \x_0) = \frac{\alpha_{t \cond t-1} \sigma_{t-1}^2 \x_t + \alpha_{t-1}\sigma_{t\cond t-1}^2 \x_0}{\sigma^2_t} \qquad 
    \sigma^2_{t-1 \cond t} = \frac{\sigma_{t \cond t-1}^2\sigma_{t-1}^2}{\sigma^2_t}.
\end{align}

The reverse kernel is:
$$q(\x_{t-1} \cond \x_t, \x_0) = \mathcal{N}\left(\x_{t-1} ; \mub_{t-1 \cond t}(\x_t, \x_0), \sigma^2_{t-1 \cond t} \mathbf{I}\right)$$

with:
\begin{align}
    \mub_{t-1 \cond t}(\x_t, \x_0) = \frac{\alpha_{t\cond t-1} \sigma_{t-1}^2 \x_t + \alpha_{t-1}\sigma_{t\cond t-1}^2 \x_0}{\sigma^2_t} \qquad
\sigma_{t-1 \cond t}  = \sigma_{t\cond t-1} \frac{\sigma_{t-1}}{\sigma_t}
\end{align}

which establishes \cref{eq:reverse-mean-apx,eq:reverse-var-apx}.

\textbf{Part III: Dimensional factorisation:}
Since all covariance matrices in the forward process are diagonal (proportional to $\I$), the reverse kernel also has diagonal covariance $\sigma_{t-1|t}^2 \I$. Furthermore, the mean $\mub_{t-1|t}$ is computed component-wise:
\begin{equation}
\mu_{t-1|t}^{(k)}(x_t^{(k)}, x_0^{(k)}) = \frac{\alpha_{t|t-1}\, \sigma_{t-1}^2}{\sigma_t^2}\, x_t^{(k)} + \frac{\alpha_{t-1}\, \sigma_{t|t-1}^2}{\sigma_t^2}\, x_0^{(k)}.
\end{equation}

A multivariate Gaussian with diagonal covariance factorises into independent univariate Gaussians, yielding:
\begin{equation}
q(\x_{t-1} \cond \x_t, \x_0) = \prod_{k=1}^d \mathcal{N}\left(x_{t-1}^{(k)} \,;\, \mu_{t-1|t}^{(k)}(x_t^{(k)}, x_0^{(k)}),\, \sigma_{t-1|t}^2\right).
\end{equation}
\end{proof}

\subsection{Discrete state space \label{apx:discrete-space-reverse-kernel}}

We derive the closed-form expression for the reverse transition kernel $q(\x_{t-1} \cond \x_t, \x_0)$ in discrete state spaces, as stated in \cref{eq:discrete-time-reversal}.

\begin{lemma}[Factorisation of the reverse transitions]\label{lem:reverse-kernel-factorisation}
Let $\{\x_t\}_{t \in \{0,\ldots,T\}}$ be a discrete-time Markov chain on $\calX^d$ with i.i.d.\ forward noising across dimensions, governed by one-step transition matrices $\Q_{t \cond t-1} \in [0,1]^{K \times K}$. Define the cumulative transition matrices $\Q_t \coloneqq \Q_{t \cond t-1}\Q_{t-1 \cond t-2}\cdots \Q_{1 \cond 0}$ and $\Q_{t|s} \coloneqq\Q_{t \cond t-1}\Q_{t-1 \cond t-2}\cdots \Q_{s+1 \cond s}$ for $s < t$.

Then the reverse kernel conditioned on $\x_0$ factorises across dimensions:
\begin{equation}\label{eq:reverse-kernel-factorised}
q(\x_{t-1} \cond \x_t, \x_0) = \prod_{k=1}^d q\bigl(x_{t-1}^{(k)} \cond x_t^{(k)}, x_0^{(k)}\bigr),
\end{equation}
where each component is a categorical distribution:
\begin{equation}\label{eq:apx:reverse-kernel-categorical}
q\bigl(x_{t-1}^{(k)} \cond x_t^{(k)}, x_0^{(k)}\bigr) = \Cat \left(x^{(k)}_{t-1} \;;\; \p = \frac{\Q_{t|t-1}^\top \e_{x^{(k)}_t} \odot \Q_{t-1}\e_{x^{(k)}_0}}{\e_{x^{(k)}_t}^\top \Q_t \e_{x^{(k)}_0}}\right).
\end{equation}
Here $\e_{x} \in \{0,1\}^K$ denotes the one-hot encoding of state $x \in \calX$, and $\odot$ denotes the Hadamard (element-wise) product.
\end{lemma}

\begin{proof}
Under i.i.d.\ forward noising across dimensions, the joint transitions factorise:

\begin{align}
q\bigl(\x_t\cond\x_{t-1}\bigr)
=
\prod_{k=1}^d
q\bigl(x_t^{(k)}\cond x_{t-1}^{(k)}\bigr),
\qquad
q\bigl(\x_{t-1}\cond\x_0\bigr)
=
\prod_{k=1}^d
q\bigl(x_{t-1}^{(k)}\cond x_0^{(k)}\bigr),
\qquad
q\bigl(\x_t\cond\x_0\bigr)
=
\prod_{k=1}^d
q\bigl(x_t^{(k)}\cond x_0^{(k)}\bigr).
\end{align}
By the Markov property, $q(\x_t \cond \x_{t-1}, \x_0) = q(\x_t \cond \x_{t-1})$. Applying Bayes' rule:
$$
q(\x_{t-1}\cond\x_t,\x_0)
= \frac{q(\x_{t-1} \cond \x_0)}{q(\x_t \cond \x_0)}q(\x_t \cond \x_{t-1},\x_0)=
\prod_{k=1}^d
\frac{q\bigl(x_t^{(k)}\cond x_{t-1}^{(k)}\bigr)\,
      q\bigl(x_{t-1}^{(k)}\cond x_0^{(k)}\bigr)}
     {q\bigl(x_t^{(k)}\cond x_0^{(k)}\bigr)}
=
\prod_{k=1}^d
q\bigl(x_{t-1}^{(k)}\cond x_t^{(k)},x_0^{(k)}\bigr),
$$
We can now derive the single dimension reverse transition, for all $i,j \in \calX$:

\begin{equation}
    \begin{aligned}
q\left(x_{t-1}=i \cond x_t=j, x_0\right) &=\frac{q\left(x_t=j \cond x_{t-1}=i\right) q\left(x_{t-1}=i \cond x_0\right)}{q\left(x_t=j \cond x_0\right)} \\
& =\frac{[\Q_{t \cond t-1}]_{j,i} \, [\Q_{t-1} \e_{x_0}]_i}{[\Q_t \e_{x_0}]_j} \\
\implies q(x_{t-1} \cond x_t, x_0)  &=  \frac{\ \Q_{t \cond t-1}\transp  \e_{x_t}\odot  \Q_{t-1}\e_{x_0}}{\e_{x_t}\transp \Q_t \e_{x_0} }
\end{aligned}
\end{equation}

Where $\e_{x_t}$ the one-hot vector representation of $x_t$.
\end{proof}

\emph{Verification of the formula:} For $x_t = j$ (so $\e_{x_t} = \e_j$):
\begin{enumerate}
\item $\Q_{t|t-1}^\top \e_j$ is a $K \times 1$ column vector whose $i$-th component is $[\Q_{t \cond t-1}]_{j,i} = q(x_t = j \cond x_{t-1} = i)$ (since $\e_{x_t}$ picks out row $j$ as a column vector).

\item $\Q_{t-1} \e_{x_0}$ as a $K\times 1$ column vector whose $i$-th component is $q(x_{t-1} = i \cond x_0)$.  

\item The element-wise product $\odot$ of these two column vectors: the $i$-th coordinate becomes $q(x_t = j \cond x_{t-1} = i) \cdot q(x_{t-1} = i \cond x_0)$.  

\item Finally, $\e_{x_t}^\top \Q_t \e_{x_0} = [\Q_t]_{j, x_0} = q(x_t = j \cond x_0)$ because $\e_{x_t}\transp$ picks out coordinate $x_t=j$.
\end{enumerate}

\textbf{Conclusion:} The reverse kernel factorises across dimensions, with each component being a categorical distribution determined by the corresponding transition matrices and initial conditions. This completes the proof of the factorisation property in discrete state space.

\section{Clean data estimation via learned posteriors \label{apx:reversal-austin}}

This appendix derives the token-level reverse transition formula used in \cref{eq:reversal-austin,eq:reversal-austin-nn}, which enables parameterisation of the reverse process through a learned posterior $p_\theta(\x_0 \cond \x_t)$.

\begin{lemma}[Token-level reverse transition]\label{lem:token-reverse}
Consider a discrete-time Markov diffusion process $\{\x_t\}_{t\in\{0,\dots,T\}}$, where $\x_t = (x^{(1)}_t, \dots, x^{(d)}_t) \in \calX^d$. Assume the forward process factorises independently across coordinates:
\begin{equation}
q(\x_t \cond \x_{t-1}) = \prod_{k=1}^d q(x_t^{(k)} \cond x_{t-1}^{(k)}), 
\qquad
q(\x_t \cond \x_0) = \prod_{k=1}^d q(x_t^{(k)} \cond x_0^{(k)}).
\end{equation}
Then the token-level reverse transition satisfies:
\begin{equation}\label{eq:token-reverse-formula}
q(x_{t-1}^{(k)} \cond \x_t) = q(x_t^{(k)} \cond x_{t-1}^{(k)}) \sum_{x_0^{(k)}} \frac{q(x_{t-1}^{(k)} \cond x_0^{(k)})}{q(x_t^{(k)} \cond x_0^{(k)})} \, q(x_0^{(k)} \cond \x_t),
\end{equation}
where $q(x_0^{(k)} \cond \x_t) = \sum_{\x_0^{\setminus k}} q(\x_0 \cond \x_t)$ denotes the marginal posterior over the $k$-th coordinate.
\end{lemma}

\paragraph{Proof.} 

The proof is algebraic, relying on Bayes' rule and the coordinate-wise independence of the forward process and Markov property. We use the notation $\x^{\setminus k} \coloneqq (x^{(1)}, \dots, x^{(k-1)}, x^{(k+1)}, \dots, x^{(d)})$ to denote all coordinates except the $k$-th.

\begin{align*} q\left(x_{t-1}^{(k)}\,\cond\, \x_t\right) &= \sum_{\x_{t-1}^{\setminus k}} q\left(\x_{t-1}\,\cond\, \x_t\right) \\
&= \sum_{\x_{t-1}^{\setminus k}} q\left(\x_t\,\cond\, \x_{t-1}\right)\, \frac{q\left(\x_{t-1}\right)}{q\left(\x_t\right)} \\
&= \sum_{\x_{t-1}^{\setminus k}} q\left(\x_t\,\cond\, \x_{t-1}\right)\, \sum_{\x_0}\frac{q\left(\x_{t-1}\,\cond\, \x_0\right)q(\x_0)}{q\left(\x_t\right)} \\
&= \sum_{\x_{t-1}^{\setminus k}} q\left(\x_t\,\cond\, \x_{t-1}\right)\, \sum_{\x_0}\frac{q\left(\x_{t-1}\,\cond\, \x_0\right)}{q\left(\x_t\,\cond\, \x_0\right)}\,q(\x_0\,\cond\, \x_t) \\
&= \sum_{\x_0} q(\x_0\,\cond\, \x_t)\; \sum_{\x_{t-1}^{\setminus k}} \frac{q\left(\x_t\,\cond\, \x_{t-1}\right)\,q\left(\x_{t-1}\,\cond\, \x_0\right)}{q\left(\x_t\,\cond\, \x_0\right)} \\
&= \sum_{\x_0} q(\x_0\,\cond\, \x_t)\; \sum_{\x_{t-1}^{\setminus k}} \frac{\prod_{j} q\left(x_t^{(j)}\,\cond\, x_{t-1}^{(j)}\right)\; \prod_{j} q\left(x_{t-1}^{(j)}\,\cond\, x_{0}^{(j)}\right)} {\prod_{j} q\left(x_t^{(j)}\,\cond\, x_{0}^{(j)}\right)} \\
&= \sum_{\x_0} q(\x_0\,\cond\, \x_t)\; \sum_{\x_{t-1}^{\setminus k}} \biggl[ \frac{q\left(x_t^{(k)}\,\cond\, x_{t-1}^{(k)}\right)\,q\left(x_{t-1}^{(k)}\,\cond\, x_{0}^{(k)}\right)} {q\left(x_t^{(k)}\,\cond\, x_{0}^{(k)}\right)} \prod_{j\neq k} \frac{q\left(x_t^{(j)}\,\cond\, x_{t-1}^{(j)}\right)\,q\left(x_{t-1}^{(j)}\,\cond\, x_{0}^{(j)}\right)} {q\left(x_t^{(j)}\,\cond\, x_{0}^{(j)}\right)} \biggr] \\
&= \sum_{\x_0} q(\x_0\,\cond\, \x_t)\; \biggl[ \frac{q\left(x_t^{(k)}\,\cond\, x_{t-1}^{(k)}\right)\,q\left(x_{t-1}^{(k)}\,\cond\, x_{0}^{(k)}\right)} {q\left(x_t^{(k)}\,\cond\, x_{0}^{(k)}\right)} \prod_{j\neq k} \sum_{x_{t-1}^{(j)}} \frac{q\left(x_t^{(j)}\,\cond\, x_{t-1}^{(j)}\right)\,q\left(x_{t-1}^{(j)}\,\cond\, x_{0}^{(j)}\right)} {q\left(x_t^{(j)}\,\cond\, x_{0}^{(j)}\right)} \biggr] \\
&= \sum_{\x_0} q(\x_0\,\cond\, \x_t)\; \biggl[ \frac{q\left(x_t^{(k)}\,\cond\, x_{t-1}^{(k)}\right)\,q\left(x_{t-1}^{(k)}\,\cond\, x_{0}^{(k)}\right)} {q\left(x_t^{(k)}\,\cond\, x_{0}^{(k)}\right)} \biggr] \\
&= q\left(x_t^{(k)}\,\cond\, x_{t-1}^{(k)}\right)\; \sum_{\x_0} q(\x_0\,\cond\, \x_t)\, \frac{q\left(x_{t-1}^{(k)}\,\cond\, x_{0}^{(k)}\right)} {q\left(x_t^{(k)}\,\cond\, x_{0}^{(k)}\right)} \\
&= q\left(x_t^{(k)}\,\cond\, x_{t-1}^{(k)}\right)\; \sum_{x_0^{(k)}} \frac{q\left(x_{t-1}^{(k)}\,\cond\, x_0^{(k)}\right)} {q\left(x_t^{(k)}\,\cond\, x_0^{(k)}\right)}\, q\left(x_0^{(k)}\,\cond\, \x_t\right). \end{align*}
This concludes the proof.

\section{From discrete to continuous-time in discrete-state space \label{apx:discrete-to-continuous-time-discrete-state}
}

\label{apx:discrete-to-continuous-time-discrete-state}

To derive a continuous-time diffusion process from a discrete-time Markov chain, we consider a process discretised over a fixed time interval $[0,1]$ with increasingly fine time steps $\Delta t \coloneqq 1/T$. The continuous-time process is obtained by taking the limit as $\Delta t \to 0$, which defines a rate matrix $\R_t$ characterising the infinitesimal dynamics.

\begin{lemma}[Rate matrix for interpolating transitions]\label{lem:rate-matrix-interpolation}
Consider a discrete state space diffusion process with cumulative transition matrix parameterised as an interpolation (\cref{eq:interpolation-discretediff}):
\begin{equation}
\Q_t = \alpha_t \I + (1-\alpha_t)\,\pnoise \mathbf{1}^\top,
\end{equation}
where $\alpha_t \in [0,1]$ is a differentiable, non-increasing noise schedule and $\pnoise \in \Delta^{K-1}$ is the noise distribution. The conditional transition matrix from time $t-\Delta t$ to time $t$ is:
\begin{equation}
\Q_{t|t-\Delta t} = \frac{\alpha_t}{\alpha_{t-\Delta t}}\,\I + \left(1 - \frac{\alpha_t}{\alpha_{t-\Delta t}}\right)\pnoise \mathbf{1}^\top.
\end{equation}
Then the rate matrix governing the continuous-time limit is:
\begin{equation}\label{eq:rate-matrix-result}
\R_t \;\coloneqq\; \lim_{\Delta t \to 0} \frac{\Q_{t|t-\Delta t} - \I}{\Delta t} \;=\; \frac{\alpha'_t}{\alpha_t}\bigl(\I - \pnoise\mathbf{1}^\top\bigr),
\end{equation}
where $\alpha'_t = \frac{\diff}{\diff t}\alpha_t \leq 0$.
\end{lemma}

\begin{proof}
Substituting the expression for $\Q_{t|t-\Delta t}$ into the rate matrix definition:

\begin{align}
    \frac{\Q_{t|t-\Delta t} -\I}{\Delta t}&=\frac{\alpha_{t|t-\Delta t}\I + (1-\alpha_{t|t-\Delta t})\pnoise \mathbf{1}\transp  - \I}{\Delta t}\\
    &= \frac{(1-\frac{\alpha_{t}}{\alpha_{t-\Delta t}})(\pnoise \mathbf{1}\transp  - \I)}{\Delta t}\\
    &= \frac{\alpha_{t-\Delta t}-\alpha_t}{\Delta t}\frac{(\pnoise \mathbf{1}\transp  - \I)}{\alpha_{t-\Delta t}}
\end{align}
As $\Delta t \to 0$:
 \begin{align}
     \lim_{\Delta t \rightarrow 0}\frac{\alpha_{t-\Delta t}-\alpha_t}{\Delta t}=-\alpha'_t\\
     \lim_{\Delta t \rightarrow 0} \alpha_{t-\Delta t} = \alpha_t
 \end{align}
Finally:

\begin{align}
    \R_t &= \lim_{\Delta t\rightarrow 0} \frac{\alpha_{t-\Delta t}-\alpha_t}{\Delta t}\frac{(\pnoise \mathbf{1}\transp  - \I)}{\alpha_{t-\Delta t}}\\
    &= \frac{\alpha'_t}{\alpha_t}(\I - \pnoise\mathbf{1}\transp).
\end{align}
\end{proof}

\begin{remark}[Properties of the rate matrix]
The rate matrix $\R_t$ satisfies the required properties for a CTMC generator:
\begin{enumerate}[label=(\roman*),leftmargin=*,itemsep=2pt]
\item \emph{Off-diagonal non-negativity}: For $i \neq j$, $[\R_t]_{ij} = \frac{\alpha'_t}{\alpha_t}[-\pnoise]_i \geq 0$ since $\alpha'_t \leq 0$ and $\alpha_t > 0$.
\item \emph{Column-sum-to-zero}: $\mathbf{1}^\top \R_t = \frac{\alpha'_t}{\alpha_t}\mathbf{1}^\top(\I - \pnoise\mathbf{1}^\top) = \frac{\alpha'_t}{\alpha_t}(\mathbf{1}^\top - \mathbf{1}^\top) = \mathbf{0}^\top$.
\end{enumerate}
\end{remark}

\section{Sequence-level rate matrix\label{apx:rate-matrix-sequence-lvl}}

\paragraph{Setup.} 
Let $\{\x_t\}_{t \in [0,1]}$ be a continuous-time Markov chain (CTMC) on the product space $\calX^d = \calX^{(1)} \times \cdots \times \calX^{(d)}$, where each $\calX^{(k)}$ has cardinality $K$. We adopt the ML convention (destination = row, source = column) for all transition and rate matrices. Let $\R_t^{(k)} \in \RR^{K \times K}$ denote the rate matrix for coordinate $k$, satisfying:
\begin{enumerate}
    \item $[\R_t^{(k)}]_{ij} \geq 0$ for $i \neq j$ (non-negative off-diagonal entries),
    \item $\mathbf{1}^\top \R_t^{(k)} = \mathbf{0}^\top$ (columns sum to zero).
\end{enumerate}

\begin{lemma}[Kronecker sum structure of sequence-level rate matrices]\label{lem:kronecker-sum-rate}
Assume an independent coordinate-wise forward noising process. Then the sequence-level rate matrix $\Rseq_t \in \RR^{K^d \times K^d}$ is the \emph{Kronecker sum} of the per-coordinate rate matrices:
\begin{equation}
\Rseq_t
\;=\; \bigoplus_{k=1}^{d} \R_t^{(k)}
\;\coloneqq\; \sum_{k=1}^{d} \Bigl(\I^{\otimes(k-1)} \otimes \R_t^{(k)} \otimes \I^{\otimes(d-k)}\Bigr),
\label{eq:kronecker-sum-rate-apx}
\end{equation}
where $\I \in \RR^{K \times K}$ is the identity matrix and $\I^{\otimes 0} \coloneqq 1$ by convention.
\end{lemma}

\begin{proof}
We provide both geometric intuition and algebraic derivation.
\paragraph{Geometric intuition (d=2).}
List joint states as a grid of pairs $(i,j)\in\calX^{(1)}\times\calX^{(2)}$.
\begin{enumerate}
\item Horizontal arrows $(i,j)\to(i',j)$ are caused by jumps in coordinate 1 with rates $\R_t^{(1)}(i',i)$.
\item Vertical arrows $(i,j)\to(i,j')$ are caused by jumps in coordinate 2 with rates $\R_t^{(2)}(j',j)$.
\item Diagonal moves $(i,j)\to(i',j')$ require two simultaneous single–site jumps and hence occur with probability $O(\Delta t^2)$.
\end{enumerate}
Thus, to a first order approximation, \emph{only} one coordinate changes.

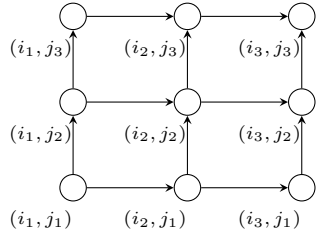
\begin{figure}[h!]
\centering
\begin{tikzpicture}[>=stealth,scale=0.95]
  \tikzset{node style/.style={circle,draw,inner sep=1.2pt,minimum size=10pt}}
  \foreach \i in {1,2,3}{
    \foreach \j in {1,2,3}{
      \node[node style] (n\i\j) at (1.6*\i,1.2*\j) {};
      \node at (1.6*\i-0.45,1.2*\j-0.45) {\scriptsize $(i_\i,j_\j)$};
    }
  }
  \foreach \j in {1,2,3}{
    \foreach \i/\ip in {1/2,2/3}{
      \draw[->] (n\i\j) -- (n\ip\j);
    }
  }
  \foreach \i in {1,2,3}{
    \foreach \j/\jp in {1/2,2/3}{
      \draw[->] (n\i\j) -- (n\i\jp);
    }
  }
\end{tikzpicture}
\caption{For $d{=}2$, $O(\Delta t)$ moves are horizontal or vertical (single–coordinate),
while diagonal two–coordinate moves are $O(\Delta t^2)$ and vanish in the rate matrix.}
\end{figure}

\paragraph{Algebraic perspective.}

Under the independent coordinate-wise noising assumption, the transition matrix on the product space factorises
as a Kronecker product of one-dimensional transitions: 

\begin{align}q(\x_{t+\Delta t}=\y \cond \x_t=\x)
=\prod_{k=1}^d q\bigl(x^{(k)}_{t+\Delta t}=y^{(k)} \cond x^{(k)}_t=x^{(k)}\bigr)
=\prod_{k=1}^d \Q_{t+\Delta t\cond t}^{(k)}\!\bigl(y^{(k)},x^{(k)}\bigr)
\end{align}
which in matrix form reads $\Qseq_{t+\Delta t \cond t} = \bigotimes_{k=1}^{d} \Q_{t+\Delta t \cond t}^{(k)}$.
Each per-coordinate transition matrix admits the infinitesimal expansion:
\begin{equation}
\Q_{t+\Delta t \cond t}^{(k)} = \I + \Delta t\,\R_t^{(k)} + o(\Delta t).
\end{equation}
Substituting in the Kronecker product yields:
\begin{align}
    \Qseq_{t+\Delta t\cond t}
 = 
\bigotimes_{k=1}^{d}\!\Bigl(I+\Delta t\,\R_t^{(k)}+o(\Delta t)\Bigr).
\end{align}
Expanding to first order using $(A+B)\otimes C = A\otimes C + B\otimes C$ and
$(\I+\Delta t A)\otimes(\I+\Delta t B)=\I\otimes \I+\Delta t(A\otimes \I+\I\otimes B)+O(\Delta t^2)$,
we obtain
\begin{align}
\Qseq_{t+\Delta t\cond t}
 = 
\I
\;+\;
\Delta t\,\sum_{k=1}^{d}\Bigl(\I^{\otimes(k-1)}\otimes \R_t^{(k)}\otimes \I^{\otimes(d-k)}\Bigr)
\;+\;
O(\Delta t^2).
\end{align}
By definition of the generator, $\Rseq_t  =  \lim_{\Delta t\to0}\frac{\Qseq_{t+\Delta t\cond t}-I}{\Delta t}$, hence the claimed Kronecker–sum form.
\end{proof}

\paragraph{2-dimensional Rate Matrix.}
Take
\(
\R_t^{(1)}=\begin{bmatrix}a_{11}&a_{12}\\ a_{21}&a_{22}\end{bmatrix},
\;
\R_t^{(2)}=\begin{bmatrix}b_{11}&b_{12}\\ b_{21}&b_{22}\end{bmatrix}.
\)
Then (indexing joint states as $(1,1),(1,2),(2,1),(2,2)$):
\begin{align*}
\R_t^{(1)}\otimes \I + \I\otimes \R_t^{(2)}
&=
\begin{bmatrix}
a_{11} & 0 & a_{12} & 0\\
0 & a_{11} & 0 & a_{12}\\
a_{21} & 0 & a_{22} & 0\\
0 & a_{21} & 0 & a_{22}
\end{bmatrix}
\;+\;
\begin{bmatrix}
b_{11} & b_{12} & 0 & 0\\
b_{21} & b_{22} & 0 & 0\\
0 & 0 & b_{11} & b_{12}\\
0 & 0 & b_{21} & b_{22}
\end{bmatrix}
\\&=
\begin{bmatrix}
a_{11}{+}b_{11} & b_{12} & a_{12} & \color{red}{0}\\
b_{21} & a_{11}{+}b_{22} & \color{red}{0} & a_{12}\\
a_{21} & \color{red}{0} & a_{22}{+}b_{11} & b_{12}\\
\color{red}{0} & a_{21} & b_{21} & a_{22}{+}b_{22}
\end{bmatrix}.
\end{align*}
Entrywise, for general sizes,
\begin{align}
\bigl[\R_t^{(1)}\otimes \I + \I\otimes \R_t^{(2)}\bigr]_{(i',j'),(i,j)}
=
\R_t^{(1)}(i',i)\,\delta_{j'j}
\;+\;
\delta_{i'i}\,\R_t^{(2)}(j',j),
\end{align}
which shows that only one coordinate changes at first order (red entries in the developed matrix). 

\begin{corollary}[Entrywise characterisation]\label{cor:rate-matrix-entries}
Let $\x = (x^{(1)}, \ldots, x^{(d)})$ and $\y = (y^{(1)}, \ldots, y^{(d)})$ be two states in $\calX^d$. The entries of $\Rseq_t$ satisfy:
\begin{enumerate}[label=(\roman*)]
    \item \emph{Single-coordinate transitions:} If $\y$ differs from $\x$ in exactly one coordinate $k$, then
    \begin{equation}
    \Rseq_t(\y, \x) = \R_t^{(k)}(y^{(k)}, x^{(k)}).
    \end{equation}
    
    \item \emph{Multi-coordinate transitions:} If $\y$ differs from $\x$ in two or more coordinates, then
    \begin{equation}
    \Rseq_t(\y, \x) = 0.
    \end{equation}
    
    \item \emph{Diagonal entries:} The column-sum-to-zero property implies
    \begin{equation}
    \Rseq_t(\x, \x) = -\sum_{\y \neq \x} \Rseq_t(\y, \x) = -\sum_{k=1}^{d} \sum_{y^{(k)} \neq x^{(k)}} \R_t^{(k)}(y^{(k)}, x^{(k)}).
    \end{equation}
\end{enumerate}
\end{corollary}

\begin{remark}[Computational implications]
The Kronecker sum structure has important computational consequences. Although $\Rseq_t$ has dimension $K^d \times K^d$, it is determined entirely by the $d$ matrices $\{\R_t^{(k)}\}_{k=1}^d$, each of size $K \times K$. This is essential for practical implementations in high-dimensional discrete diffusion models.
\end{remark}

\section{Introduction to the infinitesimal generator of a Markov process}

\subsection{Deriving the reverse generator\label{apx:reverse-generator}}

Consider a (potentially) time-inhomogeneous Markov process $\{\x_t\}_{t\in[0,1]}$ on a state space $\calX^d$ with forward marginal distributions $\{q_t\}_{t\in[0,1]}$. The \emph{(spatial) infinitesimal generator} $\Lgen_t$ quantifies the instantaneous rate of change of expectations under the process dynamics. For a test function $\phi : \calX^d \rightarrow \RR$ in the domain $\mathcal{D}(\Lgen_t)$, it is defined as:
\begin{align}
\Lgen_t\phi(\x) \;=\; \lim_{\Delta t\to0^+}\frac{\E \left[\phi\big(\x_{t+\Delta t}\big) \,\big|\, \x_t=\x\right]-\phi(\x)}{\Delta t}.
\end{align}

The evolution of marginal distributions is governed by the adjoint operator $(\Lgen_t)^*$ through the Kolmogorov forward equation (KFE):
\begin{equation}
\partial_t q_t \;=\; (\Lgen_t)^* q_t,
\qquad
q_0=\qdata,
\end{equation}
where the adjoint satisfies $\langle \Lgen_t f,\, g \rangle = \langle f,\, (\Lgen_t)^* g \rangle$ for the inner product $\langle f,g \rangle = \int f(\x) \, g(\x) \, \diff \mu(\x)$, where $\mu$ in the Lebesgue measure in continuous state space and the counting measure in discrete state space.

\paragraph{Time reversal setup.} For time reversal, we define the reversed process $\hat{\x}_s \coloneqq \x_{1-s}$ and reversed marginals $\hat{q}_s \coloneqq q_{1-s}$ for $s \in [0,1]$. By change of variable and chain rule:
\begin{align}
\partial_s \hat{q}_s(\x)
\;=\; -\,(\Lgen_{1-s})^{*}\,\hat{q}_s(\x).
\label{eq:time-reversal-marginal-apx}
\end{align}

However, \cref{eq:time-reversal-marginal-apx} cannot directly be used for simulation because $(\Lgen_{1-s})$ conditions on the future when expressed in reversed time:
\begin{equation}
\Lgen_{1-s} \phi(\hat{\x}_s) \;=\; \lim_{\Delta s\to0^+}\frac{\E_{q(\hat{\x}_{s-\Delta s} \cond \hat{\x}_s)}\left[\phi(\hat{\x}_{s-\Delta s})\right]-\phi(\hat{\x}_s)}{\Delta s}.
\end{equation}
To obtain a tractable forward-in-time description of the reversed process, we must derive an alternative generator that conditions on the past in reversed time.

\begin{proposition}[Reversed Generator]\label{prop:reversed-generator}
The reversed process $\{\hat{\x}_s\}_{s\in[0,1]}$, defined by $\hat{\x}_s \coloneqq \x_{1-s}$, is Markovian with infinitesimal generator:
\begin{equation}\label{eq:reversed-generator-formula-apx}
\hat{\Lgen}_s \phi(\x)
\;=\; \frac{1}{q_{1-s}(\x)}\Bigl((\Lgen_{1-s})^*(q_{1-s} \phi)(\x) - \phi(\x)\,(\Lgen_{1-s})^* q_{1-s}(\x) \Bigr),
\end{equation}
where $(q_{1-s} \phi) : \x \mapsto q_{1-s}(\x)\,\phi(\x)$ denotes the pointwise product, and $(\Lgen_{1-s})^*$ acts on this product as a single function. This generator satisfies the standard forward-in-$s$ definition:
\begin{equation}
\hat{\Lgen}_s \phi(\hat \x_s) \;=\; \lim_{\Delta s\to0^+}\frac{\E_{q(\hat{\x}_{s+\Delta s} \cond \hat{\x}_s)}\left[\phi(\hat{\x}_{s+\Delta s})\right]-\phi(\hat \x_s)}{\Delta s}.
\end{equation}
\end{proposition}

\begin{proof}
Starting from the definition \eqref{eq:reversed-generator-def} and using $\hat{\x}_s = \x_{1-s}$:
\begin{align}
\hat{\Lgen}_s \phi(\x) 
&= \lim_{\Delta s\to0^+}\frac{\E\left[\phi(\hat{\x}_{s+\Delta s}) \,\big|\, \hat{\x}_s=\x\right]-\phi(\x)}{\Delta s} \nonumber\\
&= \lim_{\Delta s\to0^+}\frac{\E\left[\phi(\x_{1-s-\Delta s}) \,\big|\, \x_{1-s}=\x\right]-\phi(\x)}{\Delta s}.\label{eq:rev-gen-step1}
\end{align}

The key challenge is that the conditioning is in the ``wrong'' temporal direction: we know $\x_{1-s}$ but want to compute expectations about $\x_{1-s-\Delta s}$ in the past. We reverse this conditioning using Bayes' rule. Therefore:
\begin{align}
\E\left[\phi(\x_{1-s-\Delta s}) \,\big|\, \x_{1-s}=\x\right] &= \int \phi(\y) \, q(\x_{1-s-\Delta s}=\y \cond \x_{1-s}=\x) \, \diff\mu(\y) \nonumber\\
&= \int \phi(\y) \, \frac{q(\x_{1-s}=\x \cond \x_{1-s-\Delta s}=\y) \, q_{1-s-\Delta s}(\y)}{q_{1-s}(\x)} \, \diff\mu(\y).\label{eq:cond-exp-reversed}
\end{align}

Define the forward propagator $P_{s \to t} $ by:
\begin{align}
P_{s \to t} \phi(\x) &= \E[\phi(\x_t) \cond \x_s = \x]
\end{align}

For small increments, using the definition of the generator \cref{eq:gen-def} we have the standard expansion
\begin{align*}
(P_{t\to t+\Delta t}\phi)(\y)  =  \phi(\y) \;+\; \Delta t\,(\Lgen_t \phi)(\y) \;+\; o(\Delta t),
\end{align*}
hence, with $t=1-s-\Delta s$,
\begin{align*}
P_{1-s-\Delta s \to 1-s}  =  I \;+\; \Delta s\,\Lgen_{1-s} \;+\; o(\Delta s),
\end{align*}
where $I$ is the identity on $\mathcal{D}(\Lgen_t)$ and we assume here $t\mapsto\Lgen_t$ is continuous and replace $\Lgen_{1-s-\Delta s}$ by $\Lgen_{1-s}$ without changing the $o(\Delta s)$ remainder.

We express the forward propagator using the transition kernel:
\begin{align}
   (P_{1-s-\Delta s \to 1-s}  \phi)(\y) &= \int \phi(\x) q(\x_{1-s} =\x \cond \x_{1-s-\Delta s}=\y) \diff \mu(\x) \\
    \phi(\y) + \Delta s \; (\Lgen_{1-s} \phi)(\y)+ o(\Delta s)&= \int \phi(\x) q(\x_{1-s} =\x \cond \x_{1-s-\Delta s}=\y) \diff \mu(\x)
    \label{eq:forward-propagator-expansion-equal-kernel-integral}
\end{align}
Interpreting in distribution in the $\x$ variable the LHS:
\begin{align}
\int \phi(\x) [ \delta_{\y}(\x) + \Delta s \Lgen^\star_{1-s} \delta_{\y}(\x) + o(\Delta s)] \diff \mu(\x) = \int \phi(\x) q(\x_{1-s} =\x \cond \x_{1-s-\Delta s}=\y) \diff \mu(\x)
\label{eq:distributional-push-forward-expansion}
\end{align}
By equality between \cref{eq:forward-propagator-expansion-equal-kernel-integral} and \cref{eq:distributional-push-forward-expansion}, and applying the fundamental lemma of calculus of variations, we have:
\begin{align}
    q(\x_{1-s} =\x \cond \x_{1-s-\Delta s}=\y) = \delta_{\y}(\x) + \Delta s \Lgen^\star_{1-s} \delta_{\y}(\x) + o(\Delta s)
    \label{eq:transition-expansion}
\end{align}

Substituting \eqref{eq:transition-expansion} into \eqref{eq:cond-exp-reversed}:
\begin{align}
&\E\left[\phi(\x_{1-s-\Delta s}) \,\big|\, \x_{1-s}=\x\right] \nonumber\\
&= \frac{1}{q_{1-s}(\x)} \int \phi(\y) \Bigl[\delta_{\y}(\x) + \Delta s \cdot [(\Lgen_{1-s})^* \delta_{\y}](\x)\Bigr] q_{1-s-\Delta s}(\y) \, \diff\mu(\y) + o(\Delta s) \nonumber\\
&= \frac{1}{q_{1-s}(\x)} \Bigg[\phi(\x) q_{1-s}(\x) - \Delta s\, \phi(\x)\, [(\Lgen_{1-s})^\star q_{1-s}](\x)
+ \Delta s \int \phi(\y) \big[(\Lgen_{1-s})^* \delta_{\y}\big](\x) \, q_{1-s}(\y) \, \diff\mu(\y)
\Bigg] + o(\Delta s),
\label{eq:exp-expanded}
\end{align}
where in the last line we used $q_{1-s-\Delta s}(\x) = q_{1-s}(\x) - \Delta s \Lgen_{1-s}^\star q_{1-s}(\x)+o(\Delta s)$.

By Fubini’s theorem and the adjoint identity, for any test $\psi$, we have $\int \psi(\x)\,[(\Lgen_{1-s})^*\delta_{\y}](\x)\,\diff\mu(\x)=(\Lgen_{1-s}\psi)(\y)$:
\begin{align*}
\int \psi(\x)\!\left[\int \phi(\y)\, q_{1-s}(\y)\, [(\Lgen_{1-s})^* \delta_{\y}](\x)\, \diff\mu(\y)\right]\diff\mu(\x)
&= \int \phi(\y)\, q_{1-s}(\y)\, (\Lgen_{1-s}\psi)(\y)\, \diff\mu(\y)\\
&= \langle q_{1-s}\phi,\, \Lgen_{1-s}\psi \rangle 
\;=\; \langle (\Lgen_{1-s})^*(q_{1-s}\phi),\, \psi \rangle\\
&= \int \psi(\x)\, [(\Lgen_{1-s})^*(q_{1-s}\phi)](\x)\, \diff\mu(\x).
\end{align*}
By the fundamental lemma of calculus of variations, we recover:
\begin{align}
\int \phi(\y)\, q_{1-s}(\y)\, [(\Lgen_{1-s})^* \delta_{\y}](\x)\, \diff\mu(\y)
&= [(\Lgen_{1-s})^* (q_{1-s}\phi)](\x).
\label{eq:adjoint-application}
\end{align}
Substituting \eqref{eq:adjoint-application} into \eqref{eq:exp-expanded}:
\begin{align}
\E\left[\phi(\x_{1-s-\Delta s}) \,\big|\, \x_{1-s}=\x\right] 
&= \phi(\x) + \frac{\Delta s}{q_{1-s}(\x)} \Big( [(\Lgen_{1-s})^\star (q_{1-s} \cdot \phi)](\x) - \phi(\x)\, [(\Lgen_{1-s})^\star q_{1-s}](\x) \Big) + o(\Delta s).
\end{align}
Returning to \eqref{eq:rev-gen-step1}:
\begin{align}
\hat{\Lgen}_s \phi(\x) 
&= \lim_{\Delta s\to0^+}\frac{1}{\Delta s}\left[\E\left[\phi(\x_{1-s-\Delta s}) \,\big|\, \x_{1-s}=\x\right] - \phi(\x)\right] \nonumber\\
&= \frac{1}{q_{1-s}(\x)} \Big( [(\Lgen_{1-s})^\star(q_{1-s} \cdot \phi)](\x) - \phi(\x)\, [(\Lgen_{1-s})^\star q_{1-s}](\x) \Big).
\end{align}

This establishes \eqref{eq:reversed-generator}. 

\end{proof}

\subsection{Decomposition of the time-inhomogeneous infinitesimal generator\label{apx:time-inh-generator-decomposition}}

For a time-inhomogeneous Markov process $\{\x_t\}_{t \in [0,1]}$ on state space $\calX^d$, we derive the decomposition of the extended generator into temporal and spatial components.

\begin{proposition}[Extended generator decomposition]\label{prop:extended-generator-decomp}
Let $\{\x_t\}_{t\in [0,1]}$ be a time-inhomogeneous Markov process on $\calX^d$ with spatial infinitesimal generator $\Lgen_t$. For a time-dependent test function $\phi_t : \calX^d \to \RR$ (equivalently $\phi : \calX^d \times [0,1] \to \RR$ with $\phi_t(\x) \coloneqq \phi(\x,t)$), the extended generator admits the decomposition:
\begin{equation}\label{eq:time_inhomog_generator}
\Ggen_t \phi_t(\x) \;=\; \partial_t \phi_t(\x) \;+\; \Lgen_t \phi_t(\x),
\end{equation}
where $\Lgen_t$ is the spatial generator acting on $\phi_t$ for fixed $t$.
\end{proposition}

\begin{proof}
For any test function $\phi: \calX^d \times [0,1] \rightarrow \RR$, in $\calD(\Lgen_t)$ the following limit is well defined:

\begin{align}
\Ggen_t \phi_t(\x)=\lim _{\Delta t \rightarrow 0} \frac{\E\left[\phi_{t+\Delta t}\left(\x_{t+\Delta t}\right)-\phi_t\left(\x_t\right) \cond \x_t=\x\right]}{\Delta t} .
\end{align}

We decompose the numerator by adding and subtracting $\phi_{t+\Delta t}(\x_t)$:

\begin{align}
\Ggen_t \phi_t(\x)=\lim _{\Delta t \rightarrow 0} \frac{\E\left[\phi_{t+\Delta t}\left(\x_{t+\Delta t}\right)-\phi_{t+\Delta t}\left(\x_{t}\right)\cond \x_t=\x\right] + \E\left[ \phi_{t+\Delta t}\left(\x_{t}\right)-\phi_t\left(\x_t\right) \cond \x_t=\x\right]}{\Delta t} .
\end{align}

Since $\x_t = \x$ is given, the second expectation simplifies to:
$$\mathbb{E}\left[\phi_{t+\Delta t}(\x_t) - \phi_t(\x_t) \cond \x_t = \x\right] = \phi_{t+\Delta t}(\x) - \phi_t(\x).$$
Therefore:
$$\Ggen_t \phi_t(\x) = \lim_{\Delta t \rightarrow 0^+} \frac{\mathbb{E}\left[\phi_{t+\Delta t}(\x_{t+\Delta t}) - \phi_{t+\Delta t}(\x_t) \cond \x_t = \x\right]}{\Delta t} + \lim_{\Delta t \rightarrow 0^+} \frac{\phi_{t+\Delta t}(\x) - \phi_t(\x)}{\Delta t}.$$

The first term is precisely the spatial generator $\Lgen_t$ applied to $\phi_{t+\Delta t}$:
$$\lim_{\Delta t \rightarrow 0^+} \frac{\mathbb{E}\left[\phi_{t+\Delta t}(\x_{t+\Delta t}) - \phi_{t+\Delta t}(\x_t) \cond \x_t = \x\right]}{\Delta t} = \Lgen_t \phi_t(\x),$$

where we used the continuity of $\phi$ in time to pass from $\phi_{t+\Delta t}$ to $\phi_t$ in the limit, where $\Lgen_t$ is defined as:
\begin{align}
\Lgen_t \phi_t(\x)\coloneqq\lim _{\Delta t \rightarrow 0} \frac{\E\left[\phi_{t}\left(\x_{t+\Delta t}\right)-\phi_t\left(\x_t\right) \cond \x_t=\x\right]}{\Delta t} 
\end{align}

The second term is the partial time derivative:
$$\lim_{\Delta t \rightarrow 0^+} \frac{\phi_{t+\Delta t}(\x) - \phi_t(\x)}{\Delta t} = \partial_t \phi_t(\x).$$
Combining both terms yields:
$$\Ggen_t \phi_t(\x) = \Lgen_t\phi_t(\x) + \partial_t \phi_t(\x) = \left(\partial_t + \Lgen_t\right) \phi_t(\x).$$
This establishes the decomposition.
\end{proof}

\subsection{Derivation of the Kolmogorov forward equation \label{apx:KFE}}

\begin{proposition}[Kolmogorov Forward Equation]
Let $\{\x_t\}_{t \in [0,1]}$ be a (potentially) time-inhomogeneous Markov process on state space $\calX^d$ with spatial infinitesimal generator $\Lgen_t$. The marginal distributions $\{q_t\}_{t\in[0,1]}$ satisfy:
\begin{equation}\label{eq:KFE-apx}
\partial_t q_t = (\Lgen_t)^* q_t,
\end{equation}
where $(\Lgen_t)^*$ denotes the $L^2$-adjoint of $\Lgen_t$.
\end{proposition}

\begin{proof}
The proof proceeds in three steps: (i) apply Dynkin's formula, (ii) differentiate in time, and (iii) invoke the adjoint definition and the fundamental lemma of calculus of variations.

\paragraph{Step 1: Dynkin's formula.}
For a time-independent test function $\phi : \calX^d \to \RR$ in the domain $\mathcal{D}(\Lgen_t)$, Dynkin's formula states that for any $t \in [0,1]$ and initial condition $\x_0 = \x$:
\begin{equation}
\E_{q(\x_t \cond \x_0=\x)}\left[\phi(\x_t) \right] = \phi(\x) + \E\left[\int_0^t \Lgen_s \phi(\x_s) \diff s \cond \x_0 = \x \right].
\label{eq:dynkin-KFE}
\end{equation}

Expressing the expectations as integrals against the transition density $q(\x_t \cond \x_0 = \x)$:
\begin{equation}
\int_{\calX^d} \phi(\y) \, q_t(\y \cond \x_0 = \x) \,\diff\mu(\y) = \phi(\x) + \int_0^t \int_{\calX^d} \Lgen_s \phi(\z) \, q_s(\z \cond \x_0 = \x) \,\diff\mu(\z) \,\diff s,
\label{eq:dynkin-integral-form}
\end{equation}
where we have applied Fubini's theorem to exchange the order of integration (valid under standard regularity conditions on $\phi$ and $q_s$). In discrete state spaces, $\diff\mu$ denotes counting measure and integrals become sums. In continuous state spaces with Lebesgue measure, $\diff\mu(\x) = \diff\x$.

\paragraph{Step 2: Time differentiation.}
Differentiating both sides of \cref{eq:dynkin-integral-form} with respect to $t$: on the left-hand side, we exchange differentiation and integration (justified by the dominated convergence theorem under standard regularity conditions on $q_t$); on the right-hand side, we apply the fundamental theorem of calculus:
\begin{equation}
\int_{\calX^d} \phi(\y) \, \partial_t q_t(\y \cond \x_0 = \x) \,\diff\mu(\y) = \int_{\calX^d} \Lgen_t \phi(\y) \, q_t(\y \cond \x_0 = \x) \,\diff\mu(\y).
\label{eq:differentiated-dynkin}
\end{equation}

\paragraph{Step 3: Adjoint and fundamental lemma.}
Rewriting \cref{eq:differentiated-dynkin} in inner product notation on $L^2(\calX^d, \mu)$:
\begin{equation}
\bigl\langle \phi,\, \partial_t q_t(\cdot \cond \x_0 = \x) \bigr\rangle = \bigl\langle \Lgen_t \phi,\, q_t(\cdot \cond \x_0 = \x) \bigr\rangle.
\end{equation}

By definition of the $L^2$-adjoint $(\Lgen_t)^*$:
\begin{equation}
\bigl\langle \phi,\, \partial_t q_t(\cdot \cond \x_0 = \x) \bigr\rangle = \bigl\langle \phi,\, (\Lgen_t)^* q_t(\cdot \cond \x_0 = \x) \bigr\rangle.
\end{equation}

Since this identity holds for all test functions $\phi$ in a dense subset of $L^2(\calX^d, \mu)$,\footnote{We assume that $\mathcal{D}(\Lgen_t) \cap L^2(\calX^d, \mu)$ is dense in $L^2(\calX^d, \mu)$.} the fundamental lemma of calculus of variations implies:
\begin{equation}
\partial_t q_t(\cdot \cond \x_0 = \x) = (\Lgen_t)^* q_t(\cdot \cond \x_0 = \x).
\label{eq:KFE-transition}
\end{equation}

\paragraph{Extension to marginals.}
Integrating \cref{eq:KFE-transition} over the initial distribution $\x_0 \sim \qdata$:
\begin{align}
\partial_t q_t(\y) 
&= \partial_t \int_{\calX^d} q_t(\y \cond \x_0 = \x) \, \qdata(\x) \,\diff\mu(\x) \notag\\
&= \int_{\calX^d} (\Lgen_t)^* q_t(\y \cond \x_0 = \x) \, \qdata(\x) \,\diff\mu(\x) \notag\\
&= (\Lgen_t)^* \int_{\calX^d} q_t(\y \cond \x_0 = \x) \, \qdata(\x) \,\diff\mu(\x) \notag\\
&= (\Lgen_t)^* q_t(\y),
\end{align}
where the third equality uses the linearity of $(\Lgen_t)^*$ (acting on the $\y$ variable) and Fubini's theorem.
\end{proof}

\begin{remark}[Kolmogorov backward equation]
An analogous derivation, differentiating with respect to the \emph{initial} time $s$ rather than the terminal time $t$, yields the Kolmogorov backward equation:
\begin{equation}
\partial_s q(\x_t = \y \cond \x_s = \x) = -\Lgen_s q(\x_t = \y \cond \x_s = \x),
\end{equation}
where $\Lgen_s$ acts on the initial variable $\x$. Note that this describes backward differentiation in time, not time reversal of the process itself.
\end{remark}

\subsection{Evaluation of the infinitesimal generator}\label{apx:derivation-infinitesimal-generator}

\subsubsection{Continuous state space}

We derive the closed-form expression for the infinitesimal generator of a diffusion process in continuous state space. For clarity, we present the one-dimensional case; the extension to $\RR^d$ follows by applying the multivariate Itô formula.

\begin{proposition}[Infinitesimal generator for diffusion processes]\label{prop:generator-continuous}
Let $\{x_t\}_{t \in [0,1]}$ with $x_t \in \RR$ solve the Itô SDE:
\begin{equation}
\diff x_t = f_t(x_t)\,\diff t + g_t\,\diff w_t,
\end{equation}
where $w_t$ is a standard Brownian motion, $f_t : \RR\times \RR \to \RR$ is the drift, and $g_t \in \RR^+$ is the diffusion coefficient.

\begin{enumerate}[label=(\roman*)]
\item The \emph{spatial infinitesimal generator} $\Lgen_t$, acting on time-independent test functions $\phi \in \calD(\Lgen_t)$, is:
\begin{equation}\label{eq:spatial-generator-1d}
\Lgen_t \phi(x) = f_t(x)\,\partial_x \phi(x) + \frac{1}{2}g_t^2\,\partial_x^2 \phi(x).
\end{equation}

\item The \emph{extended generator} $\Ggen_t$, acting on time-dependent test functions $\phi_t \in \calD(\Ggen_t)$, is:
\begin{equation}\label{eq:extended-generator-1d}
\Ggen_t \phi_t(x) = \partial_t \phi_t(x) + \Lgen_t \phi_t(x) = \partial_t \phi_t(x) + f_t(x)\,\partial_x \phi_t(x) + \frac{1}{2}g_t^2\,\partial_x^2 \phi_t(x).
\end{equation}
\end{enumerate}
\end{proposition}

\begin{proof}
Let $\phi_t \in C^{1,2}([0,1] \times \RR)$. By Itô's lemma applied to $\phi_t(x_t)$ \cite{oksendal2003stochastic}:
\begin{equation}
\diff \phi_t(x_t) = \partial_t \phi_t(x_t)\,\diff t + \partial_x \phi_t(x_t)\,\diff x_t + \frac{1}{2}\partial_x^2 \phi_t(x_t)\,(\diff x_t)^2.
\end{equation}

Using the Itô calculus rules $(\diff t)^2 = 0$, $\diff t \cdot \diff w_t = 0$, and $(\diff w_t)^2 = \diff t$, we compute:
\begin{equation}
(\diff x_t)^2 = \bigl(f_t(x_t)\,\diff t + g_t\,\diff w_t\bigr)^2 = g_t^2\,\diff t + o(\diff t).
\end{equation}

Substituting and collecting terms:
\begin{align}
\diff \phi_t(x_t) 
&= \partial_t \phi_t(x_t)\,\diff t + \partial_x \phi_t(x_t)\bigl[f_t(x_t)\,\diff t + g_t\,\diff w_t\bigr] + \frac{1}{2}\partial_x^2 \phi_t(x_t)\,g_t^2\,\diff t \notag\\
&= \underbrace{\Bigl[\partial_t \phi_t(x_t) + f_t(x_t)\,\partial_x \phi_t(x_t) + \frac{1}{2}g_t^2\,\partial_x^2 \phi_t(x_t)\Bigr]}_{\text{drift terms}}\,\diff t 
+ \underbrace{g_t\,\partial_x \phi_t(x_t)\,\diff w_t}_{\text{martingale term}}.
\label{eq:ito-expansion}
\end{align}

Taking the conditional expectation given $x_t = x$:
\begin{equation}
\E\bigl[\diff \phi_t(x_t) \,\big|\, x_t = x\bigr] 
= \Bigl[\partial_t \phi_t(x) + f_t(x)\,\partial_x \phi_t(x) + \frac{1}{2}g_t^2\,\partial_x^2 \phi_t(x)\Bigr]\,\diff t,
\end{equation}
since the martingale term has zero conditional expectation: $\E[\diff w_t \cond x_t = x] = 0$.

By definition of the extended generator:
\begin{equation}
\Ggen_t \phi_t(x) \coloneqq \lim_{\Delta t \to 0^+} \frac{\E\bigl[\phi_{t+\Delta t}(x_{t+\Delta t}) - \phi_t(x_t) \,\big|\, x_t = x\bigr]}{\Delta t} = \partial_t \phi_t(x) + \Lgen_t \phi_t(x),
\end{equation}
which establishes \eqref{eq:extended-generator-1d}. For time-independent test functions ($\partial_t \phi = 0$), the extended generator reduces to the spatial generator \eqref{eq:spatial-generator-1d}.
\end{proof}

\begin{proposition}[$L^2$-adjoint of the spatial generator]\label{prop:adjoint-continuous}
Let $\Lgen_t$ be the spatial generator \eqref{eq:spatial-generator-1d}. Its $L^2$-adjoint $(\Lgen_t)^*$, defined via
\begin{equation}
\int_\RR \phi(x)\,(\Lgen_t \psi)(x)\,\diff x = \int_\RR ((\Lgen_t)^* \phi)(x)\,\psi(x)\,\diff x
\quad \text{for all } \phi, \psi \in \calD({\Lgen_t})\cap C_0^\infty(\RR),
\end{equation}
is given by:
\begin{equation}\label{eq:adjoint-generator-1d}
(\Lgen_t)^* \phi(x) = -\partial_x\bigl[f_t(x)\,\phi(x)\bigr] + \frac{1}{2}\partial_x^2\bigl[g_t^2\,\phi(x)\bigr].
\end{equation}
\end{proposition}

\begin{proof}
We compute the adjoint by integration by parts, assuming test functions vanish at infinity so that boundary terms vanish.

\paragraph{Drift term.} Let $A\psi(x) \coloneqq f_t(x)\,\partial_x \psi(x)$. Then:
\begin{align}
\int_\RR \phi(x)\,(A\psi)(x)\,\diff x
&= \int_\RR \phi(x)\,f_t(x)\,\partial_x \psi(x)\,\diff x \notag\\
&= -\int_\RR \partial_x\bigl[\phi(x)\,f_t(x)\bigr]\,\psi(x)\,\diff x,
\end{align}
where we integrated by parts once. Thus:
\begin{equation}
A^* \phi(x) = -\partial_x\bigl[f_t(x)\,\phi(x)\bigr].
\end{equation}

\paragraph{Diffusion term.} Let $B\psi(x) \coloneqq \frac{1}{2}g_t^2\,\partial_x^2 \psi(x)$. Integrating by parts twice:
\begin{align}
\int_\RR \phi(x)\,(B\psi)(x)\,\diff x
&= \frac{1}{2}\int_\RR \phi(x)\,g_t^2\,\partial_x^2 \psi(x)\,\diff x \notag\\
&= -\frac{1}{2}\int_\RR \partial_x\bigl[\phi(x)\,g_t^2\bigr]\,\partial_x \psi(x)\,\diff x \notag\\
&= \frac{1}{2}\int_\RR \partial_x^2\bigl[\phi(x)\,g_t^2\bigr]\,\psi(x)\,\diff x.
\end{align}
Thus:
\begin{equation}
B^* \phi(x) = \frac{1}{2}\partial_x^2\bigl[g_t^2\,\phi(x)\bigr].
\end{equation}

\paragraph{Combined adjoint.} By linearity of adjoints:
\begin{equation}
(\Lgen_t)^* = A^* + B^* \quad \Rightarrow \quad (\Lgen_t)^* \phi(x) = -\partial_x\bigl[f_t(x)\,\phi(x)\bigr] + \frac{1}{2}\partial_x^2\bigl[g_t^2\,\phi(x)\bigr].
\end{equation}
\end{proof}

\begin{remark}[Extension to $\RR^d$]
For $\x_t \in \RR^d$ with drift $\f_t : \RR^d \to \RR^d$ and diffusion matrix $\Sigmab_t : \RR^d \to \RR^{d \times d}$, the spatial generator and its adjoint become:
\begin{align}
\Lgen_t \phi(\x) &= \f_t(\x)^\top \nabla_\x \phi(\x) + \frac{1}{2}\mathrm{tr}\bigl(\Sigmab_t(\x)\Sigmab_t(\x)^\top \nabla_\x^2 \phi(\x)\bigr), \\
(\Lgen_t)^* \phi(\x) &= -\nabla_\x \cdot \bigl[\f_t(\x)\,\phi(\x)\bigr] + \frac{1}{2}\sum_{i,j} \partial_{x_i}\partial_{x_j}\bigl[(\Sigmab_t \Sigmab_t^\top)_{ij}(\x)\,\phi(\x)\bigr].
\end{align}
For scalar diffusion coefficient $g_t$ (as in the main text), this simplifies to:
\begin{equation}
(\Lgen_t)^* \phi(\x) = -\nabla_\x \cdot \bigl[\f_t(\x)\,\phi(\x)\bigr] + \frac{1}{2}g_t^2\,\Delta_\x \phi(\x),
\end{equation}
recovering the Fokker–Planck operator in \cref{eq:KFE-continuous-sp}.
\end{remark}

\subsubsection{Discrete state space\label{apx:discrete-infinitesimal-generator}}

We derive the infinitesimal generator and its adjoint for continuous-time Markov chains (CTMCs) on finite state spaces, establishing the connection to rate matrices.

\begin{proposition}[Generator for CTMCs]\label{prop:ctmc-generator}
Let $\{\x_t\}_{t\in[0,1]}$ be a CTMC on the finite state space $\calX^d$ with $|\calX| = K$, governed by a time-dependent rate matrix $\Rseq_t \in \RR^{K^d \times K^d}$ under our ML convention (destination state = row index, source state = column index). The spatial infinitesimal generator $\Lgen_t$ acts on test functions $\phi : \calX^d \to \RR$ as:
\begin{equation}\label{eq:ctmc-generator-formula}
\Lgen_t \phi(\x) = \sum_{\y \in \calX^d} \Rseq_t(\y, \x) \, \phi(\y).
\end{equation}
In matrix-vector notation, with $\phib \coloneqq (\phi(\x))_{\x \in \calX^d} \in \RR^{K^d}$:
\begin{equation}
\Lgen_t \phib = \Rseq_t^\top \phib, \qquad \text{i.e.,} \quad \Lgen_t = \Rseq_t^\top.
\end{equation}
\end{proposition}

\begin{proof}
By definition of the spatial infinitesimal generator (cf.\ \cref{apx:time-inh-generator-decomposition}):
\begin{equation}
\Lgen_t \phi(\x) = \lim_{\Delta t \to 0^+} \frac{\E\left[\phi(\x_{t+\Delta t}) \,\big|\, \x_t = \x\right] - \phi(\x)}{\Delta t}.
\end{equation}

Expanding the conditional expectation:
\begin{equation}
\E\left[\phi(\x_{t+\Delta t}) \,\big|\, \x_t = \x\right] = \sum_{\y \in \calX^d} q(\x_{t+\Delta t} = \y \cond \x_t = \x) \, \phi(\y).
\end{equation}

From the definition of the rate matrix (\cref{eq:def-transition-rate-matrix}):
\begin{equation}
q(\x_{t+\Delta t} = \y \cond \x_t = \x) = \delta_{\x,\y} + \Rseq_t(\y, \x) \, \Delta t + o(\Delta t).
\end{equation}

Substituting and separating the diagonal ($\y = \x$) and off-diagonal ($\y \neq \x$) terms:
\begin{align}
\E\left[\phi(\x_{t+\Delta t}) \,\big|\, \x_t = \x\right]
&= \sum_{\y \in \calX^d} \left[\delta_{\x,\y} + \Rseq_t(\y, \x) \, \Delta t + o(\Delta t)\right] \phi(\y) \notag\\
&= \phi(\x) + \Rseq_t(\x, \x) \, \Delta t \, \phi(\x) + \sum_{\y \neq \x} \Rseq_t(\y, \x) \, \Delta t \, \phi(\y) + o(\Delta t) \notag\\
&= \phi(\x) + \Delta t \sum_{\y \in \calX^d} \Rseq_t(\y, \x) \, \phi(\y) + o(\Delta t).
\end{align}

Taking the limit:
\begin{equation}
\Lgen_t \phi(\x) = \lim_{\Delta t \to 0^+} \frac{\Delta t \sum_{\y \in \calX^d} \Rseq_t(\y, \x) \, \phi(\y) + o(\Delta t)}{\Delta t} = \sum_{\y \in \calX^d} \Rseq_t(\y, \x) \, \phi(\y).
\end{equation}

In matrix notation, this reads $(\Lgen_t \phib)(\x) = \sum_{\y} \Rseq_t(\y, \x) \, \phi(\y) = (\Rseq_t^\top \phib)(\x)$.
\end{proof}

\begin{proposition}[Adjoint generator for CTMCs]\label{prop:ctmc-adjoint}
The $L^2$-adjoint of $\Lgen_t$ with respect to counting measure on $\calX^d$ is:
\begin{equation}\label{eq:ctmc-adjoint-formula}
(\Lgen_t)^* \phi(\x) = \sum_{\y \in \calX^d} \Rseq_t(\x, \y) \, \phi(\y).
\end{equation}
In matrix-vector notation:
\begin{equation}
(\Lgen_t)^* \phib = \Rseq_t \, \phib, \qquad \text{i.e.,} \quad (\Lgen_t)^* = \Rseq_t.
\end{equation}
\end{proposition}

\begin{proof}
For discrete state spaces with counting measure, the inner product is:
\begin{equation}
\langle f, g \rangle = \sum_{\x \in \calX^d} f(\x) \, g(\x).
\end{equation}

The adjoint $(\Lgen_t)^*$ is defined by the relation $\langle \Lgen_t f, g \rangle = \langle f, (\Lgen_t)^* g \rangle$ for all test functions $f, g$. Computing the left-hand side using \cref{eq:ctmc-generator-formula}:
\begin{align}
\langle \Lgen_t f, g \rangle 
&= \sum_{\x \in \calX^d} (\Lgen_t f)(\x) \, g(\x) \notag\\
&= \sum_{\x \in \calX^d} \left(\sum_{\y \in \calX^d} \Rseq_t(\y, \x) \, f(\y)\right) g(\x) \notag\\
&= \sum_{\x \in \calX^d} \sum_{\y \in \calX^d} \Rseq_t(\y, \x) \, f(\y) \, g(\x).
\end{align}

Exchanging the order of summation:
\begin{align}
\langle \Lgen_t f, g \rangle 
&= \sum_{\y \in \calX^d} f(\y) \sum_{\x \in \calX^d} \Rseq_t(\y, \x) \, g(\x) \notag\\
&= \sum_{\y \in \calX^d} f(\y) \left(\sum_{\x \in \calX^d} \Rseq_t(\y, \x) \, g(\x)\right).
\end{align}

By the definition of adjoint, we identify:
\begin{equation}
((\Lgen_t)^* g)(\y) = \sum_{\x \in \calX^d} \Rseq_t(\y, \x) \, g(\x).
\end{equation}

Relabeling $\y \to \x$ and $\x \to \y$ in the summation yields the claimed formula \eqref{eq:ctmc-adjoint-formula}.

In matrix notation we have $(\Lgen_t)^* = \Rseq_t$.
\end{proof}

\subsection{Dimensional decomposition of the generator\label{apx:generator-decomposition-dimension}}

\begin{proposition}[Additive decomposition of the generator]\label{prop:generator-decomposition}
Let $\{\x_t\}_{t\in[0,1]}$ be a (time-inhomogeneous) Markov process on the product space $\calX^d = \calX^{(1)} \times \cdots \times \calX^{(d)}$. If the forward noising is \emph{independent across coordinates}, i.e., the transition kernel factorises as
\begin{equation}
q_{t|s}(\x \cond \y) = \prod_{k=1}^d q_{t|s}(x^{(k)} \cond y^{(k)}),
\end{equation}
then the sequence-level spatial generator decomposes additively:
\begin{equation}\label{eq:generator-additive-decomp}
\Lgen_t = \sum_{k=1}^d \Lgen_t^{(k)},
\end{equation}
where $\Lgen_t^{(k)}$ acts only on coordinate $k$. Consequently, defining the $k$-th coordinate marginal as
\begin{equation}
q_t^{(k)}(x^{(k)}) \coloneqq \int_{\calX^{d-1}} q_t(\x) \, \diff\mu(\x^{\setminus k}),
\end{equation}
where $\x^{\setminus k} = (x^{(1)}, \ldots, x^{(k-1)}, x^{(k+1)}, \ldots, x^{(d)})$ denotes all coordinates except the $k$-th, each coordinate marginal evolves according to its own single-dimension KFE:
\begin{equation}
\partial_t q_t^{(k)} = (\Lgen_t^{(k)})^* q_t^{(k)}.
\end{equation}
\end{proposition}

\begin{proof}
The proof proceeds in two parts: (i) establishing the additive decomposition of the generator, and (ii) deriving the marginal evolution in continuous and discrete state spaces.

\paragraph{Part I: Additive decomposition of the generator.}

Let $\Lgen_t^{(k)}$ denote the single-coordinate generator acting on $\calX^{(k)}$. We work on the dense subspace of separable test functions:
\begin{equation}
\phi(\x) = \prod_{k=1}^d \psi_k(x^{(k)}), \qquad \psi_k : \calX^{(k)} \to \RR.
\end{equation}

By the independence assumption, the short-time conditional expectation factorises:
\begin{align}
\E\!\left[\phi(\x_{t+\Delta t}) \cond \x_t = \x\right]
&= \prod_{k=1}^d \E\!\left[\psi_k(x^{(k)}_{t+\Delta t}) \cond x^{(k)}_t = x^{(k)}\right] \notag\\
&= \prod_{k=1}^d \Bigl(\psi_k(x^{(k)}) + \Delta t \cdot \Lgen_t^{(k)} \psi_k(x^{(k)}) + o(\Delta t)\Bigr).
\end{align}

Expanding the product to first order in $\Delta t$ (cross-terms are $O(\Delta t^2)$):
\begin{align}
\E\!\left[\phi(\x_{t+\Delta t}) \cond \x_t = \x\right]
= \phi(\x) + \Delta t \sum_{k=1}^d \Bigl[\Lgen_t^{(k)} \psi_k(x^{(k)}) \prod_{\ell \neq k} \psi_\ell(x^{(\ell)})\Bigr] + o(\Delta t).
\end{align}

Applying the generator definition:
\begin{align}
\Lgen_t \phi(\x)
&= \lim_{\Delta t \to 0^+} \frac{\E[\phi(\x_{t+\Delta t}) \cond \x_t = \x] - \phi(\x)}{\Delta t} \notag\\
&= \sum_{k=1}^d \Lgen_t^{(k)} \psi_k(x^{(k)}) \prod_{\ell \neq k} \psi_\ell(x^{(\ell)}).
\end{align}

By linearity and density of the separable core in $\calD(\Lgen_t) \subseteq L^2(\calX^d)$, the identity $\Lgen_t = \sum_{k=1}^d \Lgen_t^{(k)}$ extends to the full domain $\mathcal{D}(\Lgen_t)$.

\paragraph{Part II: Marginal evolution.}

We now show that each coordinate marginal evolves independently.

\subparagraph{Continuous state spaces.}
For a diffusion process with  (i.e., $\f_t(\x) = (f_t^{(1)}(x^{(1)}), \ldots, f_t^{(d)}(x^{(d)}))$ and scalar diffusion coefficient $g_t$), the single-coordinate \emph{adjoint} generator is:
\begin{equation}
(\Lgen_t^{(k)})^* \phi(x^{(k)}) = -\partial_{x^{(k)}} \bigl(f_t^{(k)}(x^{(k)}) \phi(x^{(k)})\bigr) + \tfrac{1}{2} g_t^2 \, \partial_{x^{(k)} x^{(k)}} \phi(x^{(k)}).
\end{equation}
This is simply the per-coordinate decomposition of the full adjoint generator:
\begin{equation}
(\Lgen_t)^* \phi(\x) = -\nabla_\x \cdot (\f_t(\x) \phi(\x)) + \tfrac{1}{2} g_t^2 \Delta_\x \phi(\x).
\end{equation}
The marginal $q_t^{(k)}(x^{(k)}) = \int_{\calX^{d-1}} q_t(\x) \, \diff\x^{\setminus k}$ then satisfies $\partial_t q_t^{(k)} = (\Lgen_t^{(k)})^* q_t^{(k)}$ by integrating the full KFE over complementary coordinates.

\subparagraph{Discrete state spaces.}
Let $\calX = \{1, \ldots, K\}$. The additive generator decomposition corresponds to the \emph{Kronecker sum} of rate matrices:
\begin{equation}\label{eq:kronecker-sum-apx}
\Rseq_t = \bigoplus_{k=1}^d \R_t^{(k)} = \sum_{k=1}^d \Bigl(\I^{\otimes(k-1)} \otimes \R_t^{(k)} \otimes \I^{\otimes(d-k)}\Bigr).
\end{equation}

We prove that each coordinate marginal $\q_t^{(k)} \in \RR^K$ satisfies $\partial_t \q_t^{(k)} = \R_t^{(k)} \q_t^{(k)}$.

\textit{Setup.} Stack the joint probabilities into a column vector $\q_t \in \RR^{K^d}$ using lexicographic order. Define the marginalisation operator $\Sb_k : \RR^{K^d} \to \RR^K$ that sums over all coordinates except $k$:
\begin{equation}
(\Sb_k \mathbf{q}_t)(i) = \sum_{\substack{x^{(\ell)} \in \calX \\ \ell \neq k}} q_t(x^{(1)}, \ldots, x^{(k)} = i, \ldots, x^{(d)}), \qquad i \in \{1, \ldots, K\}.
\end{equation}

Under our stacking convention, $\Sb_k$ admits the Kronecker form:
\begin{equation}\label{eq:Sk-kronecker-apx}
\Sb_k = (\mathbf{1}^\top)^{\otimes(k-1)} \otimes \I_K \otimes (\mathbf{1}^\top)^{\otimes(d-k)},
\end{equation}
where $\mathbf{1}^\top \in \RR^{1 \times K}$ is the row vector of ones and $\I_K \in \RR^{K \times K}$ is the identity.

\textit{Example ($d=2$, $K=3$).} Let states be $(i,j) \in \{1,2,3\}^2$, stacked as:
\begin{equation*}
\mathbf{q}_t = \bigl(q(1,1), q(1,2), q(1,3), q(2,1), q(2,2), q(2,3), q(3,1), q(3,2), q(3,3)\bigr)^\top.
\end{equation*}
Then:
\begin{equation*}
\Sb_1 = \I_3 \otimes \mathbf{1}^\top =
\begin{bmatrix}
1 & 1 & 1 & 0 & 0 & 0 & 0 & 0 & 0\\
0 & 0 & 0 & 1 & 1 & 1 & 0 & 0 & 0\\
0 & 0 & 0 & 0 & 0 & 0 & 1 & 1 & 1
\end{bmatrix}, \quad
\Sb_2 = \mathbf{1}^\top \otimes \I_3 =
\begin{bmatrix}
1 & 0 & 0 & 1 & 0 & 0 & 1 & 0 & 0\\
0 & 1 & 0 & 0 & 1 & 0 & 0 & 1 & 0\\
0 & 0 & 1 & 0 & 0 & 1 & 0 & 0 & 1
\end{bmatrix}.
\end{equation*}
Applying $\Sb_1$ sums each row of the $3 \times 3$ joint probability table (marginalising over $x^{(2)}$); $\Sb_2$ sums each column (marginalising over $x^{(1)}$).

\textit{Main calculation.} We will now show that $\partial_t \q_t^{(k)} = \Sb_k\partial_t \q_t = \R_t^{(k)} \q_t^{(k)}$. From the joint KFE $\partial_t \mathbf{q}_t = \Rseq_t \mathbf{q}_t$:
\begin{align}
\partial_t \q_t^{(k)} 
&= \Sb_k \, \partial_t \mathbf{q}_t 
= \Sb_k \, \Rseq_t \, \mathbf{q}_t 
= \sum_{j=1}^d \Sb_k \Bigl(\I^{\otimes(j-1)} \otimes \R_t^{(j)} \otimes \I^{\otimes(d-j)}\Bigr) \mathbf{q}_t.
\label{eq:marginal-derivation}
\end{align}

We evaluate each term using the \emph{mixed-product property}: $(A_1 \otimes A_2 \otimes A_3)(B_1 \otimes B_2 \otimes B_3) = (A_1 B_1) \otimes (A_2 B_2) \otimes (A_3 B_3)$, valid whenever the matrix products $A_i B_i$ are well-defined (i.e., the number of columns of $A_i$ equals the number of rows of $B_i$ for each $i$).

\textit{Case $j = k$:}
\begin{align}
&\Sb_k \Bigl(\I^{\otimes(k-1)} \otimes \R_t^{(k)} \otimes \I^{\otimes(d-k)}\Bigr) \notag\\
&= \Bigl[(\mathbf{1}^\top)^{\otimes(k-1)} \otimes \I_K \otimes (\mathbf{1}^\top)^{\otimes(d-k)}\Bigr] \Bigl[\I^{\otimes(k-1)} \otimes \R_t^{(k)} \otimes \I^{\otimes(d-k)}\Bigr] \notag\\
&= \underbrace{\bigl((\mathbf{1}^\top)^{\otimes(k-1)} \cdot \I^{\otimes(k-1)}\bigr)}_{= (\mathbf{1}^\top)^{\otimes(k-1)}} \otimes \underbrace{(\I_K \cdot \R_t^{(k)})}_{= \R_t^{(k)}} \otimes \underbrace{\bigl((\mathbf{1}^\top)^{\otimes(d-k)} \cdot \I^{\otimes(d-k)}\bigr)}_{= (\mathbf{1}^\top)^{\otimes(d-k)}} \notag\\
&= (\mathbf{1}^\top)^{\otimes(k-1)} \otimes \R_t^{(k)} \otimes (\mathbf{1}^\top)^{\otimes(d-k)}\\
&= \R_t^{(k)} \Bigl[(\mathbf{1}^\top)^{\otimes(k-1)} \otimes \I \otimes (\mathbf{1}^\top)^{\otimes(d-k)}\Bigr]=\R_t^{(k)} \Sb_k .
\label{eq:case-j-equals-k}
\end{align}

To recover the last equality \cref{eq:case-j-equals-k}, we simply write $\R_t^{(k)} = \I_1 \otimes \R_t^{(k)} \otimes \I_1$, where $\I_1 = [1]$ is the $1 \times 1$ identity. Applying the mixed-product property:
\begin{align}
&\R_t^{(k)} \cdot \Sb_k = \bigl(\I_1 \otimes \R_t^{(k)} \otimes \I_1\bigr) \bigl((\mathbf{1}^\top)^{\otimes(k-1)} \otimes \I_K \otimes (\mathbf{1}^\top)^{\otimes(d-k)}\bigr) \notag\\
&\qquad= (\mathbf{1}^\top)^{\otimes(k-1)} \otimes \R_t^{(k)} \otimes (\mathbf{1}^\top)^{\otimes(d-k)}.
\end{align}

\textit{Case $j \neq k$:} Without loss of generality, assume $j < k$ (the case $j > k$ is analogous). The mixed-product property yields a factor:
\begin{equation}
\mathbf{1}^\top \cdot \R_t^{(j)} = \mathbf{0}^\top,
\end{equation}
by the \emph{column-sum-to-zero property} of rate matrices: $\sum_{i=1}^K [\R_t^{(j)}]_{i\ell} = 0$ for all $\ell$. Therefore:
\begin{align}
\Sb_k \Bigl(\I^{\otimes(j-1)} \otimes \R_t^{(j)} \otimes \I^{\otimes(d-j)}\Bigr) = \mathbf{0}.\\
\end{align}

\textit{Conclusion.} Substituting into \cref{eq:marginal-derivation}:
\begin{equation}
\partial_t \q_t^{(k)} = \R_t^{(k)} \Sb_k \mathbf{q}_t = \R_t^{(k)} \q_t^{(k)}.
\end{equation}
This completes the proof.
\end{proof}

\begin{remark}[Continuous analogue of marginalisation]
In continuous state spaces with Lebesgue measure, the marginalisation operator $\Sb_k$ corresponds to integration over complementary coordinates:
\begin{equation}
(\Sb_k q_t)(x^{(k)}) = \int_{\RR^{d-1}} q_t(x^{(1)}, \ldots, x^{(k)}, \ldots, x^{(d)}) \, \diff x^{\setminus k}.
\end{equation}
The Kronecker form \eqref{eq:Sk-kronecker-apx} is its discrete analogy.
\end{remark}
\subsection{Score identity}
\label{apx:tweedie-identity}

The score identity provides a fundamental link between the intractable marginal (Stein) score $\nabla_{\x_t} \log q_t(\x_t)$ and the tractable conditional score $\nabla_{\x_t} \log q(\x_t \cond \x_0)$.

\begin{lemma}[Score Identity]\label{prop:tweedie}
Let $\{\x_t\}_{t \in [0,1]}$ be a diffusion process on $\RR^d$ with marginal density $q_t$ and forward transition density $q(\x_t \cond \x_0)$. Under standard regularity conditions permitting the exchange of differentiation and integration, the marginal score satisfies:
\begin{equation}\label{eq:tweedie-main}
\nabla_{\x_t} \log q_t(\x_t) = \E\left[\nabla_{\x_t} \log q(\x_t \cond \x_0) \,\big|\, \x_t\right].
\end{equation}
\end{lemma}

\begin{proof}
Starting from the definition of the score and the law of total probability:
\begin{align}
\nabla_{\x_t} \log q(\x_t)
&= \frac{\nabla_{\x_t} q(\x_t)}{q(\x_t)} 
\label{eq:tweedie-step1}\\[4pt]
&= \frac{1}{q(\x_t)} \nabla_{\x_t} \int_{\RR^d} q(\x_t \cond \x_0)\, q(\x_0) \,\diff \x_0 
\label{eq:tweedie-step2}\\[4pt]
&= \frac{1}{q(\x_t)} \int_{\RR^d} \nabla_{\x_t} q(\x_t \cond \x_0)\, q(\x_0) \,\diff \x_0 
\label{eq:tweedie-step3}\\[4pt]
&= \frac{1}{q(\x_t)} \int_{\RR^d} q(\x_t \cond \x_0)\, \nabla_{\x_t} \log q(\x_t \cond \x_0)\, q(\x_0) \,\diff \x_0 
\label{eq:tweedie-step4}\\[4pt]
&= \int_{\RR^d} \nabla_{\x_t} \log q(\x_t \cond \x_0) \cdot \frac{q(\x_t \cond \x_0)\, q(\x_0)}{q(\x_t)} \,\diff \x_0 
\label{eq:tweedie-step5}\\[4pt]
&= \int_{\RR^d} \nabla_{\x_t} \log q(\x_t \cond \x_0) \cdot q(\x_0 \cond \x_t) \,\diff \x_0 
\label{eq:tweedie-step6}\\[4pt]
&= \E\left[\nabla_{\x_t} \log q(\x_t \cond \x_0) \,\big|\, \x_t\right].
\label{eq:tweedie-step7}
\end{align}

\noindent\textit{Justification of key steps:}
\begin{enumerate}
    \item \textit{\Cref{eq:tweedie-step2} $\to$ \cref{eq:tweedie-step3}}: Exchange of gradient and integral, valid under the dominated convergence theorem under standard regularity assumptions on $\nabla_{\x_t} q(\x_t \cond \x_0)$.
    \item \textit{\Cref{eq:tweedie-step3} $\to$ \cref{eq:tweedie-step4}}: Uses the identity $\nabla f = f \nabla \log f$ for $f > 0$.
    \item \textit{\Cref{eq:tweedie-step5} $\to$ \cref{eq:tweedie-step6}}: Applies Bayes' rule: $q(\x_0 \cond \x_t) = \frac{q(\x_t \cond \x_0)\, q_0(\x_0)}{q_t(\x_t)}$.
\end{enumerate}
\end{proof}

\begin{corollary}[Score as Conditional Expectation Minimiser]\label{cor:score-regression}
Since the conditional expectation $\E[\,\cdot \cond \x_t]$ is the orthogonal projection onto the space of $\sigma(\x_t)$-measurable functions in $L^2$, Tweedie's identity implies that the marginal score minimises the expected squared error:
\begin{equation}\label{eq:score-regression-apx}
\nabla_{\x_t} \log q_t = \operatorname*{argmin}_{f \,\in\, L^2(\RR^d, q_t)} \E_{\x_0, \x_t}\left[\bigl\|f(\x_t) - \nabla_{\x_t} \log q(\x_t \cond \x_0)\bigr\|^2\right].
\end{equation}
This justifies denoising score matching: training a neural network $\mathbf{s}_\theta(\x_t, t)$ to predict $\nabla_{\x_t} \log q(\x_t \cond \x_0)$ recovers the marginal score $\nabla_{\x_t} \log q_t(\x_t)$ at the optimum.
\end{corollary}

\section{Loss derivation from ELBO\label{apx:ELBO_derivation}}

\subsection{Discrete time path-space to denoising step KL \label{apx:discrete-time-ELBO}}

\begin{lemma}[ELBO decomposition]\label{lem:elbo-decomposition}
Let $q(\x_{1:T} \cond \x_0) = \prod_{t=1}^{T} q(\x_t \cond \x_{t-1})$ be the forward diffusion process and $p^\thetab(\x_{0:T}) = p(\x_T) \prod_{t=1}^{T} p^\thetab(\x_{t-1} \cond \x_t)$ be the parameterised reverse process. Then the reconstruction term simplifies to
\begin{equation}\label{eq:reconstruction-simplification}
\E_{q(\x_{1:T}\cond \x_0)} \left[ \log p^\thetab(\x_0 \cond \x_{1:T}) \right] = \E_{q(\x_1 \cond \x_0)} \left[\log p^\thetab(\x_0\cond \x_1)\right],
\end{equation}
and the path-space KL divergence decomposes as
\begin{align}\label{eq:path-kl-decomposition}
D_{\mathrm{KL}} \left( q(\x_{1:T} \cond \x_0) \,\|\, p^\thetab(\x_{1:T}) \right)
&= D_{\mathrm{KL}}\bigl(q(\x_T \cond \x_0)\,\|\,p(\x_T)\bigr) \notag\\
&\quad + \sum_{t=2}^T \E_{q(\x_t \cond \x_0)}\left[D_{\mathrm{KL}}\bigl(q(\x_{t-1} \cond \x_{t}, \x_0)\,\|\,p^\thetab(\x_{t-1} \cond \x_{t})\bigr)\right],
\end{align}
wherer $p(\x_T)\approx p_\text{noise}$ the prior distribution and $D_{\mathrm{KL}}(\cdot \| \cdot)$ is the Kullback-Leibler divergence.
\end{lemma}
\begin{proof}
We prove each part separately.

\paragraph{Part I: Reconstruction term.}
Since the reverse process $p^\thetab$ is first-order Markov:
\begin{equation}
p^\thetab(\x_0, \x_1, \ldots, \x_T) = p(\x_T) \prod_{t=1}^{T} p^\thetab(\x_{t-1} \cond \x_t).
\end{equation}
Conditioning on the full trajectory $\x_{1:T}$ and applying the Markov property:
\begin{align}
p^\thetab(\x_0 \cond \x_{1:T})
&= \frac{p(\x_T) \prod_{t=1}^{T} p^\thetab(\x_{t-1} \cond \x_t)}{p(\x_T) \prod_{t=2}^{T} p^\thetab(\x_{t-1} \cond \x_t)}
= p^\thetab(\x_0 \cond \x_1).
\end{align}
Hence $\E_{q(\x_{1:T}\cond \x_0)}[\log p^\thetab(\x_0 \cond \x_{1:T})] = \E_{q(\x_1 \cond \x_0)}[\log p^\thetab(\x_0 \cond \x_1)]$.

\paragraph{Part II: Path-space KL decomposition.}
By definition of KL divergence:
\begin{align}
D_{\mathrm{KL}}\bigl(q(\x_{1:T}\cond\x_0)\,\|\,p^\thetab(\x_{1:T})\bigr)
&= \E_{q(\x_{1:T}\cond\x_0)}
   \left[
      \log\frac{\prod_{t=1}^{T}q(\x_t\cond\x_{t-1})}
               {p(\x_T)\prod_{t=2}^{T}p^\thetab(\x_{t-1}\cond\x_t)}
   \right].
\label{eq:path-kl-start}
\end{align}

\emph{Telescoping the forward process.} Using Bayes' rule on the forward kernel:
\begin{equation}
q(\x_t \cond \x_{t-1}) = q(\x_t \cond \x_{t-1}, \x_0) = q(\x_{t-1} \cond \x_t, \x_0) \cdot \frac{q(\x_t \cond \x_0)}{q(\x_{t-1} \cond \x_0)},
\end{equation}
the product telescopes:
\begin{align}
\prod_{t=1}^{T} q(\x_t \cond \x_{t-1})
&= \prod_{t=1}^{T} q(\x_{t-1} \cond \x_t, \x_0) \cdot \frac{q(\x_t \cond \x_0)}{q(\x_{t-1} \cond \x_0)} \notag\\
&= q(\x_T \cond \x_0) \prod_{t=2}^{T} q(\x_{t-1} \cond \x_t, \x_0).
\end{align}

\emph{Substituting into \cref{eq:path-kl-start}:}
\begin{align}
D_{\mathrm{KL}}\bigl(q(\x_{1:T}\cond\x_0)\,\|\,p^\thetab(\x_{1:T})\bigr)
&= \E_{q(\x_{1:T}\cond\x_0)}
   \left[
      \log\frac{q(\x_T\cond\x_0)}{p(\x_T)}
      + \sum_{t=2}^{T}
        \log\frac{q(\x_{t-1}\cond\x_t,\x_0)}{p^\thetab(\x_{t-1}\cond\x_t)}
   \right] \notag\\
&= D_{\mathrm{KL}}\bigl(q(\x_T \cond \x_0) \,\|\, p(\x_T)\bigr) 
   + \sum_{t=2}^{T} \E_{q(\x_{1:T}\cond\x_0)}\left[\log\frac{q(\x_{t-1}\cond\x_t,\x_0)}{p^\thetab(\x_{t-1}\cond\x_t)}\right].
\label{eq:path-kl-telescoped}
\end{align}

\emph{Simplifying the per-step expectations.} Since the integrand $\log\frac{q(\x_{t-1}\cond\x_t,\x_0)}{p^\thetab(\x_{t-1}\cond\x_t)}$ depends only on $(\x_{t-1}, \x_t)$, we apply the law of iterated expectation:
\begin{align}
\E_{q(\x_{1:T}\cond\x_0)}\left[\log\frac{q(\x_{t-1}\cond\x_t,\x_0)}{p^\thetab(\x_{t-1}\cond\x_t)}\right]
&= \E_{q(\x_t\cond\x_0)}\left[
   \E_{q(\x_{t-1}\cond\x_t,\x_0)}\left[
      \log\frac{q(\x_{t-1}\cond\x_t,\x_0)}{p^\thetab(\x_{t-1}\cond\x_t)}
   \right]
\right].
\label{eq:iterated-expectation}
\end{align}

By definition of KL divergence, for each fixed $\x_t$:
\begin{equation}
D_{\mathrm{KL}}\bigl(q(\x_{t-1}\cond\x_t,\x_0)\,\|\,p^\thetab(\x_{t-1}\cond\x_t)\bigr)
= \E_{q(\x_{t-1}\cond\x_t,\x_0)}\left[
   \log\frac{q(\x_{t-1}\cond\x_t,\x_0)}{p^\thetab(\x_{t-1}\cond\x_t)}
\right].
\label{eq:kl-definition}
\end{equation}

Substituting \cref{eq:kl-definition} into \cref{eq:iterated-expectation}:
\begin{equation}
\E_{q(\x_{1:T}\cond\x_0)}\left[\log\frac{q(\x_{t-1}\cond\x_t,\x_0)}{p^\thetab(\x_{t-1}\cond\x_t)}\right]
= \E_{q(\x_t \cond \x_0)} \left[ D_{\mathrm{KL}}\bigl(q(\x_{t-1}\cond\x_t,\x_0)\,\|\,p^\thetab(\x_{t-1}\cond\x_t)\bigr) \right].
\end{equation}

Combining with \cref{eq:path-kl-telescoped} completes the proof.
\end{proof}

\subsection{Path-space KL divergence and the ELBO \label{apx:ELBO_paths}}

This appendix establishes the key inequality relating the marginal KL divergence to the path-space KL divergence, which underpins the ELBO derivation.

\subsubsection{Chain rule decomposition}

\begin{proposition}[Path-space chain rule]\label{prop:path-chain-rule}
Let $\hat \Qpath$ denote the path measure of the true reverse process $\{\hat{\x}_s\}_{s \in [0,1]}$ and $\hat \Ppath^\theta$ the path measure of the learned reverse process. Then:
\begin{equation}
\KL(\hat{\Qpath} \,\|\, \hat \Ppath^\theta)
\;=\;
\KL(\qdata \,\|\, p_0^\theta)
\;+\;
\E_{\x \sim \qdata}
\Bigl[
  \KL
  \bigl(
   \hat{\Qpath}(\cdot \cond \hat{\x}_1 = \x)
    \,\big\|\,
    \hat \Ppath^\theta(\cdot \cond \hat{\x}_1^\theta = \x)
  \bigr)
\Bigr],
\label{eq:KL-chain-rule}
\end{equation}
where $\hat{\x}_1 = \x_0$ denotes the endpoint of the reverse process (i.e., the reconstructed data).
\end{proposition}

\begin{proof}
This follows from the chain rule for KL divergence \cite{leonard2014some}. For any joint distributions $P(A, B)$ and $Q(A, B)$:
\begin{equation}
\KL(P(A, B) \,\|\, Q(A, B)) = \KL(P(A) \,\|\, Q(A)) + \E_{a \sim P(A)}\bigl[\KL(P(B \cond A=a) \,\|\, Q(B \cond A=a))\bigr].
\end{equation}
Applying this with $A = \hat{\x}_1$ (the endpoint/data) and $B = \{\hat{\x}_s\}_{s \in [0,1)}$ (the rest of the path) yields \cref{eq:KL-chain-rule}.
\end{proof}

Since KL divergence is non-negative, we immediately obtain:
\begin{equation}
\KL(\qdata \,\|\, p_0^\theta) \;\leq\; \KL(\hat{\Qpath} \,\|\, \hat \Ppath^\theta).
\label{eq:marginal-path-inequality}
\end{equation}

\subsubsection{Data processing inequality}\label{apx:data_processing_inequality}

We provide an alternative of \cref{eq:marginal-path-inequality} to relate the marginal KL divergence to the path-space KL divergence, using the data processing inequality, following \cite{song2021maximum}.

\begin{proposition}[Marginal bound via data processing]
For path measures $\hat{\Qpath}$ and $\hat \Ppath^\theta$ of the true and learned reverse processes:
\begin{equation}
\KL(\qdata \,\|\, p_0^\theta) \;\leq\; \KL(\hat{\Qpath} \,\|\, \hat \Ppath^\theta).
\end{equation}
\end{proposition}

\begin{proof}
Define the Markov kernel $K$ that extracts the endpoint of a reverse path:
\begin{equation}
K\bigl(\{\hat{\x}_s\}_{s \in [0,1]},\, \y\bigr) \;\coloneqq\; \delta_{\hat{\x}_1}(\y),
\end{equation}
where $\hat{\x}_1$ is the terminal value of the reverse process (corresponding to the reconstructed data $\x_0$). Intuitively, $K$ ``forgets'' the entire trajectory except its final state.

The pushforward of a path measure under $K$ yields the marginal distribution at the endpoint:
\begin{align}
K \# \hat{\Qpath} &= \hat{q}_1 = q_0 = \qdata, \\
K \# \hat \Ppath^\theta &= p_0^\theta.
\end{align}

The \emph{data processing inequality} (monotonicity of KL divergence under Markov kernels) states that for any probability measures $\alpha, \beta$ on the domain of a kernel $K$:
\begin{equation}
\KL(K \# \alpha \,\|\, K \# \beta) \;\leq\; \KL(\alpha \,\|\, \beta).
\end{equation}
This inequality reflects the fact that coarsening or discarding information (here, forgetting the path and retaining only the endpoint) cannot increase the distinguishability between two distributions.

Applying this with $\alpha = \hat{\Qpath}$ and $\beta = \hat \Ppath^\theta$:
\begin{equation}
\KL(\qdata \,\|\, p_0^\theta) 
\;=\; \KL(K \# \hat{\Qpath} \,\|\, K \# \hat \Ppath^\theta) 
\;\leq\; \KL(\hat{\Qpath} \,\|\, \hat \Ppath^\theta).
\end{equation}
\end{proof}

\begin{corollary}[Connection to the ELBO]
The inequality $\KL(\qdata \,\|\, p_0^\theta) \leq \KL(\hat{\Qpath} \,\|\, \hat \Ppath^\theta)$ implies:
\begin{equation}
-\E_{\qdata}[\log p_0^\theta(\x_0)] \;\leq\; \KL(\hat{\Qpath} \,\|\, \hat \Ppath^\theta) + \mathcal{H}(\qdata),
\end{equation}
where $\mathcal{H}(\qdata)$ is the entropy of the data distribution. Since $\mathcal{H}(\qdata)$ is constant with respect to $\theta$, minimising the path-space KL divergence $\KL(\hat{\Qpath} \,\|\, \hat \Ppath^\theta)$ provides an upper bound on the negative log-likelihood, forming the basis of the ELBO in \cref{sect:ELBO}.
\end{corollary}

We derive the score matching (SM) objective by bounding the KL divergence between data distributions using path-space KL divergences and Girsanov's theorem.
\paragraph{Setup.}
Let $\hat q_s$ and $\hat p_s$, $s\in[0,1]$, be the marginal density paths induced by two SDEs written in the forward reverse-time parameter $s$:
\begin{align}
    \diff\hat\x_s &= \mub_q(\hat\x_s,s)\,\diff s + g_{1-s}\,\diff\wtilde_s, \quad \hat\x_0\sim q_1, \label{eq:reverse-sde-q}\\
    \diff\hat\x_s &= \mub_p(\hat\x_s,s)\,\diff s + g_{1-s}\,\diff\wtilde'_s, \quad \hat\x_0\sim p_1, \label{eq:reverse-sde-p}
\end{align}
Here $s$ increases from $0$ (noise) to $1$ (data), $\hat q_s=q_{1-s}$, and we write $\hat p_s=p_{1-s}$ for the learned reverse marginals. The processes $\wtilde_s$ and $\wtilde'_s$ are Wiener processes in the $s$-parameter. In the context of diffusion generative modelling:
\begin{enumerate}
    \item True reverse process \eqref{eq:reverse-sde-q}: The time-reversal of the forward noising process (e.g., variance-preserving Ornstein--Uhlenbeck), mapping $q_1(\x)=\int q_{1|0}(\x\cond\x_0)\,\qdata(\x_0)\,\diff\x_0$ to $q_0=\qdata$. Its drift depends on the intractable marginal score $\nabla_\x\log q_{1-s}(\x)$.
    \item Learned reverse process \eqref{eq:reverse-sde-p}: Maps a tractable prior $p_1$ to a learned distribution $p_0\approx\qdata$. Its drift replaces the true score with a learned score model $\s_\theta(\x,1-s)\approx\nabla_\x\log q_{1-s}(\x)$.
\end{enumerate}

Let $\Qpath$ and $\Ppath$ denote the path measures induced by \eqref{eq:reverse-sde-q} and \eqref{eq:reverse-sde-p} on $C([0,1],\RR^d)$.

\begin{proposition}[ELBO via path-space KL]\label{prop:elbo-girsanov}
Under the setup above, the KL divergence between the generated and data distributions satisfies:
\begin{equation}
\KL(q_0 \| p_0) \leq \KL(q_1 \| p_1) + \frac{1}{2} \int_0^1 g_t^2 \, \E_{\x_t \sim q_t} \left[ \| \s_\theta(\x_t, t) - \nabla_{\x_t} \log q_t(\x_t) \|^2 \right] \diff t.
\label{eq:elbo-bound-girsanov}
\end{equation}
\end{proposition}

\begin{proof}
By the chain rule for path-space KL, disintegrating first at the data endpoint $s=1$ and then at the noise endpoint $s=0$,
\begin{align}
    \KL(\Qpath \| \Ppath)
    &= \KL(q_0 \| p_0) + \E_{\x \sim q_0}\!\left[\KL\!\left(\Qpath(\cdot\cond\hat\x_1=\x)\,\|\,\Ppath(\cdot\cond\hat\x_1=\x)\right)\right], \label{eq:chain-rule-t0}\\
    \KL(\Qpath \| \Ppath)
    &= \KL(q_1 \| p_1) + \E_{\x \sim q_1}\!\left[\KL\!\left(\Qpath(\cdot\cond\hat\x_0=\x)\,\|\,\Ppath(\cdot\cond\hat\x_0=\x)\right)\right]. \label{eq:chain-rule-t1}
\end{align}
Subtracting \eqref{eq:chain-rule-t0} from \eqref{eq:chain-rule-t1} and using non-negativity of the endpoint-conditioned KL gives
\begin{equation}
\KL(q_0 \| p_0) \leq \KL(q_1 \| p_1) + \E_{\x \sim q_1}\!\left[\KL\!\left(\Qpath(\cdot\cond\hat\x_0=\x)\,\|\,\Ppath(\cdot\cond\hat\x_0=\x)\right)\right].
\label{eq:kl-bound-endpoint}
\end{equation}

Conditioned on the same initial point $\hat\x_0=\x$, the two SDEs have the same diffusion coefficient. Under the usual absolute-continuity and integrability hypotheses, Girsanov's theorem \cite{oksendal2003stochastic} therefore yields
\begin{equation}
\KL\!\left(\Qpath(\cdot\cond\hat\x_0=\x)\,\|\,\Ppath(\cdot\cond\hat\x_0=\x)\right)
= \frac{1}{2}\,\E_{\Qpath(\cdot\cond\hat\x_0=\x)}\!\left[
\int_0^1 \frac{\|\mub_p(\hat\x_s,s)-\mub_q(\hat\x_s,s)\|^2}{g_{1-s}^2}\,\diff s\right].
\label{eq:girsanov-kl}
\end{equation}

Averaging over $\hat\x_0\sim q_1$, applying the tower property and then Fubini--Tonelli, gives
\begin{equation}
\E_{\x \sim q_1}\!\left[\KL\!\left(\Qpath(\cdot\cond\hat\x_0=\x)\,\|\,\Ppath(\cdot\cond\hat\x_0=\x)\right)\right]
= \frac{1}{2}\int_0^1\frac{1}{g_{1-s}^2}\,
\E_{\hat\x_s\sim q_{1-s}}\!\left[\|\mub_p(\hat\x_s,s)-\mub_q(\hat\x_s,s)\|^2\right]\diff s.
\label{eq:fubini-marginal}
\end{equation}

For the reverse-time SDEs (cf.\ \cref{eq:reverse-SDE}),
\begin{align}
    \mub_q(\x,s) &= -\f_{1-s}(\x) + g_{1-s}^2\nabla_\x\log q_{1-s}(\x),\\
    \mub_p(\x,s) &= -\f_{1-s}(\x) + g_{1-s}^2\s_\theta(\x,1-s).
\end{align}
Hence
\begin{equation}
\|\mub_p(\x,s)-\mub_q(\x,s)\|^2
= g_{1-s}^4\|\s_\theta(\x,1-s)-\nabla_\x\log q_{1-s}(\x)\|^2.
\end{equation}
Substituting this identity into \eqref{eq:fubini-marginal}, combining with \eqref{eq:kl-bound-endpoint}, and changing variables $t=1-s$ yields \eqref{eq:elbo-bound-girsanov}.
\end{proof}

\section{Latent diffusion ELBO \label{apx:latent-diff-ELBO}}

We derive the evidence lower bound (ELBO) for latent diffusion models, where data $\x_0$ is encoded into a latent space $\mathcal{Z}$, and diffusion is performed in this latent space.

\begin{lemma}[Latent Diffusion ELBO]
Let $q^{\psib}(\z_0 \cond \x_0)$ be an encoder mapping data to latent representations, $p^\phib(\x_0 \cond \z_0)$ a decoder reconstructing data from latents, and $p^\thetab(\z_0)$ the latent prior learned via diffusion. The negative log-likelihood satisfies:
\begin{equation}\label{eq:latent-elbo-bound}
-\log p^\thetab(\x_0) \;\leq\; \KL\bigl(q^\psib(\z_0 \cond \x_0) \,\|\, p^\thetab(\z_0)\bigr) - \E_{q^\psib(\z_0 \cond \x_0)}\bigl[\log p^\phib(\x_0 \cond \z_0)\bigr].
\end{equation}
\end{lemma}

\begin{proof}
We derive the evidence lower bound using the standard variational decomposition. Starting from the log-marginal likelihood:
\begin{align}
\log p(\x_0) 
&= \int q(\z_0 \cond \x_0) \log p(\x_0) \,\diff\z_0 
= \int q(\z_0 \cond \x_0) \log \frac{p(\z_0)\, p(\x_0 \cond \z_0)}{p(\z_0 \cond \x_0)} \,\diff\z_0 \\
&= \int q(\z_0 \cond \x_0) \log \frac{q(\z_0 \cond \x_0)\, p(\z_0)\, p(\x_0 \cond \z_0)}{q(\z_0 \cond \x_0)\, p(\z_0 \cond \x_0)} \,\diff\z_0 \\
&= \underbrace{\E_{q(\z_0 \cond \x_0)}\bigl[\log p(\x_0 \cond \z_0)\bigr] - \KL\bigl(q(\z_0 \cond \x_0) \,\|\, p(\z_0)\bigr)}_{\text{ELBO}} + \underbrace{\KL\bigl(q(\z_0 \cond \x_0) \,\|\, p(\z_0 \cond \x_0)\bigr)}_{\geq 0}.
\end{align}

Since the KL divergence is non-negative, we obtain:
\begin{equation}
\log p(\x_0) \;\geq\; \E_{q(\z_0 \cond \x_0)}\bigl[\log p(\x_0 \cond \z_0)\bigr] - \KL\bigl(q(\z_0 \cond \x_0) \,\|\, p(\z_0)\bigr).
\end{equation}

Substituting the learned components—encoder $q^\psib(\z_0 \cond \x_0) \approx p(\z_0 \cond \x_0)$, decoder $p^\phib(\x_0 \cond \z_0) \approx p(\x_0 \cond \z_0)$, and latent prior $p^\thetab(\z_0) \approx p(\z_0)$ (learned via diffusion)—yields:
\begin{equation}
\log p^\thetab(\x_0) \;\geq\; \E_{q^\psib(\z_0 \cond \x_0)}\bigl[\log p^\phib(\x_0 \cond \z_0)\bigr] - \KL\bigl(q^\psib(\z_0 \cond \x_0) \,\|\, p^\thetab(\z_0)\bigr).
\end{equation}

\end{proof}

\newpage

\end{document}